\documentclass[final,5p,times,twocolumn]{elsarticle}
\biboptions{authoryear}

\usepackage{amsmath,amssymb,graphicx,latexsym,soul,color,hyperref,setspace,mathrsfs,framed,pdflscape,xcolor}
\hypersetup{pdftex,colorlinks=true,linkcolor=blue,citecolor=blue,filecolor=blue,urlcolor=blue}

\def\tr{{\raise0pt\hbox{$\scriptscriptstyle\top$}}}

\newtheorem{example}{Example}
\newtheorem{proposition}{Proposition}
\newtheorem{theorem}{Theorem}
\newtheorem{corollary}{Corollary}
\newtheorem{definition}{Definition}
\newtheorem{remark}{Remark}

\def\eop{\hfill{$\Box$}\medskip}

\newenvironment{proof}{\textbf{Proof}}{\eop}

\newenvironment{detail}{\smallskip\noindent\textbf{Detail.}~}{\phantom{a}\eop}

\begin{document}

\begin{frontmatter}

\title{General stochastic separation theorems with optimal bounds}

\author[LeicMath]{Bogdan Grechuk}
\ead{bg83@le.ac.uk}
\ead{https://orcid.org/0000-0002-2624-5765}
\author[LeicMath,NNU]{Alexander N. Gorban\corref{cor1}}
\ead{ag153@le.ac.uk}
\ead{https://orcid.org/0000-0001-6224-1430}
\author[LeicMath,NNU]{Ivan Y. Tyukin}
\ead{it37@le.ac.uk}
\ead{https://orcid.org/0000-0002-7359-7966}

\address[LeicMath]{Department of Mathematics, University of Leicester, Leicester, LE1 7RH, UK}
\address[NNU]{Lobachevsky University, Nizhni Novgorod, Russia}

 \cortext[cor1]{Corresponding author}

\date{}

\begin{abstract}

Phenomenon of stochastic separability was revealed and used in machine learning to correct errors of Artificial Intelligence (AI) systems and analyze AI instabilities. In high-dimensional datasets under broad assumptions each point can be separated from the rest of the set by simple and robust Fisher's discriminant (is {\em Fisher separable}). Errors or clusters of errors can be separated from the rest of the data. The ability to correct an AI system also opens up the possibility of an attack on it, and the high dimensionality induces vulnerabilities caused by the same stochastic separability that holds the keys to understanding the fundamentals of robustness and adaptivity in high-dimensional data-driven AI.
To manage errors and analyze vulnerabilities, the stochastic separation theorems should evaluate the probability that the dataset will be Fisher separable in given dimensionality and for a given class of distributions. Explicit and optimal estimates of these separation probabilities are required, and this problem is solved in present work.
The general stochastic separation theorems with optimal probability estimates are obtained for important classes of distributions: log-concave distribution, their convex combinations and product distributions. The standard i.i.d. assumption was significantly relaxed.
These theorems and estimates can be used both for correction of high-dimensional data driven AI systems and for analysis of their vulnerabilities. The third area of application is the emergence of memories in ensembles of neurons, the phenomena of grandmother's cells and sparse coding in the brain, and explanation of unexpected effectiveness of small neural ensembles in high-dimensional brain.
\end{abstract}
\begin{keyword}
AI, blessing of dimensionality, curse of dimensionality, concentration of measure, AI errors, discriminant
\end{keyword}
\end{frontmatter}

\newpage

\section{Introduction: Data mining in post-classical world}

Big data `revolution' and the growth of the data dimension are commonplace. However, some implications of this growth are not so well known. In his `millennium lecture', \citet{Donoho2000} sought to present major 21st century challenges for data analysis. He described the multidimensional post-classical world where the number of attributes $d$ (dimensionality of the dataspace)  exceeds the sample size $N$:
\begin{equation}\label{DonohoPostclass}
d\gg N.
\end{equation}

Of course, there are many practical tricks for handling data when the condition (\ref{DonohoPostclass}) holds. In such a situation, tools of the first choice are Principal Component Analysis with retaining of major components, the correlation transformation, that transforms the data set into its Gram matrix (the matrix of inner products or correlation coefficients between the data vectors), or their combination (for a case study see \citep{MoczkoMirkesGorPil2016}).  These methods return the situation from (\ref{DonohoPostclass}) to $d\leq N$ but this is not the end of the story. For the non-classical effects, the inequality (\ref{DonohoPostclass}) is not necessary. Many such effects arise when
\begin{equation}\label{Postclass2}
 d \gg \log N.
 \end{equation}
 Various examples of these effects are presented by \citet{Kurkova1993,Kainen1997,DonohoTanner2009,GorbTyuProSof2016}. High-dimensional data are very rarefied and have large data-free holes, even if the data sets are exponentially large \citep{Kainen1997}. Two effects are especially important:
 \begin{itemize}
 \item {\em Random vectors are quasiorthogonal}: $N$ random vectors $\boldsymbol{x}_i$ $(i=1,\ldots, N)$ on a unit $d$-dimensional sphere are almost orthogonal: for a given $\varepsilon>0$  and sufficiently large $d$ ($d\gg 1/\varepsilon$),  we can expect with high probability that $|(\boldsymbol{x}_i,\boldsymbol{x}_j)|<\varepsilon$ fo all $i \neq j$ when $d> A \ln N$ and $A$ depends  on $\varepsilon$ only.  (For various versions of exact formulations we refer to \citep{Kurkova1993,GorbTyuProSof2016}. Very recently, \citet{KainenKurkova2020} reviewed the concept of quasiorthogonal dimension and related notions.)
 \item {\em Random points are extreme}: for a given  $\varepsilon>0$ with a probability $p>1-\varepsilon$ all $N$ random points are vertices of their convex hull. It is sufficient that $d> B \ln N$ for some $B$ that depends on $\varepsilon$ only (for exact formulations we refer to \citep{DonohoTanner2009,convhull}).  
  \end{itemize}
Of course, detailing these properties should include clarifying what `random' means. A fairly regular distribution is usually assumed, while \citet{Donoho2000} claimed that such `blessing of dimensionality' effects hold for an unexpectedly wide class of distributions.

One more comment to (\ref{DonohoPostclass}), (\ref{Postclass2}) is necessary: existence of many attributes does not mean large dimensionality of data. 
The na\"{i}ve definition that dimensionality of data refers to  how many attributes a dataset has leads to some confusions. Indeed, in the simplest example, when  data are distributed along a straight line,  data are one-dimensional despite large number of attributes. To distinguish between the number of attributes and the dimensionality of a dataset, the latter is often referred to as the ``intrinsic dimensionality'' of the data. Not the number of attributes but the dimensionality of data should be used in the definition of the post-classical world: 
\begin{equation}\label{Postclass3}
 \dim(Dataset) \gg \log N.
 \end{equation}
Evaluation of the (intrinsic) dimensionality of data is a non-trivial problem discussed by many authors, and many approaches are used, ranged from  classical Principal Component Analysis (PCA) \citep{Joliffe2011} and their generalizations \citep{GorbanKegl2008}, to principal graphs and manifolds \citep{GorZin2010}, and fractal dimension \citep{Camastra2003}. In recent review by \citet{BacZinovyev2020} the typology of these methods is proposed and a new family of methods based on the data separability properties is presented.

In the post-classical world,   classical machine learning theory does not make much sense because it works near the limits of large $N$, when  the law of  large numbers and the central limit theorem can be used. The unlimited appetite of classical approaches for data is often considered as a `curse of dimensionality'. But  the properties  (\ref{DonohoPostclass}), (\ref{Postclass2}), or (\ref{Postclass3}) themselves are neither a curse, nor a blessing, and can be  beneficial. The idea of a `blessing of dimensionality' was formulated by \citet{Kainen1997}, but some properties  of the situations with  (\ref{DonohoPostclass}) were exploited much earlier.  In  general situation, if $d \leq N-1$, then any subsample is linearly separable from the rest of data. Therefore,  \citet[Theorem 1]{Rosenblatt1962} used a non-linear extension of the set of attributes ($A$-elements, Fig.~\ref{Fig:ElementaryPerceptron}) to prove the omnipotence of   elementary peceptrons in solving any classification problem (on a large training set, at least).

\begin{figure*} 
\centering
\includegraphics[width=0.6\textwidth]{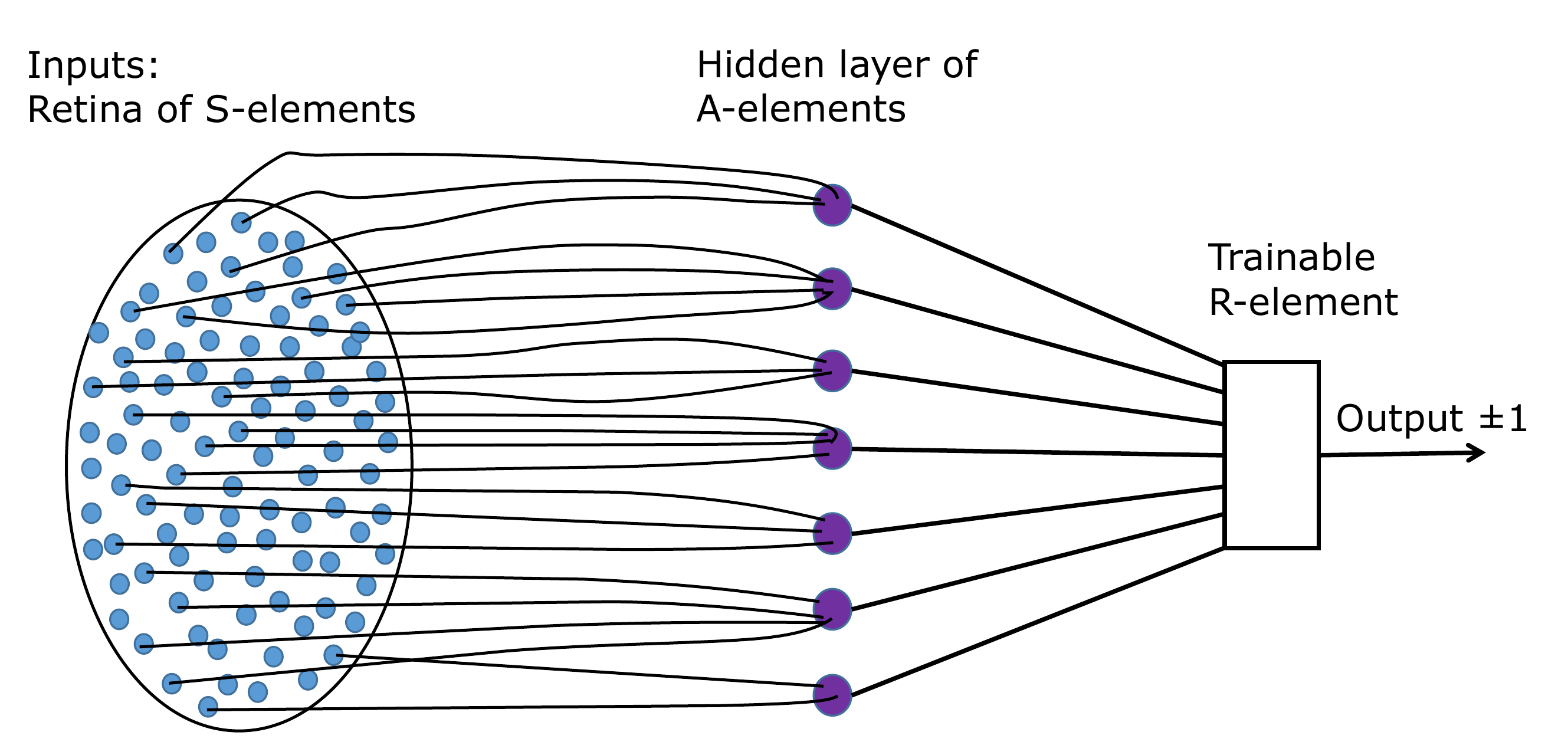}
\caption{Rosenblatt's elementary perceptron \citep{Rosenblatt1962}. $A$-and $R$-element are the  classical threshold neurons. $R$ element is trainable by the Rosenblatt algorithm, while $A$-elements should represent  a sufficient collection of features.
\label{Fig:ElementaryPerceptron}}
\end{figure*}

Other examples of post-classical phenomena are exponentially large sets of quasiorthogonal (almost orthogonal) random vectors we have already mentioned and stochastic separation in exponentially large datasets:  with high probability, any sample point is linearly separable from other points and this separation could be performed by the simple and explicit Fisher discriminant \citep{Gorbetal2018,GorbTyu2017,GorbanTyuRom2016}. This is a strengthening of the  statements \citep{convhull, Donoho2000,DonohoTanner2009}) that random points are extreme ones.  
These properties were proven for sufficiently regular probability distribution or for products of large number of low-dimensional distributions. For other examples we refer to the book by \citet{Vershynin2018}. 

The new characterization of post-classical data (\ref{Postclass3}) captures one of the qualitative characterization of the post-classical world. Fundamental open questions, however, are: 
\begin{enumerate}
\item Are there quantitatively accurate estimates of the boundary between the ``classical'' and the ``post-classical'' cases?  
\item How these boundaries depend on statistical properties of the data? 
\item If the ``post-classical'' limit  always obeys $\log(N)\ll \dim(Dataset)$ or could have different forms such as $\log(N)\ll \dim(Dataset)^p$?   
\end{enumerate}
Answering these would allow us to determine applicability bounds for a host of relevant measure concentration-based algorithms in machine learning, including one-shot error correction and learning, randomized approximation, and prevention of vulnerabilities to attacks.

The present work aims to answer these questions. In Sec.~\ref{Sec:Phenomenon} we introduce the stochastic separation phenomenon in detail and prove Theorem~\ref{Th:prototype} that is a prototype of most stochastic separation theorems. Estimates given in this theorem can be improved for specific classes of distributions but it does not use the i.i.d. assumption at al.  This major departure from the classical i.i.d. assumption in machine learning enables and justifies one-shot learning and AI correction algorithms in presence of concept drifts, sample dependencies, and non-stationarity.

Further in this work, we present such estimates for many practically important classes of probability distributions, in particular, for log-concave distributions and their convex combinations. In contrast to Theorem~\ref{Th:prototype} and Corollary~\ref{Cor:1} of Sec.~\ref{Sec:Phenomenon}, these estimates are in many cases asymptotically sharp.

In Sec.~\ref{Sec:Anal} the previously known results are analyzed, including estimations for  uniform distributions in a ball and a cube. In Sec.~\ref{Sec:StrongLCD} we prove the stochastic separation theorems with estimates of separation probability and sample sizes  for  strongly log-concave distributions using  the logarithmic Sobolev inequality and Poincare inequality. For special classes of distributions stronger results are obtained, for example, for  spherically invariant log-concave distributions including multivariate exponential distribution (Sec.~\ref{Sec:SpherInv}). The known estimates for some distributions like  uniform distribution in a ball and the standard normal distribution are significantly improved and optimal separation theorem for explicitly given distributions are found. Sec.~\ref{sec:product} derives separation theorems for independent data from product distributions, while Sec.~\ref{sec:dependent} generalizes some of these theorems to the case of dependent data relaxing the i.i.d. assumption. Sec.~\ref{Sec:Summ} briefly summarized the results, and in Sec.~\ref{sec:concl} we discuss what  these estimates are for and present the main areas of applications.

 \section{Stochastic separation phenomenon \label{Sec:Phenomenon}}
 
The `post-classical' phenomenon of separability of  random points  from random sets in high dimensionality opens up the possibility for fast and non-iterative correction of errors of data-driven Artificial Intelligence (AI). Each situation of AI functioning is represented by a vector that combines inputs, internal signals and outputs of the AI system. If a situation with error can be separated by an explicit and simple functional (Fisher's discriminant, for example) from the known situations with correct functioning then this error can be corrected forever without destroying the existing skills \citep{GorbanTyuRom2016,GorTyukPhil2018}. The corrector is a combination of the two-class classifier of situations (`AI error' versus `correct functioning') with   a modified decision rule for the `error' class. 
 
 Below in this section, a prototype of most stochastic separation theorems is introduces.
 
Recall that the classical Fisher discriminant between two classes with means $\boldsymbol{\mu}_1$ and $\boldsymbol{\mu}_2$ is separation of the classes by a hyperplane orthogonal to $\boldsymbol{\mu}_1-\boldsymbol{\mu}_2$ in the inner product 
 $$ \langle \boldsymbol{ x},\boldsymbol{ y}\rangle=( \boldsymbol{ x}, \boldsymbol{S}^{-1} \boldsymbol{ y}),$$
 where $( \cdot, \cdot )$ is the standard inner product and $ \boldsymbol{S}$ is the average (or the weighted average) of the sample covariance matrix of these two classes.  The classification rule is:  if $\langle \boldsymbol{\mu}_1-\boldsymbol{\mu}_2, \boldsymbol{x}\rangle \geq \vartheta$ then $\boldsymbol{x}$ belongs to the first class, otherwise it belongs to the second class. The threshold $\vartheta$ should be chosen  in such a way as to maximize  the quality of classification evaluated by a preselected criterion.
 
Applications of stochastic separation theory consider separating a single point (error) or a small cluster of such points from a relatively large data set. Thus,  $ \boldsymbol{S}$ is by default the empiric covariance matrix of a large data set. Further on, assume that the dataset is preprocessed,  this includes centralization (zero mean) and whitening. Whitening uses PCA to remove minor components and transform coordinates, making the empirical covariance matrix the identity matrix. After whitening, we get out of the situation described by the condition (\ref{DonohoPostclass}) but the conditions   (\ref{Postclass2})  or (\ref{Postclass3}) can persist. 
  
 It is necessary to stress that the precise whitening in applications to high-dimensional datasets could be unavailable, and $ \boldsymbol{S}$ may differ from $ \boldsymbol{1}$. If $ \boldsymbol{S}$ remains a well-conditioned matrix then this difference does not change qualitatively the separability properties. Analysis of the quantitative differences that may appear for non-isotropic   $ \boldsymbol{S}$ for some classes of probability distributions  is presented in Sec.~\ref{Sec:general}.

Presuming the described preprocessing with whitening, we  take $ \boldsymbol{S}=\boldsymbol{1}$ and $ \langle \boldsymbol{ x},\boldsymbol{ y}\rangle=( \boldsymbol{ x},   \boldsymbol{ y})$. 
 
\begin{definition}
A point $\boldsymbol{ x}$ is \emph{Fisher separable} from a set $Y \subset {\mathbb R}^n$ with center $\boldsymbol{ c} \in {\mathbb R}^n$ and threshold $\alpha \in (0, 1]$ if inequality
\begin{equation}\label{eq:Fisher}
\alpha(\boldsymbol{ x}-\boldsymbol{ c},\boldsymbol{ x}-\boldsymbol{ c}) > (\boldsymbol{ x}-\boldsymbol{ c},\boldsymbol{ y}-\boldsymbol{ c}), 
\end{equation}
holds for all $\boldsymbol{ y}\in Y$.  If \eqref{eq:Fisher} does not hold for some $\boldsymbol{ x}$ and $\boldsymbol{ y}$, we   say that $\boldsymbol{ x}$ and $\boldsymbol{ y}$ forms an (ordered) {\em $(\alpha,\boldsymbol{ c})$-inseparable pair} (see Fig.~\ref{Fig:Forbidden}).
\end{definition}
  
 \begin{figure} 
\centering
\includegraphics[width=0.25\textwidth]{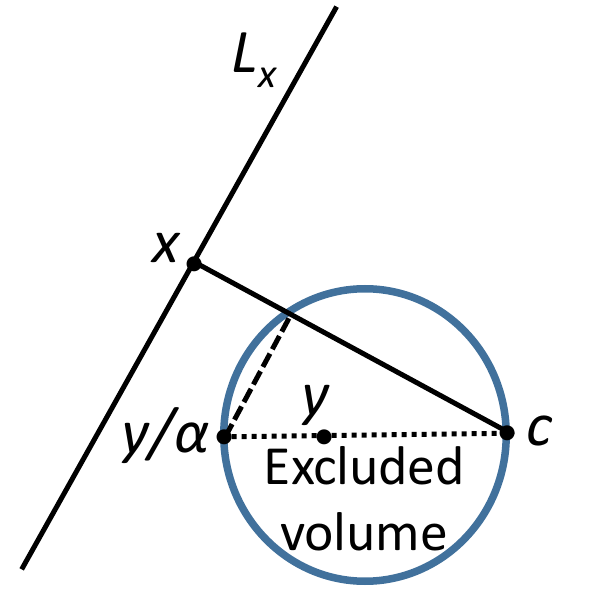}
\caption{Geometry of separation: $\alpha(\boldsymbol{ x},\boldsymbol{ x}) > (\boldsymbol{ x},\boldsymbol{ y})$ for all $\boldsymbol{ x}$ outside the outlined  ball (`excluded volume') with diameter $\|y\|/\alpha$. Here, $c$ is the origin (the data mean), $L_x$ is the hyperplane orthogonal to $x$.  If  $\boldsymbol{ x}$ belongs to the ball of excluded volume then)  $\boldsymbol{ x}$ and $\boldsymbol{ y}$ forms an ordered  $\alpha$-inseparable pair. 
\label{Fig:Forbidden}}
\end{figure} 
 
If $\boldsymbol{ c}=\boldsymbol{ 0}$ is the origin, we will write $(\alpha,\boldsymbol{ 0})$-inseparable pair as just ``$\alpha$-inseparable pair'', to simplify the notation. 
For a given $\boldsymbol{ y}$, the  set of such  $\boldsymbol{ x}$ that $\boldsymbol{ x}, \boldsymbol{ y}$ form an ordered $\alpha$-inseparable pair is a ball given by  inequality
\begin{equation}\label{excludedvolume}
\left\{\boldsymbol{z} \ \left| \ \left\|\boldsymbol{z}-\frac{\boldsymbol{y}}{2\alpha }\right\|< \frac{\|\boldsymbol{y}\|}{2\alpha} \right.  \right\}.
\end{equation}
This is the ball of excluded volume from Fig.~\ref{Fig:Forbidden}.

Two heuristic condition for the probability distribution are used in the stochastic separation theorems:
\begin{itemize}
\item   The probability distribution has no heavy tails;
\item The sets of small volume should not have large probability (what ``small'' and ``large'' mean should be strictly defined for different contexts).
\end{itemize}

In the following Theorem~\ref{Th:prototype} the absence of heavy tails is formalized as the tail cut: the support of the distribution is the $n$-dimensional unit ball $\mathbb{B}_n$.

The absence of the sets of small volume but large probability is formalized in this theorem by the following inequality:
\begin{equation}\label{bounded}
\rho(\boldsymbol{x})<\frac{C}{r^n V_n(\mathbb{B}_n)},
 \end{equation}
where $\rho$ is the distribution density, $C>0$ is an arbitrary constant, $V_n(\mathbb{B}_n)$ is the volume of the ball $\mathbb{B}_n$, and   $1>r>1/(2\alpha)$. This inequality guarantees that the probability measure of each ball with the radius less or equal than $1/(2\alpha)$ exponentially decays for $n\to \infty$. It should be stressed that the constant $C>0$ is arbitrary but  must not  {depend on $n$} in asymptotic analysis for large $n$. Condition   $1>r>1/(2\alpha)$ is possible only if $\alpha > 0.5$. Thus, the interval of possible $\alpha$ for Theorem~\ref{Th:prototype} is $\alpha \in (0.5,1]$.

\begin{theorem}\label{Th:prototype}\citep{Gorbetal2018}
\label{Theorem:ExclVol2}Let $1\geq \alpha >1/2$, $1>r>1/(2\alpha)$,  $1>\delta>0$, $Y\subset \mathbb{B}_n$ be a finite set, $|Y|<\delta (2r\alpha)^n/C$,  and $\boldsymbol{x}$ be a randomly chosen point from a distribution in the unit ball with the bounded probability density $\rho(\boldsymbol{x})$. Assume that $\rho(\boldsymbol{x})$ satisfies inequality (\ref{bounded}). Then with probability $p>1-\delta$ point $\boldsymbol{x}$ is Fisher-separable from $Y$ with threshold $\alpha$ (\ref{eq:Fisher}).
\end{theorem}
\begin{proof}
The volume of the ball (\ref{excludedvolume}) does not exceed $V=\left(\frac{1}{2\alpha}\right)^nV_n(\mathbb{B}_n)$ for each $\boldsymbol{y}\in Y$. The probability that point  $\boldsymbol{x}$ belongs to such a ball does not exceed
$$V \sup_{z\in \mathbb{B}_n }\rho(z)\leq C \left(\frac{1}{2r \alpha}\right)^n.$$
The probability that $\boldsymbol{x}$   belongs to the union of $|Y|$ such balls does not exceed $|Y| C\left(\frac{1}{2r\alpha}\right)^n$. For $|Y|<\delta (2r\alpha)^n/C$ this probability is smaller than $\delta$ and $p>1-\delta$.
\end{proof}

\begin{remark}Note that:
\begin{itemize}
\item The finite set $Y$ in Theorem~\ref{Th:prototype} is  just a finite subset of the ball $ \mathbb{B}_n$ without any assumption of its randomness.  We only used the assumption about distribution of  $\boldsymbol{x}$.
\item The distribution of  $\boldsymbol{x}$ may deviate significantly from the uniform distribution in the ball $ \mathbb{B}_n$. Moreover, this deviation may grow with dimension $n$ as a geometric progression:
$$\rho(\boldsymbol{x})/\rho_{\rm uniform}\leq {C}/{r^n},$$
where $\rho_{\rm uniform}=1/ V_n(\mathbb{B}_n)$ is the density of uniform distribution and $1/(2\alpha)<r<1$ (assuming that $1/2<\alpha  \leq 1$).
\end{itemize}
\end{remark}

\begin{example}\label{ex:th1}
Let $\alpha = 0.8$, $r=0.75$, $C=1$, $\delta=0.01$.   Table~\ref{Table1} shows the upper bounds on  $|Y|$ given by Theorem~\ref{Th:prototype} in various dimensions $n$ if the ratio $\rho(\boldsymbol{x})/\rho_{\rm uniform}$ is bounded by the geometric progression $1/r^n$.   
\begin{table}[h]  
\begin{center}
\caption{\label{Table1} The upper bound on $|Y|$ that guarantees separation of $\boldsymbol{x}$ from $Y$  by Fisher's discriminant with probability 0.99 according to Theorem~\ref{Th:prototype} for  $\alpha = 0.8$, $r=0.75$, $C=1$ in various dimensions.}
\begin{tabular}{ |c|c|c| } 
 \hline
 $n$ & $\rho(\boldsymbol{x})/\rho_{\rm uniform}\leq$ & $|Y|\leq$ \\ 
 \hline
 $10$ & $17.7$ & $0.06$ \\ 
 $50$ & $1.7\cdot 10^6$ & $91$ \\ 
 $100$ & $3.1 \cdot 10^{12}$ & $828,180$ \\ 
 $200$ & $9.7 \cdot 10^{24}$ & $6.8 \cdot 10^{13}$ \\ 
 $500$ & $2.9 \cdot 10^{62}$ & $3.9 \cdot 10^{37}$ \\ 
 $1000$ & $8.6 \cdot 10^{124}$ & $1.5 \cdot 10^{77}$ \\ 
 \hline
\end{tabular}
\end{center}
\end{table}

For example, for $n=100$, we see that for any set with $|Y|<828,180$ points in the unit ball, and any distribution whose density $\rho$ deviates from the uniform one by a factor at most $3.1 \cdot 10^{12}$, a random point from this distribution is Fisher-separable from all points in $Y$ with $99\%$ probability.
\end{example}

In the following Definition we consider separation of each points of a set from all other points by Fisher discriminant.
\begin{definition}\label{def:Fisher}
A finite set $Y \subset {\mathbb R}^n$ is \emph{Fisher separable} with center $\boldsymbol{ c} \in {\mathbb R}^n$ and threshold $\alpha \in (0, 1]$, or $(\alpha, \boldsymbol{ c})$-Fisher separable in short, if inequality 
$$
\alpha(\boldsymbol{ x}-\boldsymbol{ c},\boldsymbol{ x}-\boldsymbol{ c}) > (\boldsymbol{ x}-\boldsymbol{ c},\boldsymbol{ y}-\boldsymbol{ c}), 
$$
holds for all $\boldsymbol{ x}, \boldsymbol{ y} \in Y$ such that $\boldsymbol{ x}\neq  \boldsymbol{ y}$. 
\end{definition}

If $\boldsymbol{ c}=\boldsymbol{ 0}$ is the origin, we will write $(\alpha, \boldsymbol{ 0})$-Fisher separable set as just ``$\alpha$-Fisher separable'', to simplify the notation. 
From Theorem~\ref{Th:prototype} we obtain the following corollary.
\begin{corollary}\label{Cor:1}
If  $Y \subset \mathbb{B}_n$ is a random set $Y= \{\boldsymbol{y}_1, \ldots , \boldsymbol{y}_{|Y|}\}$ and for each $j$ the conditional distributions of vector $\boldsymbol{y}_j$ for any given positions of the other  $\boldsymbol{y}_k$ in $\mathbb{B}_n$    satisfy the same conditions as the distribution of $\boldsymbol{x}$ in  Theorem~\ref{Th:prototype}, then the probability of the random set $Y$ to be $\alpha$-Fisher separable can be easily estimated: 
$$p\geq 1- |Y|^2 C\left(\frac{1}{2r\alpha}\right)^n.$$
So, $p>0.99$ if $|Y|<(1/10)\,C^{-1/2}(2r\alpha)^{n/2}$. 
\end{corollary}

For this estimate, elements of $Y$ should not be i.i.d. random vectors and each of them can have its own distribution but with the same restrictions (with support in a ball and inequality  (\ref{bounded})). The flight from i.i.d assumption in machine learning is recognized as an important problem \citep{KurkovaComm2019}. The measure concentration phenomena can provide an instrument for avoiding this assumption \citep{Gorbetal2018,KurkovaSang2019}.

\begin{example}\label{ex:cor1}
Let $\alpha = 0.8$, $r=0.75$, $C=1$. Table~\ref{Table2} shows the upper bounds on  $|Y|$ given by Corollary \ref{Cor:1} in various dimensions $n$ if the ratio $\rho(\boldsymbol{x})/\rho_{\rm uniform}$ is bounded by the geometric progression $1/r^n$.  
\begin{table}[h]  
\begin{center}
\caption{\label{Table2} The upper bound on $|Y|$ that guarantees $\alpha$-Fisher separability of $Y$ with probability 0.99, according to Corollary \ref{Cor:1} for  $\alpha = 0.8$, $r=0.75$, $C=1$ in various dimensions.}
\begin{tabular}{ |c|c|c| } 
 \hline
 $n$ & $\rho(\boldsymbol{x})/\rho_{\rm uniform}\leq$ & $|Y|\leq$ \\ 
 \hline
 $10$ & $17.7$ & $0.25$ \\ 
 $50$ & $1.7\cdot 10^6$ & $9.54$ \\ 
 $100$ & $3.1 \cdot 10^{12}$ & $910$ \\ 
 $200$ & $9.7 \cdot 10^{24}$ & $8.2 \cdot 10^6$ \\ 
 $500$ & $2.9 \cdot 10^{62}$ & $6.2 \cdot 10^{18}$ \\ 
 $1000$ & $8.6 \cdot 10^{124}$ & $3.9 \cdot 10^{38}$ \\ 
 \hline
\end{tabular}
\end{center}
\end{table}

For example, for $n=100$, we see that for any distribution whose density $\rho$ deviates from the uniform one by a factor at most $3.1 \cdot 10^{12}$, any set with $|Y|<910$ points from this distribution is Fisher-separable with $99\%$ probability. In dimension $n=200$, we may deviate from the uniform density by a factor $9.7 \cdot 10^{24}$ and still separate over $8$ millions points.
\end{example}

In the post-classical world correction of AI errors is possible by separation of situations with errors from the situations of correct functioning. This can be done because the intrinsically high-dimensional data are very `rarefied'. At the same time, the possibility of repairing AI is closely related to the possibility of its attack. The specific post-classical vulnerabilities and new types of attacks were identified recently \cite{TyukinEtAl2020}. The exact line between the classic world of `condensed' data  and the post-classic world of rarefied data is important for both analyzing AI fixes and fixing AI vulnerability to attacks.

Theorem~\ref{Th:prototype} and Corollary~\ref{Cor:1} ensure us that if the probability distributions have no heavy tails and sets of relatively  small volume cannot have high probability, then the exponentially large sets are Fisher separable. Nevertheless, the presented estimates are far  from being optimal, and sharp estimations of probabilities and sample sizes are very desirable.

The formalization of the idea `no heavy tails' does not require a bounded distribution support. Below, we show that the exponential asymptotics at infinity will be fast enough to constructively describe the phenomenon of stochastic separation.  For this purpose, we use the class {\em log-concave distributions}, but already the first estimate shows that the asymptotics of $|Y|$ guaranteeing Fisher separability for such a general class of distributions are nonexponential: the boundary $|Y|$  that guarantees separability with a fixed probability, grows with dimension $ n $ as $ a \exp(b\sqrt{n}) $  (Theorem~\ref{th:logconcave}). It is demonstrated that this estimate cannot be significantly improved (Example~\ref{ex:logprod}). Exponential asymptotic is proved for a narrower class,  {\em strongly log-concave distributions} (Sec.~\ref{Sec:StrongLCD}). The most prominent member of this family is the normal distribution. The separability properties for the normal distribution are studied in detail in Sec.~\ref{Sec:Normal}.    

The general stochastic separation theorems are proven for convex combinations of  strongly log-concave distributions (Theorem~\ref{th:mixture}). The conditions of Theorem~\ref{th:mixture} formalize both no heavy tails condition  (through strong log-concavity) and no small sets with high probability condition. In some sense, the generality of this theorem is sufficient for most of practical purposes, but in specific cases, for narrower classes and selected distributions the estimates can be much better than for a wide general class. Therefore, we explore additional classes like product distributions (data with independent attributes) (Sec.~\ref{sec:product}), spherically invariant distributions (Sec.~\ref{Sec:SpherInv}) and some special examples: uniform distributions in a ball or in a cube and normal distribution. For data with independent attributes, the dependent samples are studied (Sec.~\ref{sec:dependent}). A  short guide on  proven theorems is presented in Sec.~\ref{Sec:Summ},  and in Sec.~\ref{sec:concl} we briefly discuss the application of the stochastic separation theorems in machine learning and neuroscience.

\section{Analysis of known stochastic separation theorems \label{Sec:Anal}}

Let us focus on Fisher separability because Fisher discriminants are robust and can be created by simple, explicit and one-shot rule. The   results of \citet{convhull} and \citet{DonohoTanner2009} about linear separability remain beyond the scope of this analysis.
  
 \citet{GorbTyu2017} proved that if $M$ points are selected independently uniformly at random in the unit ball in ${\mathbb R}^n$, then they are $1$-Fisher separable with high probability, provided that $M$ is bounded by some exponential function of $n$. A simple version of this result was later proved\footnote{The proof in \citep{Gorbetal2018} is presented for $\alpha=1$, but the argument works for general $\alpha$.} in \citep{Gorbetal2018}. 

\begin{theorem}\label{th:ballknown} \citep{Gorbetal2018}
Let points $\boldsymbol{ x}_1, \dots, \boldsymbol{ x}_M$ be i.i.d points from uniform distribution in a ball. For any $\delta>0$, if
\begin{equation}\label{eq:ballknown}
M < \sqrt{2\delta} (2\alpha)^{n/2} = \sqrt{2\delta}\exp\left(\frac{1}{2}\log(2\alpha) n \right),
\end{equation}
then set $F=\{\boldsymbol{ x}_1, \dots, \boldsymbol{ x}_M\}$ is $\alpha$-Fisher separable with probability greater than $1-\delta$.
\end{theorem}

The estimate \eqref{eq:ballknown} grows exponentially fast in $n$ provided that $\alpha>1/2$. 

\begin{example}\label{ex:ballknown}
Let $\delta=0.01$. Table~\ref{Table3} shows the upper bounds on $M$ in Theorem~\ref{th:ballknown} for $\alpha = 0.6, 0.8$ and $1$ in various dimensions $n$.
\begin{table}[h]
\begin{center}
\caption{\label{Table3} The upper bounds on $M$ that guarantees Fisher separability of $M$ i.i.d. points from uniform distribution in an $n$-dimensional ball with probability $p>0.99$  for $\alpha = 0.6, 0.8$ and $1$, according to  Theorem~\ref{th:ballknown}.} 
\begin{tabular}{ |c|c|c|c| } 
 \hline
  & $\alpha=0.6$ & $\alpha=0.8$ & $\alpha=1$\\ 
 \hline
 $n=10$ & $0.35$ & $1.48$ & $4.52$ \\ 
 $n=50$ & $13.5$ & $17,927$ & $4.7 \cdot 10^6$\\ 
 $n=100$ & $1287$ & $2.2 \cdot 10^9$ & $1.6 \cdot 10^{14}$\\ 
 $n=200$ & $1.1 \cdot 10^7$ & $3.6 \cdot 10^{19}$ & $1.8 \cdot 10^{29}$\\ 
 $n=500$ & $8.8 \cdot 10^{18}$ & $1.5 \cdot 10^{50}$ & $2.5 \cdot 10^{74}$\\ 
 $n=1000$ & $5.5 \cdot 10^{38}$ & $1.6 \cdot 10^{101}$ & $4.6 \cdot 10^{149}$\\ 
 \hline
\end{tabular}
\end{center}
\end{table}

For example, for $n=100$, we see that over $2$ billions points from the uniform distribution in the unit ball are Fisher-separable at level $\alpha=0.8$ with probability greater than $99\%$.
\end{example}

Of course, uniform distribution in a ball is a very special case, and separation theorems have been proved for various other families of distributions. We say that density $\rho:{\mathbb R}^n \to [0,\infty)$ of random vector $\boldsymbol{x}$ (and the corresponding probability distribution) is \emph{log-concave}, if set $D=\{z\in{\mathbb R}^n \,|\, \rho(z)>0\}$ is convex and $g(z)=-\log(\rho(z))$ is a convex function on $D$. We say that $\rho$ is whitened, or \emph{isotropic}, if ${\mathbb E}[\boldsymbol{x}]=\boldsymbol{ 0}$, and
\begin{equation}\label{eq:isot}
{\mathbb E}[(\boldsymbol{x},\theta)^2)]=1\quad\quad \forall \theta \in \mathbb{S}^{n-1},
\end{equation}
where $\mathbb{S}^{n-1}$ is the unit sphere in ${\mathbb R}^n$. Equation \eqref{eq:isot} is equivalent to the statment that the variance-covariance matrix for the components of $\boldsymbol{x}$ is the identity matrix. This can be achieved by linear transformation of the data during the pre-processing step, therefore this assumption is not restrictive.

\begin{theorem}\label{th:logconcave} \citep[Corollary 2]{Gorbetal2018}
Let $\{\boldsymbol{x}_1, \ldots , \boldsymbol{x}_M\}$ be a set of $M$ i.i.d. random points from an isotropic log-concave distribution in ${\mathbb R}^n$. Then set $\{\boldsymbol{x}_1, \ldots , \boldsymbol{x}_M\}$ is $1$-Fisher separable with probability greater than $1-\delta$, $\delta>0$, provided that
\begin{equation}\label{eq:logknown}
M \leq a e^{b\sqrt{n}},
\end{equation}
where $a>0$ and $b>0$ are constants, depending only on $\delta$.
\end{theorem}

The following Example demonstrates that $\sqrt{n}$ in \eqref{eq:logknown} cannot be replaced by $n^{0.5+\epsilon}$ for any $\epsilon>0$, even if points are selected from a product distribution with identical log-concave components. 

\begin{example}\label{ex:logprod}
Let $||{\boldsymbol x}||_1 = \sum_{i=1}^n|x_i|$ denotes the $l_1$ norm of ${\boldsymbol x}=(x_1,x_2,\dots,x_n) \in {\mathbb R}^n$. Let points $\boldsymbol{ x}_1, \dots, \boldsymbol{ x}_M$ be i.i.d points from (isotropic log-concave) distribution in ${\mathbb R}^n$ with density
$$
\rho({\boldsymbol x})=2^{-n/2}e^{-\sqrt{2}\cdot||{\boldsymbol x}||_1}.
$$
For any $\alpha\in(0,1]$, $a>0$, $b>0$, and $\epsilon>0$, if
\begin{equation}\label{eq:mbig}
M \geq a \exp\left(b\cdot n^{0.5+\epsilon}\right),
\end{equation}
then set $\{\boldsymbol{x}_1, \ldots , \boldsymbol{x}_M\}$ is \emph{not} $\alpha$-Fisher separable with probability tending to $1$ as $n\to\infty$.
\end{example}
\begin{detail}
The probability that any two i.i.d. points ${\boldsymbol x}=(x_1,x_2,\dots,x_n)$ and ${\boldsymbol y}=(y_1,y_2,\dots,y_n)$ from the given distribution are \emph{not} $\alpha$-Fisher separable is bounded by
\begin{equation*}
\begin{split}
&{\mathbb P}\left[\alpha(\boldsymbol{ x},\boldsymbol{ x}) \leq (\boldsymbol{ x},\boldsymbol{ y})\right] \geq {\mathbb P}\left[(\boldsymbol{ x},\boldsymbol{ x}) \leq (\boldsymbol{ x},\boldsymbol{ y})\right] = \\
 &\;\;\;\;\; {\mathbb P}\left[\sum_{i=1}^n (x_iy_i-x_i^2) \geq 0\right] = {\mathbb P}\left[\sum_{i=1}^n z_i \geq n\right],
\end{split}
\end{equation*}
where $z_i=x_iy_i-x_i^2+1, i=1,\dots,n$ are i.i.d. random variables with zero mean. Next,
\begin{equation*}
\begin{split}
{\mathbb P}\left[\sum_{i=1}^n z_i \geq n\right] \geq & {\mathbb P}\left[z_1 \geq n\right]\cdot {\mathbb P}\left[\sum_{i=2}^n z_i \geq 0\right] =
 \\ & {\mathbb P}\left[z_1 \geq n\right]\left(\frac{1}{2}+o(1)\right),
\end{split}
\end{equation*}
where the last equality follows from central limit theorem, and $o(1)$ is the quantity which goes to $0$ as $n\to\infty$. Further,
\begin{equation*}
\begin{split}
\frac{1}{2}{\mathbb P}\left[z_1 \geq n\right] \geq & \frac{1}{2}{\mathbb P}\left[\sqrt{n} \leq x_1 \leq   2\sqrt{n}\right] \times  {\mathbb P}\left[3\sqrt{n} \leq y_1\right]= \\ &\frac{1}{8}e^{-4\sqrt{2}\sqrt{n}}\left(1+o(1)\right) . 
\end{split}
\end{equation*}
Because $M$ points can be divided into $M/2$ independent pairs, the probability that all these pairs are $\alpha$-Fisher separable is at most
$$
\left(1-\frac{1}{8}e^{-4\sqrt{2}\sqrt{n}}\left(1+o(1)\right)\right)^{M/2},
$$
and the last expression vanishes as $n\to\infty$ if \eqref{eq:mbig} holds.
\end{detail}

Example \ref{ex:logprod} demonstrates that, to recover exponential dependence of $M$ from $n$, one must consider subclasses of log-concave distributions. 

We say that density $\rho:{\mathbb R}^n \to [0,\infty)$ is strongly log-concave with constant $\gamma>0$, or $\gamma$-SLC in short, if $g(z)=-\log(\rho(z))$ is strongly convex, that is, $g(z)-\frac{\gamma}{2}||z||$ is a convex function on $D$.

\begin{theorem}\label{th:strlogconc} \citep[Corollary 4]{Gorbetal2018}
Let $\{\boldsymbol{x}_1, \ldots , \boldsymbol{x}_M\}$ be a set of $M$ i.i.d. random points from an isotropic $\gamma$-SLC distribution in ${\mathbb R}^n$. Then set $\{\boldsymbol{x}_1, \ldots , \boldsymbol{x}_M\}$ is $1$-Fisher separable with probability greater than $1-\delta$, $\delta>0$, provided that
$$
M \leq a e^{bn},
$$
where $a>0$ and $b>0$ are some constants, which depends on $\delta$ and $\gamma$.
\end{theorem}

Separation theorems have also be proved for various families of distributions which are not log-concave. As an example, consider ``randomly perturbed data'' model (Example 2 in \cite{Gorbetal2018}). For a fixed $\epsilon \in (0,1)$, let $\boldsymbol{ y}_1, \boldsymbol{ y}_2, \dots, \boldsymbol{ y}_M$ be the set of $M$ arbitrary (non-random) points inside the ball with radius $1-\epsilon$ in ${\mathbb R}^n$. Let
$\boldsymbol{ x}_i, i=1,2,\dots,M$ be a point, selected uniformly at random from a ball with center $\boldsymbol{ y}_i$ and radius $\epsilon$. We think about $\boldsymbol{ x}_i$ as ``perturbed'' version of $\boldsymbol{ y}_i$.
In this model, we say that set $F=\{\boldsymbol{x}_1, \ldots , \boldsymbol{x}_M\}$ is \emph{$\alpha$-Fisher separable} if
$$
\alpha(\boldsymbol{ x}_i-\boldsymbol{ y}_i,\boldsymbol{ x}_i-\boldsymbol{ y}_i) > (\boldsymbol{ x}_i-\boldsymbol{ y}_i,\boldsymbol{ x}_j-\boldsymbol{ y}_i), 
$$
holds for all $i=1,2,\dots, M$ and $j=1,2,\dots, M$ such that $i\neq j$. 

\begin{theorem}\label{th:noisy}  \cite[Theorem 7]{Gorbetal2018}
Let $\{\boldsymbol{x}_1, \ldots , \boldsymbol{x}_M\}$ be a set of $M$ random points in the {``randomly perturbed''} model with parameter $\epsilon>0$.
For any $\vartheta$ such that $\frac{1}{\sqrt{n}} < \vartheta < 1$, set $\{\boldsymbol{x}_1, \ldots , \boldsymbol{x}_M\}$ is $1$-Fisher separable with probability at least
\begin{equation}\label{eq:probnoicy}
1 - \frac{2M^2}{\vartheta \sqrt{n}}\left(\sqrt{1-\vartheta^2}\right)^{n+1}-M\left(\frac{2\vartheta}{\epsilon}\right)^n.
\end{equation}
In particular, set $\{\boldsymbol{x}_1, \ldots , \boldsymbol{x}_M\}$ is $1$-Fisher separable with probability at least $1-\delta$, $\delta>0$, provided that $M<a b^n$, where $a,b$ are constants depending only on $\delta$ and $\epsilon$.
\end{theorem}

In Theorem \ref{th:noisy} we can, for any fixed $\epsilon, n$ and $M$, select $\vartheta \in \left(n^{-1/2},1\right)$ to maximize the lower bound \eqref{eq:probnoicy} for the probability. The optimal $\vartheta$ can be easily found numerically. 

\begin{example}\label{ex:noisy}
Let $\vartheta=\vartheta(\epsilon, n, M)$ be such that \eqref{eq:probnoicy} is maximized. Table~\ref{Table4} shows the lower bound for the probability that $M=100,000$ points in the ``randomly perturbed data'' model are $1$-Fisher separable, for various values of $n$ and $\epsilon$.
\begin{table}[h]
\begin{center}
\caption{\label{Table4}  The lower bound for probability that  100,000 points in the ``randomly perturbed data'' (Theorem~\ref{th:noisy}) are Fisher separable for various dimension $n$ and noise bound $\epsilon$.}
\begin{tabular}{ |c|c|c|c| } 
 \hline
  & $\epsilon=1/10$ & $\epsilon=1/5$ & $\epsilon=1/2$\\ 
 \hline
 $n=500$ & $<0$ & $<0$ & $<0$\\ 
 $n=1000$ & $<0$ & $<0$ & $0.9998$\\ 
 $n=2000$ & $<0$ & $<0$ & $1-5.8 \cdot 10^{-18}$\\ 
 $n=5000$ & $<0$ & $0.95$ & $1-1.2 \cdot 10^{-57}$\\ 
 $n=10000$ & $<0$ & $1-5\cdot 10^{-13}$ & $1-1.3 \cdot 10^{-123}$\\ 
 $n=20000$ & $0.96$ & $1-8\cdot 10^{-35}$ & $1 - 2.2 \cdot 10^{-255}$\\ 
 \hline
\end{tabular}
\end{center}
\end{table}

We see that Theorem \ref{th:noisy} is starting to give meaningful results only if the dimension $n$ is rather large, and the smaller $\epsilon$ the large dimension we need. This is not much surprising taking into account that the bounds in Table~\ref{Table4} are valid for an arbitrary set of $M$ points in the  $n$-dimensional ball with  radius $1-\epsilon$ and the perturbations make this random finite set closer to the i.i.d. sample from the uniform distribution in the unit ball. In the limit $\epsilon \to 1$ this randomly perturbed set turns into  such an i.i.d. sample.
\end{example}

Our final example concerns i.i.d. random points from a product distribution in a unit cube $U_n = [0,1]^n$.

\begin{theorem}\label{th:produnit} \cite[Corollary 2]{GorbTyu2017}
Let $\{\boldsymbol{x}_1, \ldots , \boldsymbol{x}_M\}$ be a set of $M$ i.i.d. random points from a product distribution in a unit cube. Let $\boldsymbol{ c} \in U_n$ be an arbitrary (non-random) point. Then set $\{\boldsymbol{x}_1, \ldots , \boldsymbol{x}_M\}$ is $(1, \boldsymbol{ c})$-Fisher separable with probability greater than $1-\delta$, $\delta>0$, provided that
\begin{equation}\label{eq:Mboundold}
(M+1)^2 < \frac{\delta}{3} \exp(0.5 n \sigma_0^4),
\end{equation}
where $\sigma_0$ is the minimal standard deviation of a component distribution.
\end{theorem}

Theorems \ref{th:ballknown}, \ref{th:logconcave}, \ref{th:strlogconc}, \ref{th:noisy} and  \ref{th:produnit} are proved in works by \citet{GorbTyu2017, Gorbetal2018} based on the following general principle, which, however, was not formulated explicitly. The high-dimensional stochastic separation theorems are formulated for the classes of distributions in  ${\mathbb R}^n$ for all sufficiently large $n$. For these classes, the probability that two random points are  $(\alpha,\boldsymbol{ c})$-Fisher inseparable is estimated from above by some function  $f(n, \alpha)$. After that, further estimates of $M$ and probabilities of separability of sets are constructed from this function,  $f(n, \alpha)$. Let us formulate this principle explicitly.

\begin{theorem}\label{th:principle}  
Let ${\cal F}$ be a family of $M$-point distributions in ${\mathbb R}^n$, $F \subset {\mathbb R}^n$ be a random $M$-point set chosen according to some distribution in ${\cal F}$, $\boldsymbol{ c} \in {\mathbb R}^n$, $\delta\in(0,1)$, and $I \subset (0,1]$. Assume that there exists a function $f(n,\alpha)$ such that for any two points $\boldsymbol{ x} \in F$ and $\boldsymbol{ y} \in F$
\begin{equation}\label{eq:2pointBound}
{\mathbb P}[\alpha(\boldsymbol{ x}-\boldsymbol{ c},\boldsymbol{ x}-\boldsymbol{ c}) \leq (\boldsymbol{ x}-\boldsymbol{ c},\boldsymbol{ y}-\boldsymbol{ c})] \leq f(n, \alpha), \quad \alpha \in I, \, n=1,2,\dots
\end{equation}
and
\begin{equation}\label{eq:MBoundexact}
M < \frac{1}{2}+\sqrt{\frac{1}{4}+\frac{\delta}{f(n, \alpha)}}.
\end{equation}
Then, for all $n$ and $\alpha \in I$, the expected number of $(\alpha,\boldsymbol{ c})$-inseparable pairs in $F$ is less than $\delta$. In particular, set $F$ is $(\alpha,\boldsymbol{ c})$-Fisher separable with probability greater than $1-\delta$.
\end{theorem}
\begin{proof}
If $I(i,j)$ is the indicator function for the event that pair $\boldsymbol{ x}_i$, $\boldsymbol{ x}_j$ is $(\alpha,\boldsymbol{ c})$-inseparable. Then the expected number of $(\alpha,\boldsymbol{ c})$-inseparable pairs is
\begin{equation*}
\begin{split}
{\mathbb E}\left[\sum_{i\neq j}I(i,j)\right] = & \sum_{i\neq j}{\mathbb E}[I(i,j)] \leq \sum_{i\neq j} f(n,\alpha) = \\
 & M(M-1)f(n,\alpha) < \delta,
\end{split}
\end{equation*}
where the last inequality follows from \eqref{eq:MBoundexact}.

If set $F$ would be $(\alpha,\boldsymbol{ c})$-Fisher separable with probability $p \leq 1-\delta$, then the expected number $E$ of $(\alpha,\boldsymbol{ c})$-inseparable pairs would be 
$$
E \geq p \cdot 0 + (1-p)\cdot 1 \geq \delta,
$$ 
which is a contradiction. Here, the first inequality follows from the fact that the number of $(\alpha,c)$-inseparable pairs is integer hence it is at least $1$. 
\end{proof}

If $\boldsymbol{ c}=\boldsymbol{ 0}$ is the origin, inequality \eqref{eq:2pointBound} simplifies to
\begin{equation}\label{eq:2pointBound0}
{\mathbb P}[\alpha(\boldsymbol{ x},\boldsymbol{ x}) \leq (\boldsymbol{ x},\boldsymbol{ y})] \leq f(n, \alpha), \quad \alpha \in I, \, n=1,2,\dots
\end{equation}
A sufficient condition for \eqref{eq:MBoundexact} is the simpler estimate
\begin{equation}\label{eq:MBound}
M \leq \sqrt{\frac{\delta}{f(n, \alpha)}}.
\end{equation}
We will always use \eqref{eq:MBound} in place of \eqref{eq:MBoundexact} unless we aim for the exact (necessary and sufficient) bound for $M$.
In particular, Theorem \ref{th:ballknown} follows from Theorem \ref{th:principle} with \eqref{eq:MBound}
and inequality
\begin{equation}\label{eq:2pointsphere}
{\mathbb P}[\alpha(\boldsymbol{ x},\boldsymbol{ x}) \leq (\boldsymbol{ x},\boldsymbol{ y})] \leq \frac{1}{2}(2\alpha)^{-n}, \quad \alpha \in (0,1], \, n=1,2,\dots
\end{equation}
which holds as equality for $\alpha=1$, see \citep{Gorbetal2018}. Theorems \ref{th:logconcave} and \ref{th:strlogconc} are proved in the same way. 
This implies the following corollary.

\begin{corollary}\label{cor:strongerconcl}
The conclusion ``set $\{\boldsymbol{x}_1, \ldots , \boldsymbol{x}_M\}$ is $1$-Fisher separable with probability greater than $1-\delta$'' in Theorems \ref{th:ballknown}, \ref{th:logconcave}, \ref{th:strlogconc}, \ref{th:noisy} and  \ref{th:produnit} can be replaced by a stronger conclusion that the expected number of inseparable pairs in this set is less than $\delta$.
\end{corollary}

This stronger conclusion is important for practical purposes because it prevents a scenario when we have many (maybe exponentially many in $n$) inseparable pairs with probability $\delta$.

The proof of Theorem \ref{th:principle} implies that the bound \eqref{eq:MBoundexact} is in fact \emph{necessary and sufficient} condition in the i.i.d case.

\begin{corollary}\label{cor:iidcase}
Let $\boldsymbol{ c}\in{\mathbb R}^n$ be fixed, $F=\{\boldsymbol{x}_1, \ldots , \boldsymbol{x}_M\}$ be a set of $M$ i.i.d. random points from an arbitrary distribution in ${\mathbb R}^n$. Let $f(n, \alpha) := {\mathbb P}[\alpha(\boldsymbol{ x}-\boldsymbol{ c},\boldsymbol{ x}-\boldsymbol{ c}) \leq (\boldsymbol{ x}-\boldsymbol{ c},\boldsymbol{ y}-\boldsymbol{ c})]$, where the probability does not depend on the choice of $\boldsymbol{ x} \in F$ and $\boldsymbol{ y} \in F$. Then the expected number of $\alpha$-inseparable pairs in $F$ is less than $\delta$ if and only if inequality \eqref{eq:MBoundexact} holds.
\end{corollary}

For example, the fact that inequality \eqref{eq:2pointsphere} is an equality for $\alpha=1$ implies the following optimal separation result.

\begin{corollary}\label{cor:ball1} 
Let $\alpha=1$, and let $F=\left\{\boldsymbol{ x}_1, \dots, \boldsymbol{ x}_M\right\}$ be the set of i.i.d points from uniform distribution in a ball. For any $\delta>0$, the expected number of $1$-inseparable pairs from $F$ is less than $\delta$ if and only if
\begin{equation}\label{eq:ball1}
M < \frac{1}{2} + \sqrt{\frac{1}{4}+\delta 2^{n+1}}.
\end{equation}
In particular, \eqref{eq:ball1} implies that $F$ is $1$-Fisher separable with probability greater than $1-\delta$.
\end{corollary}

In this paper, we prove a version of Corollary \ref{cor:ball1} for arbitrary $\alpha\in(0,1]$.

The disadvantage of Theorems \ref{th:logconcave}, \ref{th:strlogconc} and \ref{th:noisy} is that constants $a$ and $b$ in the bounds for $M$ are not explicitly given. In Theorem \ref{th:produnit}, the upper bound for $M$ is explicit but impractical in the important case if the dimension $n$ is measured in hundreds rather than in thousands. 

\begin{example}\label{ex:oldbound}
For $\delta=0.01$ (which corresponds to $99\%$ confidence), $n=500$, and $\sigma_0=0.5$ (maximal possible standard deviation for distribution with $[0,1]$ support), \eqref{eq:Mboundold} holds provided $M<141.7$.
\end{example}
In practise, however, datasets often have much more than $141$ point, but Fisher separability still holds. This motivates the search for stochastic Fisher separability theorems with better bounds. 

In this paper we obtain separation theorems for various classes of log-concave and product distributions with \emph{explicit} bounds on $M$. Moreover, we will aim to provide as good bounds as possible, ideally the optimal ones. In addition to better bounds, we also relax the i.i.d assumption.

In the i.i.d. case, Corollary \ref{cor:iidcase} implies that, if we can calculate the probability in \eqref{eq:2pointBound} \emph{exactly}, then \eqref{eq:MBoundexact} provides the optimal (necessary and sufficient) bound for $M$. This exact bound, however, is usually quite complicated, based on some integral expressions, and in such cases we will aim for simpler asymptotically tight bounds. 
We will write 
$$
f(n) \sim g(n)
$$
if $\lim\limits_{n\to\infty}\frac{f(n)}{g(n)}=1$. We say that function $g(n)$ is asymptotically tight lower (respectively, upper) bound for $f(n)$ if $f(n) \geq g(n)$ (respectively, $f(n) \leq g(n)$) and $f(n) \sim g(n)$. If $f(n,\alpha)$ in \eqref{eq:2pointBound} is the asymptotically tight upper bound for the probability in question, then \eqref{eq:MBoundexact} and \eqref{eq:MBound} provide asymptotically tight upper bounds for $M$.

If one can prove \eqref{eq:2pointBound} with $f(n,\alpha) = a e^{-2bn}$ for some constants $a$, $b$ depending on $\alpha$, one get \eqref{eq:MBound} with bound $M \leq \sqrt{\frac{\delta}{a}} e^{b n}$. 
If $f(n,\alpha) = a e^{-2bn}$, then $b = -\frac{\log(f(n,\alpha)/a)}{2n}$. In general, the last expression may depend on $n$, and we define
\begin{equation}\label{eq:bdef}
b(\alpha) = b_f(\alpha) := \lim\limits_{n \to \infty} -\frac{\log(f(n,\alpha))}{2n} = -\frac{1}{2}\lim\limits_{n \to \infty} \log\sqrt[n]{f(n,\alpha)}.
\end{equation}

Let ${\cal G}$ be the set of all functions $f(n,\alpha)$ for which \eqref{eq:2pointBound} holds.
We say that separation theorem \ref{th:principle} has optimal exponent if $b_f(\alpha) \geq b_g(\alpha)$ for all $g \in {\cal G}$.
Obviously, if bound in \eqref{eq:2pointBound} is asymptotically tight, it also has optimal exponent, but not vice versa. For non-optimal separation theorems the exponent $b(\alpha)$ is a good way to measure the ``quality'' of the theorem. We show that in all our non-optimal theorems the exponents differ from optimal by a factor less than $2$.
 
\section{Separation theorems for strongly log-concave distributions  \label{Sec:StrongLCD}}

\subsection{Separation of i.i.d. data from isotropic strongly log-concave distribution}

This Section proves the following explicit versions of Theorem \ref{th:strlogconc}.

\begin{theorem}\label{th:explstrong} 
Let $\delta>0$, $\alpha \in (0,1]$, $\gamma>0$, and let $F=\{\boldsymbol{ x}_1, \dots, \boldsymbol{ x}_M\}$ be a set of $M$ i.i.d. random points from an isotropic $\gamma$-SLC distribution in ${\mathbb R}^n$. If 
$$
M < \sqrt{\frac{\delta}{2}} \exp\left(\frac{\alpha^2(\gamma n - 1)}{4(1+\alpha)^2}\right), 
$$
then the expected number of $\alpha$-inseparable pairs in $F$ is less than $\delta$. In particular, set $F$ is $\alpha$-Fisher separable with probability greater than $1-\delta$.
\end{theorem}

\begin{theorem}\label{th:explstrong2} 
Let $\delta>0$, $\alpha \in (0,1]$, $\gamma>0$, and let $F=\{\boldsymbol{ x}_1, \dots, \boldsymbol{ x}_M\}$ be a set of $M$ i.i.d. random points from an isotropic $\gamma$-SLC distribution in ${\mathbb R}^n$. If $n > \frac{1+2\alpha^2}{\gamma \alpha^2}$ and
\begin{equation}\label{eq:betterbound}
\begin{split}
M < & \sqrt{\delta}\left(\frac{\alpha^2}{(1+\alpha^2)^{3/2}}\sqrt{2\pi(\gamma n-1)}\exp\left(-\frac{\alpha^2(\gamma n - 1)}{2(1+\alpha^2)}\right)\right.+ \\
&  \;\;\;\;\;\;\;\;\; 	\left. \exp\left(-\frac{\alpha^2(\gamma n - 1)}{2}\right)\right)^{-1/2}, 
\end{split}
\end{equation}
then the expected number of $\alpha$-inseparable pairs in $F$ is less than $\delta$. In particular, set $F$ is $\alpha$-Fisher separable with probability greater than $1-\delta$.
\end{theorem}
 
Theorem \ref{th:explstrong2} provides a less restrictive upper bound for $M$ for large $n$, while the upper bound in Theorem \ref{th:explstrong} is substantially simpler.

\begin{example}\label{ex:explstrong}
Let $\delta=0.01$ and $\alpha=1$. Table~\ref{Table5} shows the upper bounds on $M$ in Theorem~\ref{th:explstrong2} for $\gamma = 0.6, 0.8$ and $1$ in various dimensions $n$.
\begin{table}[h]
\begin{center}
\caption{\label{Table5} The upper bound of $M$ (\ref{eq:betterbound}) that guarantees Fisher separability of a $M$-element i.i.d. sample from an isotropic   $\gamma$-SLC distribution in ${\mathbb R}^n$ for $\gamma = 0.6, 0.8$ and $1$ with probability $p>0.99$ for various dimensions $n$.}
\begin{tabular}{ |c|c|c|c| } 
 \hline
  & $\gamma=0.6$ & $\gamma=0.8$ & $\gamma=1$\\ 
 \hline
 $n=10$ & $0.12$ & $0.15$ & $0.18$ \\ 
 $n=50$ & $1.71$ & $5.56$ & $18$\\ 
 $n=100$ & $61$ & $692$ & $7974$\\ 
 $n=200$ & $92,783$ & $1.2 \cdot 10^7$ & $1.8 \cdot 10^9$\\ 
 $n=500$ & $4.3 \cdot 10^{14}$ & $1.1 \cdot 10^{20}$ & $2.7 \cdot 10^{25}$\\ 
 $n=1000$ & $7 \cdot 10^{30}$ & $4.7 \cdot 10^{41}$ & $3.2 \cdot 10^{52}$\\ 
 \hline
\end{tabular}
\end{center}
\end{table}

For example, for $n=200$, we see that $12$ millions points from a strictly log-concave distribution with $\gamma=0.8$ are $1$-Fisher-separable with probability greater than $99\%$.
\end{example}

The exponent $b(\alpha)$ defined in \eqref{eq:bdef} is
$$
b(\alpha) = \frac{\alpha^2 \gamma}{4(1+\alpha)^2}
$$
for Theorem \ref{th:explstrong}, and
$$
b(\alpha) = \frac{\alpha^2\gamma}{4(1+\alpha^2)}
$$
for Theorem \ref{th:explstrong2}. For example, if $\gamma=1$ (which is the case for normal distribution) and $\alpha=1$, the exponents are $b=\frac{1}{16}=0.0625$ and $b=\frac{1}{8}=0.125$, respectively. The optimal exponent for normal distribution is given in Theorem \ref{th:normalknown} below and is equal to $\frac{1}{4}\log(2)=0.173..$, hence exponent in Theorem \ref{th:explstrong2} cannot be improved more than by a factor $2\log(2) = 1.386...$.

The first step in the proof of Theorems \ref{th:explstrong} and \ref{th:explstrong2} is the following estimate.

\begin{proposition}
Let $\boldsymbol{ x}$ and $\boldsymbol{ y}$ be two i.i.d. points from an isotropic $\gamma$-SLC distribution. Then
\begin{equation}\label{eq:step1}
{\mathbb P}[\alpha(\boldsymbol{ x},\boldsymbol{ x}) \leq (\boldsymbol{ x},\boldsymbol{ y})] \leq {\mathbb E}[e^{-\frac{\gamma \alpha^2}{2}||\boldsymbol{ x}||^2}].
\end{equation}
\end{proposition}
\begin{proof}
Theorem 5.2 in the book of \citet{Ledoux2001} states that, if random vector $\boldsymbol{ z}$ follows  a $\gamma$-SLC distribution, then logarithmic Sobolev inequality  
\begin{equation}\label{eq:sobolev}
{\mathbb E}[f^2(\boldsymbol{z})\log f^2(\boldsymbol{z})] - {\mathbb E}[f^2(\boldsymbol{z})]{\mathbb E}[\log f^2(\boldsymbol{z})]  \leq \frac{2}{\gamma} {\mathbb E}[||\nabla f(\boldsymbol{z})||^2]
\end{equation}
holds for every locally Lipschitz function $f$ on ${\mathbb R}^n$. By \citep[Theorem 5.3]{Ledoux2001}, this implies that inequality 
\begin{equation}\label{eq:1Lipschitz}
{\mathbb P}[g(\boldsymbol{z}) \geq E[g(\boldsymbol{z})] + r] \leq e^{-\gamma r^2/2}
\end{equation}
holds for every $r \geq 0$ and every $1$-Lipschitz function $g$ on ${\mathbb R}^n$.

Assuming that $\boldsymbol{ x} \neq \boldsymbol{ 0}$ is fixed, and applying \eqref{eq:1Lipschitz} to $1$-Lipschitz function $g(\boldsymbol{ y}) = \frac{(\boldsymbol{ x},\boldsymbol{ y})}{||\boldsymbol{x}||}$ with ${\mathbb E}[g(\boldsymbol{y})]=\frac{(\boldsymbol{ x},{\mathbb E}[\boldsymbol{ y}])}{||\boldsymbol{x}||} = \frac{(\boldsymbol{ x},\boldsymbol{ 0})}{||\boldsymbol{x}||} = 0$ and $r=\alpha ||\boldsymbol{x}||$, 
we obtain
\begin{equation}\label{eq:fixedx}
{\mathbb P}[\alpha(\boldsymbol{ x},\boldsymbol{ x}) \leq (\boldsymbol{ x},\boldsymbol{ y})] = {\mathbb P}\left[\frac{(\boldsymbol{ x},\boldsymbol{ y})}{||\boldsymbol{x}||} \geq 0 + \alpha ||\boldsymbol{x}|| \right] \leq e^{-\frac{\gamma \alpha^2}{2}||\boldsymbol{ x}||^2}
\end{equation}

Now let $\boldsymbol{ x}$ and $\boldsymbol{ y}$ be both random, and let $I$ be the indicator function of the event $\alpha(\boldsymbol{ x},\boldsymbol{ x}) \leq (\boldsymbol{ x},\boldsymbol{ y})$. Then
$$
{\mathbb P}[\alpha(\boldsymbol{ x},\boldsymbol{ x}) \leq (\boldsymbol{ x},\boldsymbol{ y})] = {\mathbb E}[I] = {\mathbb E}_{\boldsymbol{ x}}[{\mathbb E}_{\boldsymbol{ y}}[I|\boldsymbol{ x}]] \leq {\mathbb E}_{\boldsymbol{ x}}[e^{-\frac{\gamma \alpha^2}{2}||\boldsymbol{ x}||^2}],
$$
where the second equality follows from independence of $\boldsymbol{ x}$ and $\boldsymbol{ y}$, and the inequality follows from \eqref{eq:fixedx}.
 \end{proof}

The next proposition provides an easy estimate for the right-hand side of \eqref{eq:step1}.

\begin{proposition}\label{prop:step2}
Let $\boldsymbol{ x}$ be a points from an isotropic $\gamma$-SLC distribution. Then
\begin{equation}\label{eq:step2}
{\mathbb E}[e^{-\frac{\gamma \alpha^2}{2}||\boldsymbol{ x}||^2}] \leq 2\exp\left(-\frac{\gamma \alpha^2}{2(1+\alpha)^2}\mu^2\right),
\end{equation}
where $\mu = {\mathbb E}[||\boldsymbol{x}||]$.
\end{proposition}
\begin{proof}
For every $t > 0$,
\begin{equation*}
\begin{split}
{\mathbb E}[e^{-\frac{\gamma \alpha^2}{2}||\boldsymbol{ x}||^2}] =&{\mathbb E}[e^{-\frac{\gamma \alpha^2}{2}||\boldsymbol{ x}||^2}\,|\,||\boldsymbol{ x}|| > t] {\mathbb P}[||\boldsymbol{ x}|| > t] +  \\ 
&\;\;\; {\mathbb E}[e^{-\frac{\gamma \alpha^2}{2}||\boldsymbol{ x}||^2}\,|\,||\boldsymbol{ x}|| \leq t] {\mathbb P}[||\boldsymbol{ x}|| \leq t].
\end{split}
\end{equation*}
Now, 
$$
{\mathbb E}[e^{-\frac{\gamma \alpha^2}{2}||\boldsymbol{ x}||^2}\,|\,||\boldsymbol{ x}|| > t] \leq {\mathbb E}[e^{-\frac{\gamma \alpha^2}{2}t^2}\,|\,||\boldsymbol{ x}|| > t] = e^{-\frac{\gamma \alpha^2}{2}t^2},
$$
$$
{\mathbb P}[||\boldsymbol{ x}|| > t] \leq 1,
$$ 
$$
{\mathbb E}[e^{-\frac{\gamma \alpha^2}{2}||\boldsymbol{ x}||^2}\,|\,||\boldsymbol{ x}|| \leq t] \leq {\mathbb E}[1\,|\,||\boldsymbol{ x}|| \leq t] = 1,
$$
and
$$
{\mathbb P}[||\boldsymbol{ x}|| \leq t] = {\mathbb P}[\mu-||\boldsymbol{ x}|| \geq \mu-t] \leq e^{-\gamma (\mu-t)^2/2},
$$  
where the last inequality is an application of \eqref{eq:1Lipschitz} to $1$-Lipschitz function $g(\boldsymbol{ x}) = \mu-||\boldsymbol{ x}||$. Hence,
$$
{\mathbb E}[e^{-\frac{\gamma \alpha^2}{2}||\boldsymbol{ x}||^2}] \leq e^{-\frac{\gamma \alpha^2}{2}t^2} + e^{-\gamma (\mu-t)^2/2}.
$$
Applying the last inequality with $t = \frac{\mu}{\alpha+1}$, we get the result.
\end{proof}

The next Proposition provides an estimate for $\mu = {\mathbb E}[||\boldsymbol{x}||]$.

\begin{proposition}\label{prop:mest}
Let $\boldsymbol{ x}$ be a points from an isotropic $\gamma$-SLC distribution and let $\mu = {\mathbb E}[||\boldsymbol{x}||]$. Then
\begin{equation}\label{eq:step3}
\mu^2 \geq n - \frac{1}{\gamma}.
\end{equation}
\end{proposition}
\begin{proof}
As remarked by \citet[p. 92]{Ledoux2001}, the logarithmic Sobolev inequality \eqref{eq:sobolev} implies Poincare inequality
$$
Var[f(\boldsymbol{ x})] \leq \frac{1}{\gamma} {\mathbb E}[||\nabla f(\boldsymbol{x})||^2].
$$
Applying it with $f(\boldsymbol{ x})=||\boldsymbol{x}||$, we obtain 
$$
{\mathbb E}[||\boldsymbol{x}||^2] - \mu^2 = Var[||\boldsymbol{x}||] \leq \frac{1}{\gamma}||\nabla (||\boldsymbol{x}||)||=\frac{1}{\gamma}. 
$$
Because ${\mathbb E}[||\boldsymbol{x}||^2] = n$ for isotropic distributions, this implies \eqref{eq:step3}.
\end{proof}

\begin{proof} \textbf{of Theorem \ref{th:explstrong}.}
The combination of \eqref{eq:step1}, \eqref{eq:step2}, and \eqref{eq:step3} implies that \eqref{eq:2pointBound0} holds with 
$$
f_{\gamma}(n,\alpha)=2\exp\left(-\frac{\gamma \alpha^2}{2(1+\alpha)^2}\left(n - \frac{1}{\gamma}\right)\right),
$$
and Theorem \ref{th:explstrong} follows from Theorem \ref{th:principle}.
\end{proof}

To prove Theorem \ref{th:explstrong2}, we need an improved version of Proposition \ref{prop:step2}.

\begin{proposition}
Let $\boldsymbol{ x}$ be a points from an isotropic $\gamma$-SLC distribution. Then
\begin{equation}\label{eq:step2impr}
\begin{split}
&{\mathbb E}[e^{-\frac{\gamma \alpha^2}{2}||\boldsymbol{ x}||^2}] \leq  \\ &  \frac{\alpha^2}{(1+\alpha^2)^{3/2}}\sqrt{2\pi\gamma}\mu\exp\left(-\frac{\gamma\alpha^2\mu^2}{2(1+\alpha^2)}\right)+\exp\left(-\frac{\gamma\alpha^2\mu^2}{2}\right),
\end{split}
\end{equation}
where $\mu = {\mathbb E}[||\boldsymbol{x}||]$.
\end{proposition}
\begin{proof}
Because $e^{-\frac{\gamma \alpha^2}{2}||\boldsymbol{ x}||^2}$ takes value between $0$ and $1$,
\begin{equation*}
{\mathbb E}[e^{-\frac{\gamma \alpha^2}{2}||\boldsymbol{ x}||^2}] = \int_0^1 {\mathbb P}\left[e^{-\frac{\gamma \alpha^2}{2}||\boldsymbol{ x}||^2} > z\right] dz.
\end{equation*}

We have
\begin{equation*}
\begin{split}
p(z):=&{\mathbb P}\left[e^{-\frac{\gamma \alpha^2}{2}||\boldsymbol{ x}||^2} > z\right] = {\mathbb P}\left[||\boldsymbol{ x}|| < \sqrt{-\frac{2 \log z}{\gamma \alpha^2}}\,\right] =\\ & {\mathbb P}\left[\mu-||\boldsymbol{ x}|| > \mu - \sqrt{-\frac{2 \log z}{\gamma \alpha^2}}\,\right]
\end{split}
\end{equation*}

If $z \geq z_0 := e^{-\frac{1}{2}\gamma \alpha^2 \mu^2}$, then $r := \mu - \sqrt{-\frac{2 \log z}{\gamma \alpha^2}} \geq 0$, and \eqref{eq:1Lipschitz} with $1$-Lipschitz function $g(\boldsymbol{ x}) = \mu-||\boldsymbol{ x}||$ yields
$$
p(z) \leq q(z) := \exp\left(-\frac{\gamma}{2}\left(\mu - \sqrt{-\frac{2 \log z}{\gamma \alpha^2}}\right)^2\right).
$$
We also have a trivial estimate $p(z) \leq 1$ for $z \leq z_0$, which implies
$$
\int_0^1 p(z) dz \leq \int_0^{z_0} 1 \, dz + \int_{z_0}^1 q(z) \, dz = z_0 + I,
$$ 
where $I = \int_{z_0}^1 q(z) \, dz$. Integration in Mathematica returns
$$
I = \frac{\alpha^2}{2(1+\alpha^2)^{3/2}}\exp\left(-\frac{1}{2}\gamma(2+\alpha^2)\mu^2\right)(S_1 + S_2), 
$$
where 
$$
S_1 = -2 \sqrt{1 + \alpha^2} (e^{\gamma \mu^2} - e^{\frac{1}{2} (1 + \alpha^2) \gamma \mu^2}),
$$
and
\begin{equation*}
\begin{split}
S_2 =& \exp\left(\frac{(2 + 2 \alpha^2 + \alpha^4) \gamma \mu^2}{2(1+\alpha^2)}\right) \sqrt{2\pi\gamma}\mu \left(\phi\left(\sqrt{\frac{\gamma}{2+2\alpha^2}} m\right) \right. +\\ &\;\;\;\;\; \left. \phi\left(\alpha^2\sqrt{\frac{\gamma}{2+2\alpha^2}} m\right)\right),
\end{split}
\end{equation*}
where $\phi(y) := \frac{1}{\sqrt{\pi}}\int_{-y}^y e^{-t^2} dt$. Inequality $\alpha \leq 1$ implies that $S_1 \leq 0$. Using this and the fact that $\phi(y) \leq 1$ for all $y$, we get an estimate
\begin{equation*}
\begin{split}
\int_0^1 p(z) dz \leq z_0 + I \leq z_0 + &\frac{\alpha^2}{2(1+\alpha^2)^{3/2}}\exp\left(-\frac{1}{2}\gamma(2+\alpha^2)\mu^2\right) \times \\
&\exp\left(\frac{(2 + 2 \alpha^2 + \alpha^4) \gamma \mu^2}{2(1+\alpha^2)}\right) \sqrt{2\pi\gamma}\mu \cdot 2,
\end{split}
\end{equation*}
which simplifies to \eqref{eq:step2impr}.
\end{proof}

\begin{proof} \textbf{of Theorem \ref{th:explstrong2}.}

Let us consider the left-hand side of \eqref{eq:step2impr} as a function $f(\mu)$ of $\mu$ and show that $f$ is a decreasing function. The second term is clearly decreasing, while the first term is decreasing if the derivative of $\mu\exp\left(-\frac{\gamma\alpha^2\mu^2}{2(1+\alpha^2)}\right)$ is negative, which holds if $\mu^2 > \frac{1+\alpha^2}{\gamma \alpha^2}$. The last inequality follows from condition $n > \frac{1+2\alpha^2}{\gamma \alpha^2}$ and Proposition \ref{prop:mest}.
  
Because $f(\mu)$ is a decreasing function, and $\mu \geq n - \frac{1}{\gamma}$ by Proposition \ref{prop:mest}, we have 
$$
{\mathbb E}[e^{-\frac{\gamma \alpha^2}{2}||\boldsymbol{ x}||^2}] \leq f\left(n - \frac{1}{\gamma}\right).
$$
This together with \eqref{eq:step1} implies that \eqref{eq:2pointBound0} holds with $f_{\gamma}(n,\alpha)=\delta (R_n)^{-2}$, where $R_n$ is the right-hand side of \eqref{eq:betterbound}. Then \eqref{eq:betterbound} follows from Theorem \ref{th:principle}.
\end{proof}

\begin{remark}
In fact, the only place when we have used that the underlying  distribution is $\gamma$-SLC is the assertion that Sobolev inequality \eqref{eq:sobolev} holds. Hence, the condition that the distribution is $\gamma$-SLC in Theorems \ref{th:explstrong} and \ref{th:explstrong2} can be relaxed to the condition that the distribution is isotropic, log-concave, and such that \eqref{eq:sobolev} holds.   
\end{remark}

\subsection{Some generalizations \label{Sec:general}}

This section provides some generalizations of Theorem \ref{th:explstrong}. We first consider the case when the data are independent, but 
\begin{itemize}
\item[-] the data are not identically distributed, and
\item[-] the distributions for the data points are strongly log-concave but not necessarily isotropic.
\end{itemize} 

\begin{theorem}\label{th:indbound} 
Let $\delta>0$, $\alpha \in (0,1]$, and let $F=\{\boldsymbol{ x}_1, \dots, \boldsymbol{ x}_M\}$ be a set of $M$ independent random points in ${\mathbb R}^n$. Let $\boldsymbol{ x}_i$ follow a $\gamma_i$-SLC distribution with $\gamma_i>0$, with expectation $x_i^0=E[\boldsymbol{ x}_i]$ and norm expectation $\mu_i={\mathbb E}[||\boldsymbol{x}_i||]$. Assume that inequality
\begin{equation}\label{eq:mainassump}
||x_j^0||<\alpha \mu_i
\end{equation}
holds for every pair $1 \leq i,j \leq M$. If 
\begin{equation}\label{eq:indbound}
M < \sqrt{\frac{\delta}{2}} \exp\left(\min_{i,j}\left(\frac{\gamma_i \gamma_j(\mu_i\alpha - ||x_j^0||)^2}{2(\sqrt{\gamma_j}\alpha+\sqrt{\gamma_i})^2}\right)\right), 
\end{equation}
then the expected number of $\alpha$-inseparable pairs in $F$ is less than $\delta$. In particular, set $F$ is $\alpha$-Fisher separable with probability greater than $1-\delta$.
\end{theorem}
\begin{proof}
We will use inequality \eqref{eq:1Lipschitz}, which is valid for every $\gamma$-SLC distribution, not necessarily isotropic. Fix some indices $i$ and $j$. Define
\begin{equation}\label{eq:tdef}
t = \frac{\sqrt{\gamma_i}\mu_i + \sqrt{\gamma_j}||x_j^0||}{\sqrt{\gamma_j}\alpha + \sqrt{\gamma_i}}.
\end{equation}
Then
\begin{equation}\label{eq:mumint}
\mu_i - t = \frac{\sqrt{\gamma_j}(\alpha\mu_i - ||x_j^0||)}{\sqrt{\gamma_j}\alpha + \sqrt{\gamma_i}} > 0,
\end{equation}
where the inequality follows from \eqref{eq:mainassump}.

We have
\begin{equation}\label{eq:probsome}
\begin{split}
 &{\mathbb P}[\alpha(\boldsymbol{ x}_i,\boldsymbol{ x}_i) \leq (\boldsymbol{ x}_i,\boldsymbol{ x}_j)] \leq \\ & \;\;\; {\mathbb P}[||\boldsymbol{ x}_i||\leq t] + {\mathbb P}[\alpha(\boldsymbol{ x}_i,\boldsymbol{ x}_i) \leq (\boldsymbol{ x}_i,\boldsymbol{ x}_j)   \text{ conditional to } ||\boldsymbol{ x}_i||\geq t].
\end{split}
\end{equation}

Applying \eqref{eq:1Lipschitz} to $1$-Lipschitz function $g(\boldsymbol{ x}_i) = \mu_i-||\boldsymbol{ x}_i||$, we obtain
\begin{equation}\label{eq:term1}
{\mathbb P}[||\boldsymbol{ x}_i|| \leq t] = {\mathbb P}[\mu_i-||\boldsymbol{ x}_i|| \geq \mu_i-t] \leq e^{-\gamma_i (\mu_i-t)^2/2}.
\end{equation} 
Let us now estimate the second term in \eqref{eq:probsome}. Assume that $\boldsymbol{ x}_i$ such that $||\boldsymbol{ x}_i||\geq t$ is fixed. Applying \eqref{eq:1Lipschitz} to $1$-Lipschitz function $g(\boldsymbol{ x}_j) = \frac{(\boldsymbol{ x}_i,\boldsymbol{ x}_j)}{||\boldsymbol{x}_i||}$ with ${\mathbb E}[g(\boldsymbol{x}_j)]=\frac{(\boldsymbol{ x}_i,x_j^0)}{||\boldsymbol{x}_i||}$ and $r=\alpha ||\boldsymbol{x}_i|| - \frac{(\boldsymbol{ x}_i,x_j^0)}{||\boldsymbol{x}_i||}$, 
we obtain
$$
{\mathbb P}[\alpha(\boldsymbol{ x}_i,\boldsymbol{ x}_i) \leq (\boldsymbol{ x}_i,\boldsymbol{ x}_j)] = {\mathbb P}\left[\frac{(\boldsymbol{ x}_i,\boldsymbol{ x}_j)}{||\boldsymbol{x}_i||} \geq \frac{(\boldsymbol{ x}_i,x_j^0)}{||\boldsymbol{x}_i||} + r \right] \leq e^{-\frac{\gamma_j}{2}r^2},
$$
provided that $r>0$.
In fact,
$$
r=\alpha ||\boldsymbol{x}_i|| - \frac{(\boldsymbol{ x}_i,x_j^0)}{||\boldsymbol{x}_i||} \geq \alpha t - \frac{||\boldsymbol{x}_i||\cdot||x_j^0||}{||\boldsymbol{x}_i||} = \alpha t - ||x_j^0|| = \frac{\sqrt{\gamma_i}}{\sqrt{\gamma_j}}(\mu_i - t),
$$
where the last equality follows from \eqref{eq:tdef}. This implies that $r>0$, and also that
\begin{equation*}
\begin{split}
&{\mathbb P}[\alpha(\boldsymbol{ x}_i,\boldsymbol{ x}_i)\leq (\boldsymbol{ x}_i,\boldsymbol{ x}_j)] \leq \exp\left(-\frac{\gamma_j}{2}r^2\right) \leq \\ 
& \;\;\;\;\; \exp\left(-\frac{\gamma_j}{2}\left(\frac{\sqrt{\gamma_i}}{\sqrt{\gamma_j}}(\mu_i - t)\right)^2\right) = e^{-\gamma_i (\mu_i-t)^2/2}.
\end{split}
\end{equation*}
Because this inequality holds for every fixed $\boldsymbol{ x}_i$ such that $||\boldsymbol{ x}_i||\geq t$, it implies that
$$
{\mathbb P}[\alpha(\boldsymbol{ x}_i,\boldsymbol{ x}_i) \leq (\boldsymbol{ x}_i,\boldsymbol{ x}_j) \text{ conditional to } ||\boldsymbol{ x}_i||\geq t] \leq e^{-\gamma_i (\mu_i-t)^2/2}.
$$
Combining this with \eqref{eq:term1} and \eqref{eq:probsome}, we obtain
\begin{equation}\label{eq:finalest}
\begin{split}
 {\mathbb P}[\alpha(\boldsymbol{ x}_i,\boldsymbol{ x}_i) \leq & (\boldsymbol{ x}_i,\boldsymbol{ x}_j)] \leq  2e^{-\gamma_i (\mu_i-t)^2/2} \leq \\
& 2\exp\left(-\min_{i,j}\left(\frac{\gamma_i \gamma_j(\mu_i\alpha - ||x_j^0||)^2}{2(\sqrt{\gamma_j}\alpha+\sqrt{\gamma_i})^2}\right)\right),
\end{split}
\end{equation}
where we have used \eqref{eq:mumint}. The last bound holds for any pair of indices $i,j$, and application of Theorem \ref{th:principle} finishes the proof.
\end{proof}

Repeating the proof of Proposition \ref{prop:mest}, we get that
$$
\mu_i^2 \geq {\mathbb E}[||\boldsymbol{ x}_i||^2] - \frac{1}{\gamma_i}.
$$
Writing $\boldsymbol{ x}_i=(x_i^1, \dots, x_i^n)$ component-wise, we obtain that
$$
{\mathbb E}[||\boldsymbol{ x}_i||^2] = \sum\limits_{k=1}^n {\mathbb E}[(x_i^k)^2].
$$
Hence, if we assume that
\begin{itemize}
\item[-] All $\gamma_i$ are bounded from below by some constant independent of $n$, and
\item[-] The averages $\frac{1}{n}\sum\limits_{k=1}^n {\mathbb E}[(x_i^k)^2]$ are bounded from below by some constant independent of $n$, and
\item[-] Ratios $\frac{||x_j^0||}{\alpha \mu_i}$ are bounded from above by some constant $\beta<1$ independent from $n$,
\end{itemize}
then the bound in \eqref{eq:indbound} grows exponentially in $n$.

Next we consider the case when the data are i.i.d. but follow the distribution which is a mixture of $\gamma$-SLC distributions.

\begin{theorem}\label{th:mixture} 
Let $\delta>0$, $\alpha \in (0,1]$, and let $F=\{\boldsymbol{ x}_1, \dots, \boldsymbol{ x}_M\}$ be a set of $M$ i.i.d. random points in ${\mathbb R}^n$, which follow the distribution with density
$$
f(x) = \sum\limits_{i=1}^k \beta_i f_i(x),
$$
where $\beta_i\geq 0$ are coefficients such that $\sum\limits_{i=1}^k \beta_i = 1$, and $f_i$ are densities of $\gamma_i$-SLC distributions with $\gamma_i>0$. Let $x_i^0$ and $\mu_i$ be the expectation and norm expectation, respectively, of a random vector following distribution with density $f_i$.
Assume that inequality
$$
||x_j^0||<\alpha \mu_i
$$
holds for every pair $1 \leq i,j \leq k$. If 
$$
M < \sqrt{\frac{\delta}{2}} \exp\left(\min_{i,j}\left(\frac{\gamma_i \gamma_j(\mu_i\alpha - ||x_j^0||)^2}{2(\sqrt{\gamma_j}\alpha+\sqrt{\gamma_i})^2}\right)\right), 
$$
then the expected number of $\alpha$-inseparable pairs in $F$ is less than $\delta$. In particular, set $F$ is $\alpha$-Fisher separable with probability greater than $1-\delta$.
\end{theorem}
\begin{proof}
Let $\Omega \subset {\mathbb R}^{2n}$ be the set of points $(x_1,\dots,x_n,y_1,\dots,y_n) \in {\mathbb R}^{2n}$ such that $\alpha\sum\limits_{k=1}^n x_k^2 \leq \sum\limits_{k=1}^n x_k y_k$. Then, for any $\boldsymbol{ x}, \boldsymbol{ y} \in F$,
\begin{equation*}
\begin{split}
&{\mathbb P}[\alpha(\boldsymbol{ x},\boldsymbol{ x}) \leq (\boldsymbol{ x},\boldsymbol{ y})]=\int_\Omega f(x)f(y)dxdy = \\ 
&\;\;\;\;\; \int_\Omega \left(\sum\limits_{i=1}^k \beta_i f_i(x)\right)\left(\sum\limits_{j=1}^k \beta_j f_j(y)\right)dxdy = \\
&\;\;\;\;\; \sum\limits_{i=1}^k \sum\limits_{j=1}^k \beta_i \beta_j \int_\Omega f_i(x)f_j(y)dxdy \leq \sum\limits_{i=1}^k \sum\limits_{j=1}^k \beta_i \beta_j B = B,
\end{split}
\end{equation*}
where 
$$
B = 2\exp\left(-\min_{i,j}\left(\frac{\gamma_i \gamma_j(\mu_i\alpha - ||x_j^0||)^2}{2(\sqrt{\gamma_j}\alpha+\sqrt{\gamma_i})^2}\right)\right)
$$
is the right-hand side of \eqref{eq:finalest}. Then application of Theorem \ref{th:principle} finishes the proof.
\end{proof}

A straightforward combination of Theorems \ref{th:indbound} and \ref{th:mixture} allows to treat even more general case when the data are independent but not identically distributed, and the distribution of each data point is a mixture of log-concave ones, but the notation become messy so we omit the details.

\section{Separation theorems for spherically invariant log-concave distributions \label{Sec:SpherInv}}

Assume that points in ${\mathbb R}^n$ are selected from distribution whose density $\hat{\rho}:{\mathbb R}^n \to {\mathbb R}_+$, where ${\mathbb R}_+=[0,\infty)$, is spherically invariant, that is, 
\begin{equation}\label{eq:sphinv}
\hat{\rho}(x)=C_n\rho(||x||), \quad \forall x \in {\mathbb R}^n 
\end{equation} 
for some function $\rho:{\mathbb R}_+ \to {\mathbb R}_+$, where the factor $C_n$ is selected such that the density integrates to $1$. In fact,
$$
C_n = \frac{\Gamma\left(\frac{n}{2}\right)}{2\pi^{n/2}}\left(\int_0^\infty r^{n-1} \rho(r) dr\right)^{-1},
$$
where $\Gamma(z) = \int_0^\infty x^{z-1}e^{-x}dx$ is the gamma function.

This section derives separation theorems for such distributions. We start with optimal separation theorem for the most famous example of spherically invariant distribution, the standard normal one.  

\subsection{Standard normal distribution \label{Sec:Normal}}

For standard normal distribution, the following result is presented in the conference paper \citep[Corollary 6]{Grechuk}.

\begin{theorem}\label{th:normalknown} 
Let points $\boldsymbol{ x}_1, \dots, \boldsymbol{ x}_M$ are i.i.d points from standard normal distribution. For any $\delta>0$, if
\begin{equation}\label{eq:normalknown}
M < \sqrt{\delta}\exp\left(\frac{1}{4}\log(1+\alpha^2) n\right) = \sqrt{\delta} (1+\alpha^2)^{n/4},
\end{equation}
then set $F=\{\boldsymbol{ x}_1, \dots, \boldsymbol{ x}_M\}$ is $\alpha$-Fisher separable with probability greater than $1-\delta$.
\end{theorem}

Theorem \ref{th:normalknown} follows from Theorem \ref{th:principle} and estimate 
$$
{\mathbb P}[\alpha(\boldsymbol{ x},\boldsymbol{ x}) \leq (\boldsymbol{ x},\boldsymbol{ y})] \leq (1+\alpha^2)^{-n/2}
$$
for i.i.d. points $\boldsymbol{ x}$ and $\boldsymbol{ y}$ from standard normal distribution. Here, we derive the \emph{exact} expression for ${\mathbb P}[\alpha(\boldsymbol{ x},\boldsymbol{ x}) \leq (\boldsymbol{ x},\boldsymbol{ y})]$.

From rotation invariance, we may assume that $\boldsymbol{ x}=(||\boldsymbol{ x}||,0,0,\dots,0)$. Then 
\begin{equation*}
\begin{split}
&{\mathbb P}[\alpha(\boldsymbol{ x},\boldsymbol{ x}) \leq (\boldsymbol{ x},\boldsymbol{ y})] =  {\mathbb P}[\alpha||\boldsymbol{ x}||^2 \leq ||\boldsymbol{ x}||y_1] 
= \\&\;\;\;\;\; {\mathbb P}[0<y_1]{\mathbb P}\left[\frac{||\boldsymbol{ x}||}{y_1} \leq \frac{1}{\alpha}\right] = \frac{1}{2}{\mathbb P}\left[\frac{||\boldsymbol{ x}||^2/n}{y_1^2} \leq \frac{1}{n\alpha^2}\right],
\end{split}
\end{equation*}
where $y_1$ is the first component of $\boldsymbol{ y}$, which follows the standard normal distribution. The sum of squares of $k$ independent standard normal random variables follows the chi-squared distribution $\chi(k)$ with degree $k$. Hence, $\frac{||\boldsymbol{ x}||^2/n}{y_1^2}$ is the ratio of two independent random variables from chi-squared distributions with degrees $n$ and $1$, respectively, scaled by their degrees. This ratio is known to follow so-called F-distribution $F(d_1, d_2)$ with parameters $d_1=n$ and $d_2=1$. The cumulative distribution function of F-distribution is 
$$
F(x; d_1, d_2) = I_{\frac{d_1x}{d_1x+d_2}}\left(\frac{d_1}{2}, \frac{d_2}{2}\right),
$$
where $I_z(a,b)$ is the cumulative distribution function of beta distribution, also known as regularized incomplete beta function. It is given by
\begin{equation}\label{eq:regincbeta}
I_z(a,b) = \frac{B_z(a, b)}{B(a, b)},
\end{equation} 
where $B_z(a, b)=\int_0^z t^{a-1}(1-t)^{b-1}dt$ is the incomplete beta function, and
$$
B(a, b) = \int_0^1 t^{a-1}(1-t)^{b-1}dt = \frac{\Gamma(a)\Gamma(b)}{\Gamma(a+b)}
$$
is the beta function.

Hence,
\begin{equation}\label{eq:normalexact}
{\mathbb P}[\alpha(\boldsymbol{ x},\boldsymbol{ x}) \leq (\boldsymbol{ x},\boldsymbol{ y})] = \frac{1}{2}F\left(\frac{1}{n\alpha^2}; n, 1\right) = \frac{1}{2} I_{\frac{1}{1+\alpha^2}}\left(\frac{n}{2}, \frac{1}{2}\right)
\end{equation}

With Theorem \ref{th:principle}, this implies the following optimal separation result.

\begin{theorem}\label{th:normalnew}
Let points $\boldsymbol{ x}_1, \dots, \boldsymbol{ x}_M$ are i.i.d points from standard normal distribution. For any $\delta>0$, the expected number of $\alpha$-inseparable pairs from set $F=\{\boldsymbol{ x}_1, \dots, \boldsymbol{ x}_M\}$ is less than $\delta$ if and only if
\begin{equation}\label{eq:normalnew}
M < \frac{1}{2}+\sqrt{\frac{1}{4}+\frac{2\delta}{I_{\frac{1}{1+\alpha^2}}\left(\frac{n}{2}, \frac{1}{2}\right)}}.
\end{equation}
 
In particular, \eqref{eq:normalnew} implies that $F$ is $\alpha$-Fisher separable with probability greater than $1-\delta$.
\end{theorem}

\begin{example}\label{ex:normalnew}
Let $\delta=0.01$. Table~\ref{Table6} shows the upper bounds on $M$ in Theorem~\ref{th:normalnew} for $\alpha = 0.6, 0.8$ and $1$ in various dimensions $n$.
\begin{table}[h]
\begin{center}
\caption{\label{Table6} The upper bound of $M$ (\ref{eq:normalnew}) that guarantees $\alpha$-Fisher separability of an $M$-element i.i.d. sample from the standard normal distribution with probability $p>0.99$ for various $\alpha$ and dimensions.}
\begin{tabular}{ |c|c|c|c| } 
 \hline
  & $\alpha=0.6$ & $\alpha=0.8$ & $\alpha=1$\\ 
 \hline
 $n=10$ & $1.19$ & $1.45$ & $1.99$ \\ 
 $n=50$ & $14$ & $164$ & $2075$\\ 
 $n=100$ & $794$ & $93,806$ & $1.4 \cdot 10^7$\\ 
 $n=200$ & $2\cdot 10^6$ & $2.6 \cdot 10^{10}$ & $5.6 \cdot 10^{14}$\\ 
 $n=500$ & $2.6 \cdot 10^{16}$ & $4.2 \cdot 10^{26}$ & $2.6 \cdot 10^{37}$\\ 
 $n=1000$ & $1.5 \cdot 10^{33}$ & $3.6 \cdot 10^{53}$ & $1.3 \cdot 10^{75}$\\ 
 \hline
\end{tabular}
\end{center}
\end{table}

For example, for $n=100$, we see that $14$ millions points from the standard normal distribution are $1$-Fisher-separable with probability greater than $99\%$. In dimension $n=200$, millions of points become Fisher separable even at level $\alpha=0.6$.
\end{example}

The following proposition establishes asymptotic behaviour of \eqref{eq:normalexact} as $n$ goes to infinity.

\begin{proposition}\label{prop:iest}
For every $a>0$, $b\in(0,1)$ and $z\in(0,1)$, we have
\begin{equation}\label{eq:iest}
I_z(a,b) \leq \frac{z^a (1-z)^{b-1} a^{b-1}}{\Gamma(b)},
\end{equation}
and the bound is asymptotically tight if $b$ and $z$ are fixed but $a\to\infty$, in sense that ratio of the right and left sides converges to $1$. In particular,
\begin{equation}\label{eq:iest2}
\frac{1}{2} I_{\frac{1}{1+\alpha^2}}\left(\frac{n}{2}, \frac{1}{2}\right) \leq \sqrt{\frac{1+\alpha^2}{2\pi n \alpha^2}}(1+\alpha^2)^{-n/2},
\end{equation}
and the bound is asymptotically tight if $\alpha$ is fixed and $n\to\infty$.
\end{proposition}
\begin{proof}
\citet{Wendel} proved that for every $a>0$ and $b\in(0,1)$ 
$
\frac{\Gamma(a+b)}{a^b \Gamma(a)} \leq 1, 
$
and $\lim\limits_{a\to\infty} \frac{\Gamma(a+b)}{a^b \Gamma(a)} = 1$. This is equivalent to 
\begin{equation}\label{eq:Wendel}
B(a,b) \geq \Gamma(b)a^{-b} \quad \text{and} \quad \lim\limits_{a\to\infty} \frac{B(a,b)}{\Gamma(b)a^{-b}} = 1.
\end{equation} 
Asymptotic expansion \citep{Lopez} implies that 
$$
B_x(a,b) \leq \frac{x^a(1-x)^{b-1}}{a},
$$ 
and $\lim\limits_{a\to\infty} \frac{B_x(a,b) a}{x^a(1-x)^{b-1}} = 1$. Hence, inequality \eqref{eq:iest} holds and is asymptotically tight. Applying it with $z=\frac{1}{1+\alpha^2}$, $a=n/2$, $b=1/2$, and using the fact that $\Gamma(1/2)=\pi$, we get \eqref{eq:iest2}.
\end{proof}

Theorem \ref{th:normalnew} and Proposition \ref{prop:iest} implies the following corollary, which provide a simple but asymptotically tight estimate for $M$.

\begin{corollary}\label{cor:normalnew}
Let points $\boldsymbol{ x}_1, \dots, \boldsymbol{ x}_M$ are i.i.d points from standard normal distribution. For any $\delta>0$, if
\begin{equation}\label{eq:normalnew2}
M < \sqrt[4]{\frac{2 \pi n \alpha^2 \delta^2}{1+\alpha^2}}(1+\alpha^2)^{n/4},
\end{equation}
then the expected number of $\alpha$-inseparable pairs in set $F=\{\boldsymbol{ x}_1, \dots, \boldsymbol{ x}_M\}$ is less than $\delta$. In particular, \eqref{eq:normalnew2} implies that set $F$ is $\alpha$-Fisher separable with probability greater than $1-\delta$. 
\end{corollary}

The bound in \eqref{eq:normalnew2} is weaker than in \eqref{eq:normalnew}, but the ratio of these bounds converges to $1$ if $\alpha$ is fixed and $n$ goes to infinity. In particular, the optimal exponent \eqref{eq:bdef} for standard normal distribution is
$$
b(\alpha) = \frac{1}{4}\log(1+\alpha^2).
$$
Corollary \ref{cor:normalnew} is an improvement over Theorem \ref{th:normalknown} if $n>\frac{1+\alpha^2}{2\pi\alpha^2}$. 

\begin{example}\label{ex:normal}
If $\delta=0.01$, $\alpha = 0.9$ and $n=100$, then \eqref{eq:normalknown} in Theorem \ref{th:normalknown} reduces to $M<276,671$, \eqref{eq:normalnew2} in Corollary \ref{cor:normalnew} reduces to $M<1,132,950$, while the optimal bound \eqref{eq:normalnew} in Theorem \ref{th:normalnew} is $M<1,141,060$.
\end{example}

\subsection{Optimal separation theorem for explicitly given distribution}

This section establishes optimal separation theorem if the rotation invariant distribution is not necessary standard normal but is explicitly given.

\begin{proposition}\label{prop:sph_exact}
Let ${\boldsymbol x}$ and ${\boldsymbol y}$ be two points selected independently from spherically invariant distributions with the same center. Then
\begin{equation}\label{eq:pest}
\begin{split}
&{\mathbb P}[\alpha(\boldsymbol{ x},\boldsymbol{ x}) \leq (\boldsymbol{ x},\boldsymbol{ y})] = \\
&\;\;\;\;\; \frac{\alpha}{B(\frac{n-1}{2}, \frac{1}{2})}\int_0^{1/\alpha} (1-\alpha^2 t^2)^{\frac{n-3}{2}}{\mathbb P}\left[\frac{||{\boldsymbol x}||}{||{\boldsymbol y}||} \leq t\right] dt,
\end{split}
\end{equation}
where $B(.)$ is the beta function.
\end{proposition}
\begin{proof} Note that
\begin{equation*}
\begin{split}
&{\mathbb P}[\alpha(\boldsymbol{ x},\boldsymbol{ x}) \leq (\boldsymbol{ x},\boldsymbol{ y})] = \\
&\;\;\;\;\; {\mathbb P}[\alpha||\boldsymbol{ x}||^2 \leq ||\boldsymbol{ x}|| \cdot ||\boldsymbol{ y}|| \cos\beta ] =
{\mathbb P}\left[\alpha t \leq \cos\beta \right],
\end{split}
\end{equation*}
where $t = \frac{||{\boldsymbol x}||}{||{\boldsymbol y}||}$, and $\beta$ is an angle between ${\boldsymbol x}$ and ${\boldsymbol y}$. By spherical invariance, the last probability is equal to the ratio of the area of the hyperspherical cap with angle $\beta_0=\arccos(\alpha t) = \arcsin(\sqrt{1-\alpha^2 t^2})$
to the area of the whole hypersphere, provided that $t \leq \frac{1}{\alpha}$. By \citep{Li}, this ratio is equal to $\frac{1}{2}I_{\sin^2\beta_0}\left(\frac{n-1}{2}, \frac{1}{2}\right)$, where
$I_z(a,b)$ is given by \eqref{eq:regincbeta}.
 Hence, 
$$
{\mathbb P}\left[\alpha t \leq \cos\beta \right] = \frac{1}{2}\int_0^{1/\alpha} I_{1-\alpha^2 t^2}\left(\frac{n-1}{2}, \frac{1}{2}\right) u(t) dt,
$$  
where $u$ is the density of the distribution of $\frac{||{\boldsymbol x}||}{||{\boldsymbol y}||}$. 

Using the formula for density of beta distribution, we get
\begin{equation*}
\begin{split}
&\frac{d}{dt}I_{1-\alpha^2 t^2}\left(\frac{n-1}{2}, \frac{1}{2}\right) =  \\
&\;\;\;\;\; \frac{(1-\alpha^2 t^2)^{\frac{n-1}{2}-1}(1-(1-\alpha^2t^2))^{1/2-1}}{B(\frac{n-1}{2}, \frac{1}{2})}\frac{d}{dt}(1-\alpha^2t^2) =\\
&\;\;\;\;\; \frac{-2\alpha(1-\alpha^2 t^2)^{\frac{n-3}{2}}}{B(\frac{n-1}{2}, \frac{1}{2})},
\end{split}
\end{equation*}
where $B(a,b) = \int_0^1 z^{a-1}(1-z)^{b-1}dz$ is the beta function. Hence, integration by parts yields \eqref{eq:pest}.
\end{proof}

If ${\boldsymbol x}$ has density given by \eqref{eq:sphinv}, then $||{\boldsymbol x}||$ has density given by $C_n r^{n-1} \rho(r)$, where $C_n=\left(\int_0^\infty r^{n-1} \rho(r) dr\right)^{-1}$ is the normalization constant. 
Hence,
\begin{equation}\label{eq:hdef}
{\mathbb P}\left[\frac{||{\boldsymbol x}||}{||{\boldsymbol y}||} \leq t\right] = h(n,t) := \frac{\int_0^\infty y^{n-1} \rho(y) dy \int_0^{ty} x^{n-1} \rho(x) dx}{\left(\int_0^\infty r^{n-1} \rho(r) dr\right)^2},
\end{equation}
where $\rho$ is defined in \eqref{eq:sphinv}.
Hence, Proposition \ref{prop:sph_exact} in combination Theorem \ref{th:principle} implies the following optimal separation theorem.

\begin{theorem}\label{th:sepfixed} 
Let $\delta>0$, $\alpha \in (0,1]$, and let $F=\{\boldsymbol{x}_1, \ldots , \boldsymbol{x}_M\}$ be a set of $M$ i.i.d. random points from a spherically invariant distribution in ${\mathbb R}^n$. 
Then the expected number of $\alpha$-inseparable pairs from set $F$ is less than $\delta$ if and only if
\begin{equation}\label{eq:sepfixed}
\begin{split}
&M < \frac{1}{2} + \sqrt{\frac{1}{4}+\frac{\delta}{p(n,\alpha)}}, \\
& \mbox{where  } p(n,\alpha) = \frac{\alpha}{B(\frac{n-1}{2}, \frac{1}{2})}\int_0^{1/\alpha} (1-\alpha^2 t^2)^{\frac{n-3}{2}} 
h(n,t) dt,  
\end{split}
\end{equation}
and $h(n,t)$ is defined in \eqref{eq:hdef}.
In particular, \eqref{eq:sepfixed} implies that $F$ is $\alpha$-Fisher separable with probability greater than $1-\delta$.
\end{theorem}

We next apply Theorem \ref{th:sepfixed} to some famous rotation invariant distributions.

\subsection{Uniform distribution in a ball}

We may assume that the ball has radius $1$.
Uniform distibution in the unit ball is given by \eqref{eq:sphinv} with $\rho(r) = 1, 0 \leq r \leq 1$.
Substituting this into \eqref{eq:hdef} and integrating, we get
$$
h(n,t) = 
\begin{cases} 
t^n/2, \quad\quad 0 \leq t \leq 1\\
1-t^{-n}/2, \quad 1\leq t.
\end{cases}
$$
Hence $p(n,\alpha)$ in \eqref{eq:sepfixed} is given by:
\begin{equation}\label{eq:ballnew}
\begin{split}
p(n,\alpha) = & \frac{\alpha}{2 B(\frac{n-1}{2}, \frac{1}{2})}\left(\int_0^1 (1-\alpha^2 t^2)^{\frac{n-3}{2}} 
t^n dt  \right. +  \\ 
&\left. \int_1^{1/\alpha} (1-\alpha^2 t^2)^{\frac{n-3}{2}} 
(2-t^{-n}) dt \right).
\end{split}
\end{equation}
Note that the answer may be written down explicitly using hypergeometric functions, but we find it more convenient to work with the integral expression.

With Theorem \ref{th:sepfixed}, this implies the following result.

\begin{theorem}\label{th:ballnew} 
Let $\alpha\in (0, 1]$, and let points $\boldsymbol{ x}_1, \dots, \boldsymbol{ x}_M$ be i.i.d points from uniform distribution in a ball. For any $\delta>0$, the expected number of $\alpha$-inseparable pairs from set $F=\left\{\boldsymbol{ x}_1, \dots, \boldsymbol{ x}_M\right\}$ is less than $\delta$ if and only if
$
M < \frac{1}{2} + \sqrt{\frac{1}{4}+\frac{\delta}{p(n,\alpha)}},
$
where $p(n,\alpha)$ is given by \eqref{eq:ballnew}.
In particular, in this case $F$ is $\alpha$-Fisher separable with probability greater than $1-\delta$.
\end{theorem}
 
To find the asymptotic growth of \eqref{eq:ballnew} as $n \to \infty$, we will use the Laplace's method. Informally, it states that, if function $h(t)$ has a unique maximum on $[a,b]$ attained at $t=c$, and $\phi(c) \neq 0$, then, for large $n$, the value of integral
\begin{equation}\label{eq:integral}
I(n) = \int_a^b \phi(t) e^{n h(t)} dt
\end{equation}
depends mainly on $\phi(c)$ and the behaviour of $h(t)$ is the neighbourhood of $c$. We can then replace in \eqref{eq:integral} $\phi(t)$ by $\phi(c)$ and $h(t)$ by its Taylor expansion at $t=c$ up to the first non-zero term, and integrate. We get
\begin{equation}\label{eq:Laplace1}
I(n) \sim \phi(c)e^{n h(c)}\sqrt{\frac{2\pi}{n|h''(c)|}}
\end{equation}
if $a<c<b$, $h'(c)=0$, $h''(c) \neq 0$,
$$
I(n) \sim \phi(c)e^{n h(c)}\sqrt{\frac{\pi}{2n|h''(c)|}}
$$
if $c=a$ or $c=b$, and $h'(c)=0$, $h''(c) \neq 0$, and
\begin{equation}\label{eq:Laplace3}
I(n) \sim \frac{\phi(c)e^{n h(c)}}{n |h'(c)|}
\end{equation}
if $c=a$ or $c=b$, and $h'(c) \neq 0$, we refer to \citet[Theorem 1, p. 58]{Wong} for a formal statement and proof.

Applying this method to \eqref{eq:ballnew}, we get the following estimate.

\begin{proposition}\label{prop:ballest} 
Let $p(n,\alpha)$ be given by \eqref{eq:ballnew} and $n>3$. 
\begin{itemize}
\item[I.] If $0 < \alpha < \frac{\sqrt{2}}{2}$, then
\begin{equation}\label{eq:ballest1}
\begin{split}
&p(n,\alpha) \leq q(n,\alpha) := \sqrt{\frac{1}{2\pi}}\frac{1}{\alpha(1-2\alpha^2)}\frac{n^{3/2}}{(n-3)^2}(1-\alpha^2)^{\frac{n+3}{2}}, \\ & \quad \text{and} \quad p(n,\alpha) \sim q(n,\alpha), 
\end{split}
\end{equation} 
\item[II.] If $\frac{\sqrt{2}}{2} < \alpha \leq 1$, then
\begin{equation}\label{eq:ballest2}
p(n,\alpha) \leq \frac{1}{2}(2\alpha)^{-n}, \quad \text{and} \quad p(n,\alpha) \sim \frac{1}{2}(2\alpha)^{-n} 
\end{equation}
(with equality for $\alpha=1$).
\end{itemize}
\end{proposition}
\begin{proof}
For the coefficient in \eqref{eq:ballnew}, \eqref{eq:Wendel} implies that
\begin{equation}\label{eq:coef}
\frac{\alpha}{2 B(\frac{n-1}{2}, \frac{1}{2})} \leq \frac{\alpha\sqrt{n}}{2\sqrt{2\pi}}, \quad \text{and} \quad \frac{\alpha}{2 B(\frac{n-1}{2}, \frac{1}{2})} \sim \frac{\alpha\sqrt{n}}{2\sqrt{2\pi}}.
\end{equation}
The first integral $I_1(n):=\int_0^1 (1-\alpha^2 t^2)^{\frac{n-3}{2}} 
t^n dt$ in \eqref{eq:ballnew} can be written in the form \eqref{eq:integral} with $h(t)=\frac{1}{2}\log(1-\alpha^2 t^2)+\log(t)$, $n := n-3$, $\phi(t)=t^3$. If $\frac{\sqrt{2}}{2} < \alpha \leq 1$, $h(t)$ attains maximum on $(0,1]$ at point $t=c:=\frac{1}{\sqrt{2}a}$, $h(c)=-\log(2\alpha)$, $h'(c)=0$, $h''(c)=-8\alpha^2$, $\phi(c)=(\sqrt{2}\alpha)^{-3}$, and \eqref{eq:Laplace1} implies that
\begin{equation}\label{eq:auxest1}
I_1(n) \sim (\sqrt{2}\alpha)^{-3}
e^{(n-3) (-\log(2\alpha))}
\sqrt{\frac{2\pi}{8(n-3)\alpha^2}} \sim \sqrt{\frac{2\pi}{n\alpha^2}}(2\alpha)^{-n}.
\end{equation}
If $0 < \alpha < \frac{\sqrt{2}}{2}$, $h(t)$ attains maximum on $(0,1]$ at point $t=1$, $h(1)=\frac{1}{2}\log(1-\alpha^2)$, $h'(1)=\frac{1-2\alpha^2}{1-\alpha^2}$, and inequalities $h(t) \leq h(1)+h'(1)(t-1), \, 0<t\leq 1$,  $\phi(t)=t^3 \leq 1, \, 0<t\leq 1$ imply that
\begin{equation}\label{eq:auxest2}
\begin{split}
I_1(t) \leq & \int_0^1 \exp\left[(n-3)(h(1)+h'(1)(t-1))\right] dt = \\
  & \frac{(1-\alpha^2)^{\frac{n-1}{2}}}{(1-2\alpha^2)(n-3)}\left(1-\exp\left(-\frac{(1-2\alpha^2)(n-3)}{1-\alpha^2}\right)\right) \leq \\ 
  &\frac{(1-\alpha^2)^{\frac{n-1}{2}}}{(1-2\alpha^2)(n-3)},
\end{split}
\end{equation}
and \eqref{eq:Laplace3} implies that
\begin{equation}\label{eq:auxest3}
I_1(t) \sim \frac{(1-\alpha^2)^{\frac{n-1}{2}}}{(1-2\alpha^2)(n-3)}.
\end{equation}
The second integral $I_2(n):=\int_1^{1/\alpha} (1-\alpha^2 t^2)^{\frac{n-3}{2}}(2-t^{-n})$ in \eqref{eq:ballnew} can be estimated similarly. Inequalities $\frac{1}{2}\log(1-\alpha^2t^2) \leq \frac{1}{2}\log(1-\alpha^2) - \frac{\alpha^2}{1-\alpha^2}(t-1)$ and $-\log(t)\geq -(t-1)$ imply that
\begin{equation}\label{eq:auxest4}
\begin{split}
&I_2(t) \leq \\
& \int_1^{1/\alpha}\!\!\! \exp\left[(n\!-\!3)\left(\frac{\log(1-\alpha^2)}{2} - \frac{\alpha^2(t-1)}{1-\alpha^2}\right)\right]\left(2 - e^{-n(t-1)}\right) dt =\\
  & (1\!-\! \alpha^2)^{\frac{n-1}{2}}
\left(\frac{2}{\alpha^2(n\!-\!3)}\!-\!\frac{1}{n\!-\!3\alpha^2}\!+\!\frac{\exp\left(\!-\frac{n-3\alpha^2}{\alpha+\alpha^2}\right)}{n-3\alpha^2}\!-\!\frac{\exp\left(\!-\frac{\alpha(n-3)}{1+\alpha}\right)}{\alpha^2(n-3)}\right)\! \leq \\
 & (1-\alpha^2)^{\frac{n-1}{2}}
\left(\frac{2}{\alpha^2(n-3)}-\frac{1}{n-3\alpha^2}\right) 
\leq (1-\alpha^2)^{\frac{n-1}{2}}\frac{(2-\alpha^2)n}{\alpha^2(n-3)^2},
\end{split}
\end{equation}
where the second inequality follows from the facts that $n-3\alpha^2 \geq \alpha^2(n-3)$ and $\exp\left(-\frac{n-3\alpha^2}{\alpha+\alpha^2}\right) \leq \exp\left(-\frac{\alpha(n-3)}{1+\alpha}\right)$.
Moreover, \eqref{eq:Laplace3} implies that the first inequality in \eqref{eq:auxest4} is asymptotically tight, and the asymptotic tightness of the second and third inequalities in \eqref{eq:auxest4} is straightforward, hence
\begin{equation}\label{eq:auxest5}
I_2(t) \sim (1-\alpha^2)^{\frac{n-1}{2}}\frac{(2-\alpha^2)n}{\alpha^2(n-3)^2}.
\end{equation}
If $0 < \alpha < \frac{\sqrt{2}}{2}$, the combination of \eqref{eq:coef}, \eqref{eq:auxest2}, and \eqref{eq:auxest4} yields
\begin{equation*}
\begin{split}
&p(n,\alpha) \leq \\
&\;\;\; \frac{\alpha\sqrt{n}}{2\sqrt{2\pi}}\frac{(1-\alpha^2)^{\frac{n-1}{2}}n}{(n-3)^2}\left(\frac{1}{1-2\alpha^2}\frac{n-3}{n} + \frac{2-\alpha^2}{\alpha^2}\right) \leq q(n,\alpha), 
\end{split}
\end{equation*}
where the second inequality follows from substituting $1$ instead of $\frac{n-3}{n}$ and simplifying. The $p(n,\alpha) \sim q(n,\alpha)$ part of \eqref{eq:ballest1} follow from \eqref{eq:coef}, \eqref{eq:auxest3}, and \eqref{eq:auxest5}.

The inequality in \eqref{eq:ballest2} follows from \eqref{eq:2pointsphere}. For $\frac{\sqrt{2}}{2} < \alpha \leq 1$, the $\sim$ part of \eqref{eq:ballest2} follows from \eqref{eq:coef}, \eqref{eq:auxest1}, and \eqref{eq:auxest5}.
\end{proof}

We conjecture that factor $\frac{n^{3/2}}{(n-3)^2}$ in \eqref{eq:ballest1} can be improved to a simpler factor $\sqrt{n}$, which would allow to remove the condition $n>3$, but this improvement is negligible for large $n$, and the $\sim$ part of \eqref{eq:ballest1} implies that asymptotically non-negligible improvement is impossible, and bound \eqref{eq:ballest1} is essentially the best possible if $0 < \alpha < \frac{\sqrt{2}}{2}$. Similarly, bound \eqref{eq:ballest2} is essentially the best possible if $\frac{\sqrt{2}}{2} < \alpha \leq 1$. 

Theorem \ref{th:ballnew} and Proposition \ref{prop:ballest} imply the following corollary. 
 
\begin{corollary}\label{cor:ballnewsimple} 
Let $\alpha\in \left(0, \frac{1}{\sqrt{2}}\right)$, $n>3$, and let  $F=\{\boldsymbol{ x}_1, \dots, \boldsymbol{ x}_M\}$ be the set of $M$ i.i.d points from uniform distribution in a ball. For any $\delta>0$, if
\begin{equation}\label{eq:ballnewsimple}
M \leq \sqrt[4]{2\pi}\sqrt{\delta \alpha (1-2\alpha^2)}\frac{n-3}{n^{3/4}}\left(\frac{1}{1-\alpha^2}\right)^{\frac{n+3}{4}},
\end{equation}
then the expected number of $\alpha$-inseparable pairs in $F$ is less than $\delta$. In particular, \eqref{eq:ballnewsimple} implies that $F$ is $\alpha$-Fisher separable with probability greater than $1-\delta$.
\end{corollary}

Proposition \ref{prop:ballest} implies that the bound for $M$ in Corollary \ref{cor:ballnewsimple} is asymptotically tight, and has the advantage of being a simple explicit formula. For $\alpha \geq \frac{1}{\sqrt{2}}$, \eqref{eq:ballest2} implies that an asymptotically tight bound is given in Theorem \ref{th:ballknown}. 


\begin{example}\label{ex:ballsimple}
Let $\delta=0.01$. Table~\ref{Table7} shows the upper bounds on $M$ in Corollary~\ref{cor:ballnewsimple} for $\alpha = 0.5, 0.6$ and $0.7$ in various dimensions $n$.
\begin{table}[h]
\begin{center}
\caption{\label{Table7} The upper bound of $M$ (\ref{eq:ballnewsimple}) that guarantees $\alpha$-Fisher separability with probability $p>0.99$ of  $M$ i.i.d. points sampled from uniform distribution in a ball, for various $\alpha$ and dimensions.}
\begin{tabular}{ |c|c|c|c| } 
 \hline
  & $\alpha=0.5$ & $\alpha=0.6$ & $\alpha=0.7$\\ 
 \hline
 $n=10$ & $0.25$ & $0.34$ & $0.2$ \\ 
 $n=50$ & $8.9$ & $60$ & $350$\\ 
 $n=100$ & $400$ & $19,491$ & $1.9 \cdot 10^{6}$\\ 
 $n=200$ & $642,645$ & $1.6 \cdot 10^9$ & $4.8 \cdot 10^{13}$\\ 
 $n=500$ & $1.9 \cdot 10^{15}$ & $7.1 \cdot 10^{23}$ & $5.2 \cdot 10^{35}$\\ 
 $n=1000$ & $9.4 \cdot 10^{30}$ & $1.4 \cdot 10^{48}$ & $2.2 \cdot 10^{72}$\\ 
 \hline
\end{tabular}
\end{center}
\end{table}

For example, for $n=200$, we see that $642,645$ points from the uniform distribution in the unit ball are Fisher-separable at level $\alpha=0.5$ with probability greater than $99\%$. For comparison, with the same parameters, the bound \eqref{eq:ballknown} in Theorem \ref{th:ballknown} reduces to $M < 0.14$. The optimal bound in Theorem \ref{th:ballnew} gives $M<661,243$.
\end{example}

The results of this section can be straightforwardly extended to the case when the points in ${\mathbb R}^n$ are selected from the uniform distribution in a spherical layer, that is, from the distribution \eqref{eq:sphinv} with 
$$
\rho(r) = 
\begin{cases} 
\frac{1}{1-R}, \quad R \leq r \leq 1,\\
0, \quad\quad \text{otherwise},
\end{cases}
$$
where $R\in(0,1)$ is a parameter. Then $h(n,t)$ in \eqref{eq:hdef} is given by
\begin{equation}\label{eq:layer}
h(n,t) = 
\begin{cases} 
0, \quad\quad\quad\quad\quad\quad t \leq R\\
\frac{t^{-n}(R^n-t^n)^2}{2(1-R^n)^2}, \quad\quad\,\,\,\,\, R < t \leq 1,\\
1-\frac{t^n(R^n-t^{-n})^2}{2(1-R^n)^2}, \quad\,\, 1 \leq t < 1/R,\\
1, \quad\quad\quad\quad\quad\quad 1/R < t.
\end{cases}
\end{equation}

With Theorem \ref{th:sepfixed}, this implies that if $M$ i.i.d. points are selected from this distribution, then the expected number of $\alpha$-inseparable pairs is less than $\delta$ if and only if
$$
M < \frac{1}{2} + \sqrt{\frac{1}{4}+\frac{\delta}{p(n,\alpha)}},
$$
where $p(n,\alpha)$ is given by \eqref{eq:sepfixed} with $h(n,t)$ given by \eqref{eq:layer}. 

Some (weaker) estimates for spherical layer were received earlier by  \citet{Sidorov2020}.

\subsection{Multivariate exponential distribution}

By multivariate exponential distribution in ${\mathbb R}^n$ we will mean rotation invariant distribution such that $\rho(||{\boldsymbol x}||)$ in \eqref{eq:sphinv} is equal to $\exp(-||{\boldsymbol x}||)$. In this case, the distribution of $||{\boldsymbol x}||$ is the standard Gamma distribution with $n$ degrees of freedom, and, for i.i.d. ${\boldsymbol x}$ and ${\boldsymbol y}$, ratio $\frac{{\boldsymbol x}}{{\boldsymbol y}}$ follows beta prime distribution, that is, 
$$
{\mathbb P}\left[\frac{||{\boldsymbol x}||}{||{\boldsymbol y}||} \leq t\right] = I_{\frac{t}{1+t}}(n,n),
$$
where $I_z(a,b)$ is the regularized incomplete beta function defined in \eqref{eq:regincbeta}. 
Hence, Theorem \ref{th:sepfixed} implies the following result.

\begin{theorem}\label{th:expon} 
Let $\alpha\in (0, 1]$, and let points $\boldsymbol{ x}_1, \dots, \boldsymbol{ x}_M$ be i.i.d points from exponential distribution in ${\mathbb R}^n$. For any $\delta>0$, the expected number of $\alpha$-inseparable pairs from set $F=\left\{\boldsymbol{ x}_1, \dots, \boldsymbol{ x}_M\right\}$ is less than $\delta$ if and only if
\begin{equation}\label{eq:expon}
\begin{split}
&M < \frac{1}{2}+\sqrt{\frac{1}{4}+\frac{\delta}{p(n,\alpha)}}, \\ & \mbox{where   } p(n,\alpha) = \frac{\alpha}{B(\frac{n-1}{2}, \frac{1}{2})}\int_0^{1/\alpha} (1-\alpha^2 t^2)^{\frac{n-3}{2}} 
I_{\frac{t}{1+t}}(n,n) dt. 
\end{split}
\end{equation}
In particular, \eqref{eq:expon} implies that $F$ is $\alpha$-Fisher separable with probability greater than $1-\delta$.
\end{theorem}

\begin{example}\label{ex:expon}
Let $\delta=0.01$. Table~\ref{Table8} shows the upper bounds on $M$ in Theorem~\ref{th:expon} for $\alpha = 0.6, 0.8$ and $1$ in various dimensions $n$.
\begin{table}[h]
\begin{center}
\caption{\label{Table8} The upper bounds on $M$ (\ref{eq:expon})  that guarantees $\alpha$-Fisher separability of $M$ i.i.d. points from exponential distribution with probability $p>0.99$ for various $\alpha$ and dimensions.}
\begin{tabular}{ |c|c|c|c| } 
 \hline
  & $\alpha=0.6$ & $\alpha=0.8$ & $\alpha=1$\\ 
 \hline
 $n=10$ & $0.65$ & $0.81$ & $1.06$ \\ 
 $n=50$ & $7.6$ & $43$ & $249$\\ 
 $n=100$ & $218$ & $6,662$ & $203,805$\\ 
 $n=200$ & $154,501$ & $1.3 \cdot 10^8$ & $1.1 \cdot 10^{11}$\\ 
 $n=500$ & $4.1 \cdot 10^{13}$ & $7.6 \cdot 10^{20}$ & $1.6 \cdot 10^{28}$\\ 
 $n=1000$ & $3.8 \cdot 10^{27}$ & $1.1 \cdot 10^{42}$ & $4.8 \cdot 10^{56}$\\ 
 \hline
\end{tabular}
\end{center}
\end{table}
For example, for $n=100$, we see that over $200,000$ points from the multivariate exponential distribution are Fisher-separable at level $\alpha=1$ with probability greater than $99\%$. In dimension $n=200$, over $150,000$ points from the same distribution become Fisher-separable at level $\alpha=0.6$.
\end{example}

The growth of factor $I_{\frac{t}{1+t}}(n,n)$ in \eqref{eq:expon} is described by the following proposition.

\begin{proposition}\label{prop:ibound} 
For any $t>0$ and $n\geq 1$,
\begin{equation}\label{eq:ialt}
I_{\frac{t}{1+t}}(n,n) = 
\begin{cases}
\frac{1}{2}I_{\frac{4t}{(1+t)^2}}(n,\frac{1}{2}), \quad 0<t \leq 1,\\
1-\frac{1}{2}I_{\frac{4t}{(1+t)^2}}(n,\frac{1}{2}), 1 \leq t.
\end{cases}
\end{equation}
In particular,
\begin{equation}\label{eq:ibound}
I_{\frac{t}{1+t}}(n,n) \leq \frac{1}{2\sqrt{\pi n}}\frac{1+t}{1-t}\left(\frac{4t}{(1+t)^2}\right)^n, \quad 0 < t < 1,
\end{equation}
and this upper bound is asymptotically tight if $t$ is fixed and $n \to \infty$.
\end{proposition}
\begin{proof}
If $0<t\leq 1$, then $z=\frac{t}{1+t} \leq \frac{1}{2}$, and
\begin{equation*}
\begin{split}
B_z(n, n)=&\int_0^z u^{n-1}(1-u)^{n-1}du = \\ &\int_0^{4z(1-z)} (s/4)^{n-1}\frac{ds}{4\sqrt{1-s}} = 4^{-n}B_{4z(1-z)}\left(n,\frac{1}{2}\right) ,
\end{split}
\end{equation*}
where the second inequality is the change of variables $s=4u(1-u)$. Next,
\begin{equation*}
\begin{split}
B(n, n) = & \int_0^1 u^{n-1}(1-n)^{n-1}dt = \\ & 2\int_0^{1/2} u^{n-1}(1-n)^{n-1}dt = 2 B_{\frac{1}{2}}(n, n) = 4^{-n}B\left(n,\frac{1}{2}\right),
\end{split}
\end{equation*}
hence
$$
I_z(n,n)=\frac{B_z(n, n)}{B(n, n)} = \frac{4^{-n}B_{4z(1-z)}\left(n,\frac{1}{2}\right)}{4^{-n}B\left(n,\frac{1}{2}\right)} = I_{4z(1-z)}\left(n,\frac{1}{2}\right),
$$
which with $z=\frac{t}{1+t}$ implies the first line of \eqref{eq:ialt}. The second line follows from the first one and the identity $I_z(a,b)=1-I_{1-x}(b,a)$.

For $0 < t < 1$, \eqref{eq:iest} with $a=n$, $b=1/2$, and $z=\frac{4t}{(1+t)^2}$ implies that
\begin{equation*}
\begin{split}
&I_{\frac{t}{1+t}}(n,n) = \frac{1}{2}I_{\frac{4t}{(1+t)^2}}(n,\frac{1}{2}) \leq \\
&\;\;\;\;\; \left(\frac{4t}{(1+t)^2}\right)^n \left(1-\frac{4t}{(1+t)^2}\right)^{-1/2} \frac{n^{-1/2}}{2\Gamma(1/2)},
\end{split}
\end{equation*}
which simplifies to the right-hand side of \eqref{eq:ibound}.
\end{proof}

The next proposition establishes asymptotic growth of $p(n,\alpha)$ in \eqref{eq:expon} as $n\to\infty$.

\begin{proposition}\label{prop:expest} 
Let $p(n,\alpha)$ be given by \eqref{eq:expon}. Then
\begin{equation}\label{eq:expest}
\begin{split}
&p(n,\alpha) \sim \\
& \frac{\sqrt{1+5\alpha^2+(1+\alpha^2)\sqrt{1+8\alpha^2}}}{2\alpha\sqrt{\pi n}\sqrt[4]{1+8\alpha^2}}\left(\frac{4\sqrt{2}\alpha(\sqrt{1+8\alpha^2}-1)}{(\sqrt{1+8\alpha^2}+4\alpha^2-1)^{3/2}} \right)^n. 
\end{split}
\end{equation} 
\end{proposition}
\begin{proof}
Proposition \ref{prop:ibound} together with obvious bound $I_{\frac{t}{1+t}}(n,n) \leq 1$ implies that the integral in \eqref{eq:expon} is bounded by
\begin{equation}\label{eq:auxstep}
I_n := \int_0^{1/\alpha} (1-\alpha^2 t^2)^{\frac{n-3}{2}} I_{\frac{t}{1+t}}(n,n) dt \leq J_n + L_n,
\end{equation}
where
\begin{equation*}
\begin{split}
&J_n:=\int_0^{1} (1-\alpha^2 t^2)^{\frac{n-3}{2}} \min\left\{\frac{1}{2\sqrt{\pi n}}\frac{1+t}{1-t}\left(\frac{4t}{(1+t)^2}\right)^n,1\right\} dt, \\
& L_n := \int_1^{1/\alpha} (1-\alpha^2 t^2)^{\frac{n-3}{2}} dt.
\end{split}
\end{equation*}
Because for $n \geq 3$ we have $J_n \leq 1$ and $L_n \leq \frac{1}{\alpha}-1$, and \eqref{eq:ibound} is asymptotically tight, \eqref{eq:auxstep} is also asymptotically tight by dominated convergence theorem. 

Integral $J_n$ can be written in the form \eqref{eq:integral} with 
$$
\phi_n(t):=(1-\alpha^2 t^2)^{-3/2}\min\left\{\frac{1}{2\sqrt{\pi n}}\frac{1+t}{1-t},\left(\frac{(1+t)^2}{4t}\right)^n\right\}
$$ 
and 
$$
h(t):=\log\left(\sqrt{1-\alpha^2 t^2}\frac{4t}{(1+t)^2}\right).
$$

Function $h(t)$ attains maximum on $(0,1)$ at 
$$
t_0 = t_0(a) := \frac{\sqrt{1+8a^2}-1}{4a^2}
$$
and by \eqref{eq:Laplace1}
\begin{equation*}
\begin{split}
J_n \sim & 
\phi_n(t_0)\sqrt{\frac{2\pi}{n|h''(t_0)|}}e^{n h(t_0)}
\sim \\
&(1-\alpha^2 t_0^2)^{-3/2}\frac{1}{n}\frac{1+t_0}{1-t_0}\sqrt{\frac{1}{2|h''(t_0)|}}e^{n h(t_0)},
\end{split}
\end{equation*}
where the second $\sim$ follows from the fact that $\frac{1}{2\sqrt{\pi n}}\frac{1+t_0}{1-t_0} <\left(\frac{(1+t_0)^2}{4t_0}\right)^n$ for large $n$.

Similarly, by \eqref{eq:Laplace3}
$$
L_n \sim \frac{1-\alpha^2}{n\alpha^2}\left(\sqrt{1-\alpha^2}\right)^{(n-3)}.
$$
Because $e^{h(t_0)} > \sqrt{1-\alpha^2}$ for all $0<\alpha \leq 1$, $I_n \sim J_n + L_n \sim J_n$ as $n \to \infty$.
This together with \eqref{eq:expon} and \eqref{eq:coef} implies that
$$
p(n,\alpha) \sim \frac{\alpha\sqrt{n}}{\sqrt{2\pi}}(1-\alpha^2 t_0^2)^{-3/2}\frac{1}{n}\frac{1+t_0}{1-t_0}\sqrt{\frac{1}{2|h''(t_0)|}}e^{n h(t_0)},
$$
which simplifies to \eqref{eq:expest}.
\end{proof}

For $\alpha=1$, Theorem \ref{th:expon} and Proposition \ref{prop:expest} imply the following corollary.

\begin{corollary}\label{cor:expon1} 
Let points $\boldsymbol{ x}_1, \dots, \boldsymbol{ x}_M$ are i.i.d points from exponential distribution in ${\mathbb R}^n$. For any $\delta>0$, if
$$
M < \sqrt{\frac{\delta}{p(n,1)}} \sim \sqrt{\delta}\sqrt[4]{\pi n}\left(\frac{\sqrt[4]{27}}{2}\right)^n,
$$
then the expected number of $1$-inseparable pairs in set $F=\{\boldsymbol{ x}_1, \dots, \boldsymbol{ x}_M\}$ is less than $\delta$. In particular, set $F$ is $1$-Fisher separable with probability greater than $1-\delta$.
\end{corollary}
 
In particular,
$$
b(1) = \log\left(\frac{\sqrt[4]{27}}{2}\right)=0.1308...,
$$
where $b(\alpha)$ is defined in \eqref{eq:bdef}.
For comparison, for uniform distribution in a ball $b(1)=\frac{1}{2}\log 2 = 0.3465...$, while for standard normal distribution $b(1)=\frac{1}{4}\log 2 = 0.1732...$.

\subsection{General log-concave spherically invariant distribution}

This section derives separation theorems for arbitrary spherically invariant distribution. We start with the following easy result.

\begin{theorem}\label{th:rotsimple} 
Let $\delta>0$, $\alpha \in (1/2,1]$, and let $\{\boldsymbol{x}_1, \ldots , \boldsymbol{x}_M\}$ be a set of $M$ i.i.d. random points from a spherically invariant log-concave distribution in ${\mathbb R}^n$. If 
\begin{equation}\label{eq:rotsimple}
M < \sqrt{\frac{\delta}{2}} \exp\left(n\frac{(2\alpha-1)^2}{8(2\alpha+1)^2}\right), 
\end{equation}
then the expected number of $\alpha$-inseparable pairs in set $F=\{\boldsymbol{ x}_1, \dots, \boldsymbol{ x}_M\}$ is less than $\delta$. In particular, set $F$ is $\alpha$-Fisher separable with probability greater than $1-\delta$.
\end{theorem}
\begin{proof}
 Let ${\boldsymbol x}$ and ${\boldsymbol y}$ be two i.i.d points from the given distribution. Inequality \eqref{eq:Fisher} can be rewritten as 
$$
\left\Vert{\boldsymbol x} - \frac{{\boldsymbol y}}{2\alpha}\right\Vert \leq \left\Vert\frac{{\boldsymbol y}}{2\alpha}\right\Vert,
$$
that is, ${\boldsymbol x}$ belongs to a ball of radius $\left\Vert\frac{{\boldsymbol y}}{2\alpha}\right\Vert$. For every fixed $t>0$, this may happen if either 
\begin{itemize}
\item[(i)] $||{\boldsymbol y}||>t$, or 
\item[(ii)] ${\boldsymbol x}$ belongs to a ball of radius at most $\frac{t}{2\alpha}$.
\end{itemize} 
We will prove that, for $t=\frac{4\alpha}{2\alpha+1}\mu$, where $\mu = {\mathbb E}[||{\boldsymbol x}||] = {\mathbb E}[||{\boldsymbol y}||]$, both these possibilities can happen with probability at most $e^{-n\frac{(2\alpha-1)^2}{4(2\alpha+1)^2}}$. With \eqref{eq:MBound}, this will imply \eqref{eq:rotsimple}. 

As observed by \citet[p. 328]{Bobkov}, if random variable ${\boldsymbol x}$ has density given by \eqref{eq:sphinv} with log-concave $\rho$, then $||{\boldsymbol x}||$ has log-concave distribution of order $n$ (that is, has density of the form $q(r) = r^{n-1} \rho(r)$ for log-concave $\rho$). According to  \citet[Corollary 3.2]{Bobkov}, this implies that, for any $h \in [0.1]$,
\begin{equation}\label{eq:normupper}
{\mathbb P}[||{\boldsymbol x}|| - \mu \geq h\mu ] \leq e^{-nh^2/4}
\end{equation}
and
\begin{equation}\label{eq:normlower}
{\mathbb P}[\mu - ||{\boldsymbol x}|| \geq h\mu ] \leq e^{-nh^2/4}.
\end{equation}
With $h=\frac{2\alpha-1}{2\alpha+1}$, \eqref{eq:normupper} implies that probability of (i) is at most $e^{-n\frac{(2\alpha-1)^2}{4(2\alpha+1)^2}}$, while \eqref{eq:normlower} implies that  
$$
{\mathbb P}\left[||{\boldsymbol x}|| \leq \frac{2}{2\alpha+1}\mu \right] \leq e^{-n\frac{(2\alpha-1)^2}{4(2\alpha+1)^2}}.
$$
In other words, the probability that ${\boldsymbol x}$ belongs to a ball $B$ of radius $\frac{t}{2\alpha}$ centred at origin is at most $e^{-n\frac{(2\alpha-1)^2}{4(2\alpha+1)^2}}$. However, because the density $\hat{\rho}$ is spherically invariant and log-concave, we have $\hat{\rho}(x) \geq \hat{\rho}(y)$ for every $x \in B$ and $y \not \in B$, hence shifting the ball cannot increase the probability for a point to belong to it. 
\end{proof}

The bound \eqref{eq:rotsimple} in Theorem \ref{th:rotsimple} is simple and explicit. For example, for $\alpha=1$ it reduces to
\begin{equation}\label{eq:salpha1}
M < \sqrt{\frac{\delta}{2}} \exp\left(\frac{n}{72}\right). 
\end{equation}
However, the bound is far from being optimal, and the Theorem is not applicable for $\alpha \leq \frac{1}{2}$. We next prove a separation theorem with more complicated but better bound. It also applies to a broader class of distributions, because it does not requires for $\rho$ in \eqref{eq:sphinv} to be non-increasing.

\begin{theorem}\label{th:rotgeneral} 
Let $\delta>0$, $\alpha \in (0,1]$, and let $F=\{\boldsymbol{x}_1, \ldots , \boldsymbol{x}_M\}$ be a set of $M$ i.i.d. random points from a distribution in ${\mathbb R}^n$ given by \eqref{eq:sphinv} with log-concave $\rho$. If 
\begin{equation}\label{eq:genbound}
M < \sqrt{\delta} f(n,\alpha)^{-1/2}, 
\end{equation}
where $f(n,\alpha)$ is an explicit function defined in formulas \eqref{eq:psidef}-\eqref{eq:fdef} below,
then the expected number of $\alpha$-inseparable pairs in set $F=\{\boldsymbol{ x}_1, \dots, \boldsymbol{ x}_M\}$ is less than $\delta$. In particular, \eqref{eq:genbound} implies that set $F$ is $\alpha$-Fisher separable with probability greater than $1-\delta$.
\end{theorem}
\begin{proof}
Let ${\boldsymbol x}$ and ${\boldsymbol y}$ be any i.i.d. points from the given distribution.
Let us derive an upper bound for ${\mathbb P}\left[\frac{||{\boldsymbol x}||}{||{\boldsymbol y}||}\leq t\right]$ for any $t \in (0,1/\alpha)$. Let $q(.)$ be the density for absolute value distribution. We have
\begin{equation*}
\begin{split}
{\mathbb P}\left[\frac{||{\boldsymbol x}||}{||{\boldsymbol y}||} \leq t\right] = & \int_0^{\infty} q(x) dx \int_{x/t}^{\infty} q(y) dy = \\
&\int_0^{\infty} q(x) dx \cdot {\mathbb P}\left[||{\boldsymbol y}||\geq \frac{x}{t}\right].
\end{split}
\end{equation*}
We claim that
\begin{equation}\label{eq:psidef}
{\mathbb P}\left[||{\boldsymbol y}||\geq \frac{x}{t}\right] \leq \psi_{n,t}(x):= \begin{cases}
1, \quad x \leq t,\\
\exp\left(-\frac{n(x-t)^2}{4t^2}\right), \quad t \leq x \leq 2t\\
\exp\left(-\frac{nx}{8t}\right), \quad 2t \leq x.
\end{cases} 
\end{equation}
Indeed, the first line in \eqref{eq:psidef} is trivial. If $t \leq x \leq 2t$, then, applying \eqref{eq:normupper} with $\mu=1$ and $h=\frac{x-t}{t}$, we get the second line in \eqref{eq:psidef}. Further, equation (3.9) in the cited work \citep{Bobkov} states that
$$
{\mathbb P}\left[||{\boldsymbol y}||\geq h \mu\right] \leq \exp\left(-\frac{nh}{8}\right), \quad h \geq 2.
$$
Applying this with $\mu=1$ and $h=\frac{x}{t}$, we get the third line in \eqref{eq:psidef}. With \eqref{eq:psidef},
\begin{equation*}
\begin{split}
{\mathbb P}\left[\frac{||{\boldsymbol x}||}{||{\boldsymbol y}||} \leq t\right] \leq &\int_0^{\infty} q(x) \psi_{n,t}(x) dx = \\ 
& \int_0^{\infty} {\mathbb P}[||{\boldsymbol x}|| \leq x] (-\psi'_{n,t}(x)) dx,
\end{split}
\end{equation*}
where the last equality is integration by parts. Because $\psi_{n,t}(x)$ is non-increasing, $-\psi'_{n,t}(x)$ is non-negative, and ${\mathbb P}[||{\boldsymbol x}|| \leq x]$ can be bounded by
\begin{equation}\label{eq:gdef}
{\mathbb P}[||{\boldsymbol x}|| \leq x] \leq g_n(x) := \begin{cases}
\exp\left(-\frac{n(1-x)^2}{4}\right), \quad x \leq 1,\\
1, \quad x \geq 1,
\end{cases} 
\end{equation}
where the first line in \eqref{eq:gdef} follows from \eqref{eq:normlower} with $\mu=1$ and $h=1-x$. Hence,
$$
{\mathbb P}\left[\frac{||{\boldsymbol x}||}{||{\boldsymbol y}||} \leq t\right] \leq \phi(t,n) := \int_0^{\infty} g_n(x) (-\psi'_{n,t}(x)) dx.
$$

Applying this bound to \eqref{eq:pest}, we get
\begin{equation}\label{eq:fdef}
p \leq f(n,\alpha) := \frac{\alpha}{B(\frac{n-1}{2}, \frac{1}{2})}\int_0^{1/\alpha} (1-\alpha^2 t^2)^{\frac{n-3}{2}}\phi(t,n) dt,
\end{equation}
and \eqref{eq:genbound} follows from \eqref{eq:MBound}.
\end{proof}

The function $f(n,\alpha)$ in \eqref{eq:genbound} is complicated but explicit and, for any specific values of $n$ and $\alpha$, can be easily computed in any package like Mathematica. In particular, we verified in Mathematica that
$$
\frac{-\log f(n,1)}{n} \geq 0.14, \quad 1\ \leq n \leq 4000.
$$
This together with Theorem \ref{th:rotgeneral} implies the following Corollary.

\begin{corollary}\label{cor:alpha1} 
Let $\delta>0$, and let $F=\{\boldsymbol{ x}_1, \dots, \boldsymbol{ x}_M\}$ be a set of $M$ i.i.d. random points from a distribution in ${\mathbb R}^n$ given by \eqref{eq:sphinv} with log-concave $\rho$. If $1 \leq n \leq 4000$ and
\begin{equation}\label{eq:alpha1}
M < \sqrt{\delta} \exp\left(0.07 n\right) 
\end{equation}
then the expected number of $1$-inseparable pairs in set $F$ is less than $\delta$. In particular, \eqref{eq:alpha1} implies that set $F$ is $1$-Fisher separable with probability greater than $1-\delta$.
\end{corollary}

If $n > 4000$, then we can use \eqref{eq:salpha1} and get the bound much higher than needed for any practical purposes. However, for smaller $n$, Corollary \ref{cor:alpha1} is a significant improvement comparing to \eqref{eq:salpha1}. 

\begin{example}
Let $\alpha=1$ and $\delta=0.01$.
\begin{itemize}
\item[(a)] If $n=4001$, then \eqref{eq:salpha1} reduces to $M<96,158,590,065,160,622,896,817$;
\item[(b)] If $n=400$, then \eqref{eq:salpha1} reduces to $M \leq 18$, while \eqref{eq:alpha1} reduces to $M \leq 144,625,706,429$;
\item[(c)] If $n=200$, \eqref{eq:alpha1} still gives a reasonable bound $M \leq 120,260$.
\end{itemize}
\end{example}

\begin{example}\label{ex:alpha1}
Let $\alpha=1$ and $\delta=0.01$. Table~\ref{Table9} shows the upper bounds on $M$ in Corollary \ref{cor:alpha1} in various dimensions $n$.
\begin{table}[h]
\begin{center}
\caption{\label{Table9}  The upper bounds on $M$ in Corollary \ref{cor:alpha1} in various dimensions $n$ for $\alpha=1$ and $\delta=0.01$.}
\begin{tabular}{ |c|c| } 
 \hline
 $n$ & $M\leq$\\ 
 \hline
 $10$ & $0.2$ \\ 
 $50$ & $3.3$\\ 
 $100$ & $109$\\ 
 $200$ & $120,260$\\ 
 $500$ & $1.5 \cdot 10^{14}$\\ 
 $1000$ & $2.5 \cdot 10^{29}$\\ 
 \hline
\end{tabular}
\end{center}
\end{table}
\end{example}

Corollary \ref{cor:expon1} demonstrates that constant $0.07$ in Corollary \ref{cor:alpha1} is within a factor less than $2$ from being optimal.

\section{Improved bounds for product distributions in the unit cube}\label{sec:product}

\subsection{The general case}\label{sec:gen}

In this section we assume the following.

\begin{itemize}
\item[(a)] all points in a finite set $F$ are chosen independently;
\item[(b)] points in $F$ are not necessary identically distributed, but have the same mean $\boldsymbol{ \mu}=(\mu_1, \dots, \mu_n) \in {\mathbb R}^n$;
\item[(c)] for each $\boldsymbol{ x}=(x_1, \dots, x_n) \in F$, components $x_1, \dots, x_n$ are independent and have $[0,1]$ support;
\item[(d)] there are no point $\boldsymbol{ x} \in F$ such that ${\mathbb P}[\boldsymbol{ x} = \boldsymbol{ \mu}] = 1$.
\end{itemize}

From (c), $F$ is a subset of the unit cube $U_n = [0,1]^n$. 
From (b), $E[x_i]=\mu_i$ for all $i=1,\dots,n$ and for all $\boldsymbol{ x} \in F$. Let
$$
\sigma_0^2 = \min_{\boldsymbol{ x} \in F} \left( \frac{1}{n} \sum_{i=1}^n Var[x_i] \right),
$$
that is, the minimal value of average variance of the components. From (d), $\sigma_0^2 > 0$.

Fix any point $\boldsymbol{ c}=(c_1,\dots,c_n) \in U_n$,
and any pair $\boldsymbol{ x},\boldsymbol{ y} \in F$. Let
\begin{equation}\label{eq:zidef}
z_i = (x_i-c_i)(y_i-c_i) - \alpha (x_i-c_i)^2, \quad i=1,\dots,n.
\end{equation}
Inequality \eqref{eq:2pointBound} reduces to
$$
{\mathbb P}\left[\sum_{i=1}^n z_i \geq 0 \right] \leq f(n,\alpha).
$$
From (a) and (c) it follows that all random variables $z_i$ are independent.
Next,
$$
E[z_i]=E[(x_i-c_i)(y_i-c_i)]-\alpha E[(x_i-c_i)^2].
$$
By independence, $E[(x_i-c_i)(y_i-c_i)]=E[(x_i-c_i)]E[(y_i-c_i)]=(\mu_i-c_i)^2$,
and
$
E[z_i]=(\mu_i-c_i)^2-\alpha E[(x_i-c_i)^2] = (1-\alpha)(\mu_i-c_i)^2 - \alpha Var[x_i-c_i].
$
Hence,
$$
E\left[ \sum_{i=1}^n z_i \right] = (1-\alpha) \sum_{i=1}^n (\mu_i-c_i)^2 - \alpha \sum_{i=1}^n Var[x_i] \leq 
$$
$$
\leq (1-\alpha) \sum_{i=1}^n (\mu_i-c_i)^2 - n \alpha \sigma_0^2 = - nt, 
$$
where
\begin{equation}\label{eq:tDef}
t := \alpha \sigma_0^2 - (1-\alpha) \frac{1}{n}\sum_{i=1}^n (\mu_i-c_i)^2.
\end{equation}
Note that $t$ is guaranteed to be positive if either (i) $\alpha$ is sufficiently close to $1$, or (ii) $\boldsymbol{ c}=\boldsymbol{ \mu}$.

The following Proposition established bounds on $z_i$.

\begin{proposition}\label{prop:zibounds}
Let $c'_i := \max\{c_i, 1-c_i\}$ for all $i$, and $f(c):=-c+c^2(1-\alpha)$.
\begin{itemize}
\item[(i)] if $\alpha\geq 0.5$, then 
$$
-c'_i + (c'_i)^2(1-\alpha) \leq z_i \leq \frac{(c'_i)^2}{4\alpha}. 
$$
In particular, $-\alpha \leq z_i \leq \frac{1}{4\alpha}$ for all $i$;
\item[(ii)] if $\alpha\leq 0.5$, then 
$$
\min\{f(c_i), f(1-c_i)\}\leq z_i \leq (1-\alpha)(c'_i)^2. 
$$
In particular, $-\frac{1}{4(1-\alpha)} \leq z_i \leq 1-\alpha$ for all $i$;
\item[(iii)] if $c_i=\frac{1}{2}$ for all $i$, then $-\frac{1}{2} + \frac{1}{4}(1-\alpha)\leq z_i \leq \frac{1}{16\alpha} $ for all $i$ if $\alpha \geq 0.5$ and $-\frac{1}{2} + \frac{1}{4}(1-\alpha)\leq z_i \leq \frac{1}{4}(1-\alpha)$ for all $i$ if $\alpha \leq 0.5$.  
\end{itemize}
\end{proposition}
\begin{proof}
For each fixed $c_i$ and $y_i$, $z_i$ in \eqref{eq:zidef} is maximized if $x_i=\frac{y_i-c_i}{2\alpha}+c_i$, resulting in $z_i = \frac{(y_i-c_i)^2}{4\alpha}$. The last expression is maximized if $y_i$ is either $0$ or $1$, with maximum equal to $\frac{(c'_i)^2}{4\alpha}$. This bound is tight if $\alpha \geq 0.5$. If $\alpha < 0.5$, then critical point $\frac{y_i-c_i}{2\alpha}+c_i$ lies outside of $[0,1]$ and \eqref{eq:zidef} is maximized if $x_i$ is either $0$ or $1$, resulting in bound $(1-\alpha)(c'_i)^2$. 

Similarly, $z_i$ in \eqref{eq:zidef} is minimized when either $x_i=1$ and $y_i=0$ or vice versa, resulting in bound $\min\{f(c_i), f(1-c_i)\}\leq z_i$. Because $f(c) \geq -\frac{1}{4(1-\alpha)}$ for all $c$, bound $-\frac{1}{4(1-\alpha)} \leq z_i$ follows. If $\alpha \geq 0.5$, then $f$ is monotone decreasing on $[0,1]$, hence $\min\{f(c_i), f(1-c_i)\}= f(c'_i) = -c'_i + (c'_i)^2(1-\alpha) \geq f(1) = -\alpha$. 
\end{proof}

Let $S_n = \sum_{i=1}^n z_i$. By Hoeffding's inequality \citep{Hoeffding}, \citep[Theorem 2.8]{Boucheron2013},  
$$
{\mathbb P}[S_n \geq 0] \leq {\mathbb P}[S_n - E[S_n] \geq nt] \leq \exp\left(-\frac{2n^2t^2}{\sum\limits_{i=1}^n(b_i-a_i)^2}\right),
$$
provided that $t>0$, where $[a_i, b_i]$ is the support of random variable $z_i$.
Applying Proposition \ref{prop:zibounds} to bound $b_i-a_i$, we get the following result.

\begin{theorem}\label{th:proddistr}
Assume that (a)-(d) hold. Let $\delta>0$, $\alpha \in (0,1]$, and let $\boldsymbol{ c}$ be an arbitrary point inside unit cube $[0,1]^n$ such that $t$ in \eqref{eq:tDef} is positive. Let $c'_i := \max\{c_i, 1-c_i\}$ and $g(c):=\min\{-c+c^2(1-\alpha),-(1-c)+(1-c)^2(1-\alpha)\}$.
If $\alpha \geq 0.5$ and 
$$
M < \sqrt{\delta}\exp\left(\frac{n^2t^2}{\sum_{i=1}^n\left(c'_i - (c'_i)^2(1-\alpha) +\frac{(c'_i)^2}{4\alpha}\right)^2}\right),
$$
or $\alpha \leq 0.5$ and
$$
M < \sqrt{\delta}\exp\left(\frac{n^2t^2}{\sum_{i=1}^n\left((1-\alpha)(c'_i)^2-g(c_i)\right)^2}\right),
$$
then set $F=\{\boldsymbol{ x}_1, \dots, \boldsymbol{ x}_M\}$ is $(\alpha, \boldsymbol{ c})$-Fisher separable with probability greater than $1-\delta$.
\end{theorem}

For $\alpha=1$, we get the following corollary
\begin{corollary}
Assume that (a)-(d) hold.  Let $\delta>0$, and let $\boldsymbol{ c}$ be an arbitrary point inside unit cube $[0,1]^n$. If
\begin{equation}\label{eq:Mbound1}
M < \sqrt{\delta}\exp\left(\frac{16}{25}n\sigma_0^4\right),
\end{equation}
then set $F=\{\boldsymbol{ x}_1, \dots, \boldsymbol{ x}_M\}$ is $(1, \boldsymbol{ c})$-Fisher separable with probability greater than $1-\delta$. 
\end{corollary}

\begin{example}\label{ex:newbound1}
With $\delta=0.01$, $n=500$, and $\sigma_0=0.5$, (same values as in Example \ref{ex:oldbound}) \eqref{eq:Mbound1} reduces to $M<48,516,519$.
\end{example}

By selecting $\boldsymbol{ c}$ being the center of the cube, we can improve the bound further.

\begin{corollary}\label{cor:center}
Assume that (a)-(d) hold.  Let $\delta>0$, and let $\boldsymbol{ c}=(\frac{1}{2}, \dots, \frac{1}{2})$ be the center of unit cube $[0,1]^n$. If
\begin{equation}\label{eq:Mboundcenter}
M < \sqrt{\delta}\exp\left(\frac{256}{81}n\sigma_0^4\right),
\end{equation}
then set $F=\{\boldsymbol{ x}_1, \dots, \boldsymbol{ x}_M\}$ is $(1, \boldsymbol{ c})$-Fisher separable with probability greater than $1-\delta$. 
\end{corollary} 

In contrast to \eqref{eq:Mbound1}, bound \eqref{eq:Mboundcenter} may be practical in dimension $n=100$.

\begin{example}\label{ex:lowdim}
If $\delta=0.01$ and $n=100$, then, even with maximal possible $\sigma_0=0.5$, 
\eqref{eq:Mbound1} reduces to $M<5.5$. In contrast, \eqref{eq:Mboundcenter} with these parameters gives $M<37,901,503$.
\end{example} 

In lager dimensions, bound \eqref{eq:Mboundcenter} may be practical for (slightly) lower $\sigma_0$.

\begin{example}\label{ex:newbound2}
If $\delta=0.01$, and $\sigma_0=\frac{1}{2\sqrt{3}}$ (standard deviation of uniform distribution on $[0,1]$), then, even with $n=1000$, \eqref{eq:Mbound1} reduces to $M<8.5$. In contrast, \eqref{eq:Mboundcenter} with these parameters gives $M<340,283,178$.
\end{example}

If $\alpha<1$, it is convenient to apply Theorem \ref{th:proddistr} with $\boldsymbol{ c}=\boldsymbol{ \mu}$. In this case $t$ in \eqref{eq:tDef} is guaranteed to be positive, and bounds in Proposition \ref{prop:zibounds} (i),(ii) imply the following result.

\begin{corollary}\label{cor:meancentered}
Assume that (a)-(d) hold.  Let $\delta>0$, $\alpha \in (0,1]$.
If
$$
M < \sqrt{\delta}\exp\left(\frac{16\alpha^4}{(1+4\alpha^2)^2}n \sigma_0^4\right), \quad \alpha \geq 0.5,
$$
or
$$
M < \sqrt{\delta}\exp\left(\frac{16(1-\alpha)^2\alpha^2}{(1+4(1-\alpha)^2)^2}n \sigma_0^4\right), \quad \alpha \leq 0.5,
$$
 then set $F=\{\boldsymbol{ x}_1, \dots, \boldsymbol{ x}_M\}$ is $(\alpha, \boldsymbol{ \mu})$-Fisher separable with probability greater than $1-\delta$.
\end{corollary}

\begin{example}\label{ex:newbound3}
With $\delta=0.01$, $n=500$, and $\sigma_0=0.5$, and $\alpha=0.9$, Corollary \ref{cor:meancentered} is applicable if $M < 8,411,607$.  
\end{example}

\subsection{The mean-centered distributions}\label{sec:meancen}

In this section we consider a special case when $\boldsymbol{ \mu}=(\frac{1}{2}, \dots, \frac{1}{2})$ is the center of the unit cube. In this case, Theorem \ref{th:proddistr} with Proposition \ref{prop:zibounds} (iii) implies that set $F$ is $(\alpha, \boldsymbol{ \mu})$-Fisher separable with probability greater than $1-\delta$ provided that
\begin{equation}\label{eq:largesigma}
M < \sqrt{\delta}\exp\left(\frac{256\alpha^4}{(1+2\alpha)^4}n \sigma_0^4\right), \quad \alpha \geq 0.5,
\end{equation}
or
$$
M < \sqrt{\delta}\exp\left(4\alpha^2 n \sigma_0^4\right), \quad \alpha \leq 0.5.
$$

With $\alpha=1$, \eqref{eq:largesigma} reduces to \eqref{eq:Mboundcenter}. It is practical if $\sigma_0$ is close to its maximal value $0.5$, but, because of factor $\sigma_0^4$, quickly becomes useless if $\sigma_0$ decreases.

\begin{example}\label{ex:smallsigma}
If $\delta=0.01$, and $\sigma_0=0.2$, then, even with $n=1000$, \eqref{eq:Mboundcenter} reduces to $M<15.7$.
\end{example}

The theorem below uses Bernstein inequality to derive an alternative bound with better dependence of $\sigma_0$.

\begin{theorem}\label{th:meancentered}
Assume that (a)-(d) hold, and assume that $\boldsymbol{ \mu}=(\frac{1}{2}, \dots, \frac{1}{2})$ is the center of unit cube $[0,1]^n$. For any $\delta>0$, $\alpha \in (0,1]$, if
$$
M < \sqrt{\delta}\exp\left(\frac{12\alpha^2}{12\alpha^2 + 13}n \sigma^2_0\right), \quad \alpha \geq 0.5,
$$
or
$$
M < \sqrt{\delta}\exp\left(\frac{3\alpha^2}{2\alpha^2 + \alpha + 3}n \sigma^2_0\right), \quad \alpha \leq 0.5,
$$
then set $F=\{\boldsymbol{ x}_1, \dots, \boldsymbol{ x}_M\}$ is $(\alpha, \boldsymbol{ \mu})$-Fisher separable with probability greater than $1-\delta$.
\end{theorem}
\begin{proof}
Bernstein inequality \cite[p. 36]{Boucheron2013} states that, if $S_n = \sum_{i=1}^n z_i$ is the sum of independent random variables with finite variance such that $z_i \leq b$ for some $b>0$ with probability $1$ for all $i=1,2,\dots,n$, then, for any $T>0$, 
\begin{equation}\label{eq:Mcentered}
{\mathbb P}[S_n - E[S_n] \geq T] \leq \exp\left(-\frac{T^2}{2(v+bT/3)}\right),
\end{equation}
where $v=\sum_{i=1}^n E[z_i^2]$. 
With $z_i$ given by \eqref{eq:zidef}, $c_i=\mu_i=1/2$, and notation $\bar{x}_i=x_i-1/2$, $\bar{y}_i=y_i-1/2$, 
$$
E[z_i] = E[\bar{x}_i]E[\bar{y}_i] - \alpha E[\bar{x}_i^2] = - \alpha E[\bar{x}_i^2].
$$
Let $\sigma_i^2=E[\bar{x}_i^2]$ be the variance of $x_i$, and $\sigma^2_x = \frac{1}{n} \sum_{i=1}^n \sigma_i^2$ be the average variance of the components of $\boldsymbol{ x}$. Then $E[S_n]=\sum_{i=1}^n E[z_i] = -n \alpha \sigma^2_x$. Also,
$$
E[z_i^2]=E[\bar{x}_i^2]E[\bar{y}_i^2]-2\alpha E[\bar{x}_i^3]E[\bar{y}_i]+\alpha^2 E[\bar{x}_i^4].
$$
 Because $\bar{y}_i$ has support $[-1/2, 1/2]$, $E[\bar{y}_i]=0$, and $f(y)=y^2$ is a convex function, $E[\bar{y}_i^2]$ is maximal if $\bar{y}_i$ takes values $\pm 1/2$ with equal chances, and $E[\bar{y}_i^2] \leq 1/4$. Next, denoting $u_i=\bar{x}_i^2$, we note that $E[u_i]=\sigma_i^2$, support of $u_i$ is $[0,1/4]$, hence $E[u_i^2]$ is maximal if $u_i$ takes values $0$ and $1/4$ with probabilities $1-4\sigma_i^2$ and $4\sigma_i^2$, respectively. Hence, $E[\bar{x}_i^4] = E[u_i^2] \leq (1/4)^2 4\sigma_i^2 = \sigma_i^2/4$. This implies that $E[z_i^2] \leq \sigma_i^2 (1/4)+\alpha^2(\sigma_i^2/4) = (1+\alpha^2)\sigma_i^2/4$. Hence, $v=\sum_{i=1}^n E[z_i^2] \leq \frac{1+\alpha^2}{4} \sum_{i=1}^n \sigma_i^2 = \frac{1+\alpha^2}{4} n \sigma^2_x$.
 
By Proposition \ref{prop:zibounds}(iii), $z_i \leq b$ for all $i$, where $b=\frac{1}{16\alpha}$ if $\alpha \geq 0.5$ and $b=\frac{1}{4}(1-\alpha)$ if $\alpha \leq 0.5$.  
 
Hence, for $\alpha \geq 0.5$, \eqref{eq:Mcentered} implies that
$$
{\mathbb P}[S_n \geq 0] = {\mathbb P}[S_n - E[S_n] \geq n \alpha \sigma^2_x] \leq 
$$
$$
\leq \exp\left(-\frac{(n \alpha \sigma^2_x)^2}{2((1+\alpha^2) n \sigma^2_x/4+n \alpha \sigma^2_x/48\alpha)}\right) = 
$$
$$
= \exp\left(\frac{-24\alpha^2}{12\alpha^2 + 13}n \sigma^2_x \right) \leq \exp\left(\frac{-24\alpha^2}{12\alpha^2 + 13}n \sigma^2_0 \right). 
$$
For $\alpha \geq 0.5$, similar calculation gives
$$
{\mathbb P}[S_n \geq 0] \leq \exp\left(\frac{-6\alpha^2}{2\alpha^2 + \alpha + 3}n \sigma^2_0 \right)
$$
Combining these bounds with \eqref{eq:MBound}, we obtain the desired result.
\end{proof}

For $\alpha=1$, this gives the following corollary.

\begin{corollary}\label{cor:smallsigma}
Assume that (a)-(d) hold, and assume that $\boldsymbol{ \mu}=(\frac{1}{2}, \dots, \frac{1}{2})$ is the center of unit cube $[0,1]^n$. For any $\delta>0$, if
\begin{equation}\label{eq:smallsigma}
M < \sqrt{\delta}\exp\left(\frac{12}{25}n \sigma^2_0\right)
\end{equation}
then set $F=\{\boldsymbol{ x}_1, \dots, \boldsymbol{ x}_M\}$ is $(1, \boldsymbol{ \mu})$-Fisher separable with probability greater than $1-\delta$.
\end{corollary}

This bound is better than \eqref{eq:Mboundcenter}, provided that $\frac{12}{25}n \sigma^2_0 > \frac{256}{81}n\sigma_0^4$, or $\sigma_0 < \frac{9\sqrt{3}}{40} \approx 0.39$.

\begin{example}\label{ex:smallsigma2}
If $\delta=0.01$, $\sigma_0=0.2$, and $n=1000$ (the same parameters as in Example \ref{ex:smallsigma}) \eqref{eq:smallsigma} reduces to $M<21,799,877$.
\end{example}

How close these bounds to being optimal? If each point in $F$ is distributed uniformly among vertices of the cube, then $2$ points $\boldsymbol{ x}$ and $\boldsymbol{ y}$ are not Fisher separable if and only if they coincide, which may happen with probability $2^{-n}$. Hence, Fisher separability of a set of $M$ points holds with probability $1-\delta$ for
$$
M \approx \sqrt{\frac{\delta}{2^{-n}}} = \sqrt{\delta}\exp\left(\frac{\log 2}{2}n\right) \approx \sqrt{\delta}\exp (0.35 n).
$$
In this example, $\sigma_0=0.5$, and \eqref{eq:Mboundcenter} gives bound $\sqrt{\delta}\exp\left(\frac{16}{81}n\right) \approx \sqrt{\delta}\exp (0.2 n)$.  Note that the coefficients in these estimates differ less than by the factor of $2$.

Corollary \ref{cor:smallsigma} follows from two-point bound \eqref{eq:2pointBound} with $f(n,\alpha)=\exp\left(-\frac{24}{25}n\sigma_0^2\right)$. Can we significantly improve the constant here, or the dependence from $\sigma_0$? Consider two points $\boldsymbol{ x}$ and $\boldsymbol{ y}$, such that all components $x_i$ of $\boldsymbol{ x}$ take values $0, 1/2, 1$ with probabilities $2\sigma_0^2$, $1-4\sigma_0^2$, $2\sigma_0^2$, respectively, and all components $y_i$ of $\boldsymbol{ y}$ take values $0, 1$ with equal chances. Then $z_i$ given by \eqref{eq:zidef} take values $-1/2$ and $0$ with probabilities $2\sigma_0^2$ and $1-2\sigma_0^2$, respectively, hence $\boldsymbol{ x}$ and $\boldsymbol{ y}$ are  not Fisher separable with probability $p_n=(1-2\sigma_0^2)^n$. For small $\sigma_0$, $p_n \approx \exp(-2n\sigma_0^2)$, hence the quadratic dependence on $\sigma_0$ cannot be improved, and the coefficient $\frac{24}{25}$ cannot be improved to any value higher than $2$.

\subsection{Better bounds if the product distribution is known}

The results in Sections \ref{sec:gen} and \ref{sec:meancen} are valid for the whole family of product distributions satisfying certain conditions. This Section studies the case when the data distribution is explicitly given. In this case, we can deduce improved estimates from Chernoff's inequality. Our first result is for general product distributions, not necessary bounded in the unit cube.

\begin{theorem}\label{th:explicit}
Let points $\boldsymbol{ x}_1, \dots, \boldsymbol{ x}_M$ be i.i.d points from an arbitrary but explicitly given product distribution $G$ in ${\mathbb R}^n$. For any $\delta>0$, let
\begin{equation}\label{eq:explicit}
M < \sqrt{\delta}\exp\left(\gamma_n\right),
\end{equation}
where
$$
\gamma_n = \frac{1}{2}\sup\limits_{\lambda \geq 0}\left( -\sum_{i=1}^n\log E\left[e^{\lambda(x_iy_i - \alpha x_i^2)}\right]\right),
$$
where $x_i$ and $y_i$ are independent random variables distributed as $i$-th component of $G$. Then set $F=\{\boldsymbol{ x}_1, \dots, \boldsymbol{ x}_M\}$ is $(\alpha, \boldsymbol{ 0})$-Fisher separable with probability greater than $1-\delta$.
\end{theorem}
\begin{proof}
With $\boldsymbol{ c}=\boldsymbol{ 0}$, \eqref{eq:zidef} simplifies to
$$
z_i = x_iy_i - \alpha x_i^2, \quad i=1,\dots,n,
$$
where $x_i$ and $y_i$ are independent and distributed as $i$-th component of $G$. 
Points $\boldsymbol{ x}$ and $\boldsymbol{ y}$ are not Fisher separable if $S \geq 0$, where $S=\sum_{i=1}^n z_i$.

Chernoff's inequality \cite[p. 21]{Boucheron2013} states that, for any random variable $S$, and any real number $t$,
$$
{\mathbb P}[S \geq t] \leq \exp[-\psi^*_S(t)],
$$
where 
$$
\psi^*_S(t) = \sup\limits_{\lambda \geq 0} (\lambda t - \psi_S(\lambda)),
$$
where $\psi_S(\lambda) = \log(E[e^{\lambda S}])$.
If $S=\sum_{i=1}^n z_i$ for independent random variables $z_i$,
$$
e^{\psi_S(\lambda)} = E[e^{\lambda \sum_{i=1}^n z_i}] = E\left[\prod_{i=1}^n e^{\lambda z_i}\right] = \prod_{i=1}^n E[e^{\lambda z_i}],
$$
hence
$$
\psi_S(\lambda) = \log\left(\prod_{i=1}^n E[e^{\lambda z_i}]\right) = \sum_{i=1}^n\log E\left[e^{\lambda(x_iy_i - \alpha x_i^2)}\right],
$$
and $\psi^*_S(0)=2\gamma_n$. Hence, Chernoff's inequality with $t=0$ implies that
$$
{\mathbb P}[S \geq 0] \leq  \exp[-\psi^*_S(0)] = \exp[-2\gamma_n],
$$
and \eqref{eq:explicit} follows from \eqref{eq:MBound}.
\end{proof}

\begin{corollary}\label{cor:identcomp}
If all components of $G$ in Theorem \ref{th:explicit} have the same distribution, then estimate \eqref{eq:explicit} simplifies to
\begin{equation}\label{eq:identcomp}
M < \sqrt{\delta}\exp\left(\gamma n\right),
\end{equation}
where
$$
\gamma = \frac{1}{2}\sup\limits_{\lambda \geq 0}\left( -\log E\left[e^{\lambda(x y - \alpha x^2)}\right]\right),
$$
where $x$ and $y$ are independent random variables distributed as a component of $G$. In particular, if the component distribution has density $f$, then
$$
\gamma = \frac{1}{2}\sup\limits_{\lambda \geq 0}\left( -\log\left[\int_{{\mathbb R}^2}e^{\lambda (xy-\alpha x^2)}f(x)f(y)dxdy\right]\right).
$$
\end{corollary}

It follows from the proof of Theorem \ref{th:explicit} and Cramer’s theorem \cite[Theorem 2.1]{Pham} that the exponent $\gamma$ in \eqref{eq:identcomp} is the best possible.
However, estimate \eqref{eq:identcomp} maybe non-optimal in lower order terms. Below we give a formula for the asymptotically best possible upper bound for $M$ in Corollary \ref{cor:identcomp}.

Let $\lambda^*$ be the (unique) minimizer of $E\left[e^{\lambda(x y - \alpha x^2)}\right]$, and let
$$
c^* := \left.\frac{d}{d\lambda}\left(\frac{E\left[(x y - \alpha x^2)e^{\lambda(x y - \alpha x^2)}\right]}{E\left[e^{\lambda(x y - \alpha x^2)}\right]}\right)\right|_{\lambda=\lambda^*}
$$
The exact asymptotic growth of the probability ${\mathbb P}[S \geq 0]$ in Theorem \ref{th:explicit} is given by \cite[Theorem 1]{Petrov}
$$
{\mathbb P}[S \geq 0] = \frac{\exp[-2\gamma n]}{\lambda^*\sqrt{c^*}\sqrt{2 \pi n}}(1+o(1)),
$$
hence the exact asymptotic estimate for $M$ in Corollary \ref{cor:identcomp} is
$$
M < \sqrt{\delta}\sqrt{\lambda^*}\sqrt[4]{2\pi c^* n}\exp\left(\gamma n\right)(1+o(1)).
$$  

We can see that estimate \eqref{eq:identcomp} differs from the optimal one by $\sqrt{\lambda^*}\sqrt[4]{2\pi c^* n}$ term. However, the advantages of estimate \eqref{eq:identcomp} is simplicity and the absence of $(1+o(1))$ term.

We now apply Corollary \ref{cor:identcomp} to some special cases. First, Corollary \ref{cor:identcomp} specialised to standard normal distribution implies Theorem \ref{th:normalknown}. As another example, we apply Corollary \ref{cor:identcomp} to the uniform distribution in a cube. 

\begin{corollary}\label{cor:uniform}
Let $\alpha=1$ and points $\boldsymbol{ x}_1, \dots, \boldsymbol{ x}_M$ are i.i.d points from uniform distribution in a cube with center $\boldsymbol{ \mu}$. For any $\delta>0$, if
\begin{equation}\label{eq:uniform2}
M < \sqrt{\delta}\exp\left(\gamma n\right),
\end{equation}
where
$$
\gamma = \frac{1}{2}\sup\limits_{\lambda \geq 0}\left( -\log\left[\int_{-\frac{1}{2}}^{\frac{1}{2}} \int_{-\frac{1}{2}}^{\frac{1}{2}} e^{\lambda (xy-x^2)}dxdy\right]\right)= 0.23319...
$$
then set $F=\{\boldsymbol{ x}_1, \dots, \boldsymbol{ x}_M\}$ is $(1, \boldsymbol{ \mu})$-Fisher separable with probability greater than $1-\delta$.
\end{corollary}

For the unit cube, $\sigma^2_0 = \frac{1}{12}$, and Corollary \ref{cor:smallsigma} implies $(1, \boldsymbol{ \mu})$-Fisher separability with probability greater than $1-\delta$ provided that 
\begin{equation}\label{eq:uniform}
M < \sqrt{\delta}\exp\left(\frac{n}{25}\right)
\end{equation}
We can see that \eqref{eq:uniform2} is a substantial improvement over \eqref{eq:uniform}. This is because \eqref{eq:uniform2} works for uniform distribution only, while \eqref{eq:uniform} works for \emph{any} product distribution in the unit cube with $\sigma^2_0 = \frac{1}{12}$.


\begin{example}\label{ex:uniform}
Let $\alpha=1$ and $\delta=0.01$. Table~\ref{Table10} shows the upper bounds on $M$ in Corollary \ref{cor:uniform} in various dimensions $n$.
\begin{table}[h]
\begin{center}
\caption{\label{Table10}  The upper bounds on $M$ (\ref{eq:uniform}) in Corollary \ref{cor:uniform} in various dimensions $n$ for $\alpha=1$ and $\delta=0.01$.}
\begin{tabular}{ |c|c| } 
 \hline
 $n$ & $M\leq$\\ 
 \hline
 $10$ & $1.02$ \\ 
 $50$ & $11,578$\\ 
 $100$ & $1.3 \cdot 10^9$\\ 
 $200$ & $1.7 \cdot 10^{19}$\\ 
 $500$ & $4.3 \cdot 10^{49}$\\ 
 $1000$ & $1.8 \cdot 10^{100}$\\ 
 \hline
\end{tabular}
\end{center}
\end{table}
For example, for $n=100$, we see that over a billion points from the uniform distribution in the unit cube are $1$-Fisher-separable with probability greater than $99\%$ (compare to the earlier weaker estimates by \citet{GorbTyu2017} and  \citet{Sidorov2020}).
\end{example}

\section{Fisher separability for dependent data from product distribution}\label{sec:dependent}

The key assumption in Section \ref{sec:product} is that all points in set $F$ are chosen independently. This section establishes a sufficient condition for Fisher separability with high probability in a datasets with dependent data points, as soon as the corresponding conditional distributions are product distributions in the unit cube $U_n = [0,1]^n$.

Formally, we assume the following.
\begin{itemize}
\item[(*)] For any $\boldsymbol{ x}\in F$ and $\boldsymbol{ y} \in F$, and any $\boldsymbol{ y}_0 \in U_n$, the conditional distribution of $\boldsymbol{ x}$ given $\boldsymbol{ y}=\boldsymbol{ y}_0$ is a product distribution with support in $U_n$.
\end{itemize}

For every $\boldsymbol{ x}\in F, \boldsymbol{ y} \in F, \boldsymbol{ y}_0 \in U_n$ and index $i \in {1,2,\dots,n}$, let $\sigma_i^2(\boldsymbol{ x},\boldsymbol{ y}, \boldsymbol{ y}_0)$ be the variance of the conditional distribution of the $i$-th component of $\boldsymbol{ x}$ given $\boldsymbol{ y}=\boldsymbol{ y}_0$. 
Let
$$
\sigma_0^2 = \min_{\boldsymbol{ x} \in F, \boldsymbol{ y} \in F, \boldsymbol{ y}_0 \in U_n} \left( \frac{1}{n} \sum_{i=1}^n \sigma_i^2(\boldsymbol{ x},\boldsymbol{ y}, \boldsymbol{ y}_0) \right)
$$
be the minimal value of average variance of the components of such conditional distribution. Also, let $\boldsymbol{ c}^*=(1/2,\dots,1/2)$ be the center of $U_n$.

\begin{theorem}\label{th:dependent}
Assume that (*) holds. For any $\delta>0$, $\alpha \in (0,1]$, if
$$
\sigma^2_0 > \frac{1}{16 \alpha^2}
$$
and
$$
M < \sqrt{\delta}\exp\left(\frac{256\alpha^4}{(1+2\alpha)^4}\left(\sigma^2_0 - \frac{1}{16 \alpha^2}\right)^2 n\right)
$$
then set $F=\{\boldsymbol{ x}_1, \dots, \boldsymbol{ x}_M\}$ is $(\alpha, \boldsymbol{ c}^*)$-Fisher separable with probability greater than $1-\delta$.
\end{theorem}
\begin{proof}
By \eqref{eq:MBound}, the statement of the theorem follows from \eqref{eq:2pointBound} with $\boldsymbol{ c} = \boldsymbol{ c}^*$ and  
$$
f(n,\alpha)=\exp\left(-2\left(\frac{4\alpha}{2\alpha+1}\right)^4\left(\sigma^2_0 - \frac{1}{16 \alpha^2}\right)^2 n\right).
$$
We will show that for any $\boldsymbol{ x}\in F, \boldsymbol{ y} \in F$, and $\boldsymbol{ y}_0 \in U_n$
\begin{equation}\label{eq:condBound}
{\mathbb P}[\alpha(\boldsymbol{ x}-\boldsymbol{ c}^*,\boldsymbol{ x}-\boldsymbol{ c}^*) \leq (\boldsymbol{ x}-\boldsymbol{ c}^*,\boldsymbol{ y}-\boldsymbol{ c}^*)\,|\,\boldsymbol{ y}=\boldsymbol{ y}_0 ] \leq f(n,\alpha), 
\end{equation}
which would imply \eqref{eq:2pointBound} and finish the proof. The set of all $\boldsymbol{ x} \in {\mathbb R}^n$ which does not satisfy the inequality $\alpha(\boldsymbol{ x}-\boldsymbol{ c}^*,\boldsymbol{ x}-\boldsymbol{ c}^*) \leq (\boldsymbol{ x}-\boldsymbol{ c}^*,\boldsymbol{ y}-\boldsymbol{ c}^*)$ is the ball with center $\boldsymbol{ c}=\boldsymbol{ c}^* + (\boldsymbol{ y}_0-\boldsymbol{ c}^*)/2\alpha$ and radius $r=||\boldsymbol{ y}_0-\boldsymbol{ c}^*||/2\alpha$, see \cite{Gorbetal2018}. Because $\boldsymbol{ y}_0 \in U_n$, $r^2 = ||\boldsymbol{ y}_0-\boldsymbol{ c}^*||^2/4\alpha^2 \leq n(1/2)^2/4\alpha^2 = n/16\alpha^2$, and \eqref{eq:condBound} would follow from
\begin{equation}\label{eq:condBound2}
{\mathbb P}\left[(\boldsymbol{ x}-\boldsymbol{ c},\boldsymbol{ x}-\boldsymbol{ c}) \leq \frac{n}{16\alpha^2}\,|\,\boldsymbol{ y}=\boldsymbol{ y}_0 \right] \leq f(n,\alpha).
\end{equation}

Let $x_i$ be the random variable whose distribution is the conditional distribution of the $i$-th component of $\boldsymbol{ x}$ given $\boldsymbol{ y}=\boldsymbol{ y}_0$. Let
$z_i=-(x_i - c_i)^2$, where $c_i$ is the $i$-th component of $\boldsymbol{ c}$. Then
$$
E[z_i] = -E[(x_i - c_i)^2] \leq -E[(x_i - E[x_i])^2] =-\sigma_i^2(\boldsymbol{ x},\boldsymbol{ y}, \boldsymbol{ y}_0),
$$
and
$$
\sum_{i=1}^n E[z_i] \leq -\sum_{i=1}^n \sigma_i^2(\boldsymbol{ x},\boldsymbol{ y}, \boldsymbol{ y}_0) \leq -n\sigma_0^2.
$$
Hence,
$$
{\mathbb P}\left[(\boldsymbol{ x}-\boldsymbol{ c},\boldsymbol{ x}-\boldsymbol{ c}) \leq \frac{n}{16\alpha^2}\,|\,\boldsymbol{ y}=\boldsymbol{ y}_0 \right] = {\mathbb P}\left[\sum_{i=1}^n z_i \geq -\frac{n}{16\alpha^2} \right] 
$$
$$
\leq {\mathbb P}\left[\sum_{i=1}^n (z_i-E[z_i]) \geq -\frac{n}{16\alpha^2}+n\sigma_0^2 \right].
$$

In fact, $c_i=1/2+(y^0_i-1/2)/2\alpha$, where $y^0_i$ is the $i$-th component of  $\boldsymbol{ y}_0$. Because $0 \leq y^0_i \leq 1$, we get $1/2-1/4\alpha \leq c_i \leq 1/2+1/4\alpha$, hence $-(1/2+1/4\alpha)^2 \leq z_i \leq 0$. 
By Hoeffding's inequality \cite[Theorem 2.8]{Boucheron2013}, 
$$
{\mathbb P}\left[\sum_{i=1}^n (z_i-E[z_i]) \geq t \right] \leq \exp\left(-\frac{2t^2}{n(1/2+1/4\alpha)^4}\right).
$$
With $t=-\frac{n}{16\alpha^2}+n\sigma_0^2$, this proves \eqref{eq:condBound2}.
\end{proof}

We remark that because $\sigma_0 \leq 0.5$, the bound $\sigma^2_0 > \frac{1}{16 \alpha^2}$ in Theorem \ref{th:dependent} may hold only if $\alpha>1/2$. 

For $\alpha=1$, we have the following corollary.

\begin{corollary}\label{cor:dependent}
Assume that (*) holds. For any $\delta>0$, if
$$
\sigma^2_0 > \frac{1}{16}
$$
and
\begin{equation}\label{eq:dependent}
M < \sqrt{\delta}\exp\left(\frac{256}{81}\left(\sigma^2_0 - \frac{1}{16}\right)^2 n\right)
\end{equation}
then set $F=\{\boldsymbol{ x}_1, \dots, \boldsymbol{ x}_M\}$ is $(1, \boldsymbol{ c}^*)$-Fisher separable with probability greater than $1-\delta$.
\end{corollary}


\begin{example}\label{ex:uniform}
Let $\alpha=1$ and $\delta=0.01$. Table~\ref{Table11} shows the upper bounds on $M$ in Corollary \ref{cor:dependent} for $\sigma_0=0.4, 0.45$, and $0.5$ in various dimensions $n$.
\begin{table}[h]
\begin{center}
\caption{\label{Table11} The upper bounds on $M$ in Corollary \ref{cor:dependent} for $\alpha=1$ and $\delta=0.01$, various   $\sigma_0$ and dimensions.}
\begin{tabular}{ |c|c|c|c| } 
 \hline
  & $\sigma_0=0.4$ & $\sigma_0=0.45$ & $\sigma_0=0.5$ \\ 
 \hline
 $n=10$ & $0.13$ & $0.18$ & $0.3$\\ 
 $n=50$ & $0.44$ & $2.21$ & $25$\\ 
 $n=100$ & $2$ & $49$ & $6,691$\\ 
 $n=200$ & $40$ & $24,017$ & $4.4 \cdot 10^8$\\ 
 $n=500$ & $334,248$ & $2.8 \cdot 10^{12}$ & $1.3 \cdot 10^{23}$\\ 
 $n=1000$ & $1.1 \cdot 10^{12}$ & $8 \cdot 10^{25}$ & $1.7 \cdot 10^{47}$\\ 
 \hline
\end{tabular}
\end{center}
\end{table}

For example, for $n=500$ and $\sigma_0=0.4$, we see that over $300,000$ points are Fisher-separable with probability greater than $99\%$.
\end{example}

Corollary \ref{cor:dependent} is not applicable if $\sigma^2_0 \leq 1/16$. However, this is unavoidable. Indeed, let set $F$ contain points $\boldsymbol{ x}$ and $\boldsymbol{ y}$ such that $\boldsymbol{ y}$ is uniformly distributed among the vertices of the unit cube, and $\boldsymbol{ x}$ is uniformly distributed among the vertices of the (twice smaller) cube with main diagonal connecting $\boldsymbol{ c}^*$ and $\boldsymbol{ y}$. Then the variance of the components of $\boldsymbol{ x}$ is $1/16$, but $\boldsymbol{ x}$ and $\boldsymbol{ y}$ are not $(1, \boldsymbol{ c}^*)$-Fisher separable with probability $1$.

\section{Summary: a short guide on  proven theorems \label{Sec:Summ}}

We established new stochastic separation theorems for a broad class of log-concave and product distributions. All the theorems state that if the number of points $M$ does not exceed some bound $M_0$, then the points are Fisher separable with high probability. In all theorems, the bound $M_0$ grows exponentially in dimension $n$. The exact rate of growth of $M_0$ depends on the distribution assumptions we impose. If we make stronger assumptions, we can prove theorems with faster-growing upper bound $M_0$, and can ensure separation of more points. 

We can get the strongest bound separation theorems if we assume that the data are i.i.d. and are taken from a fixed given distribution such as the standard normal distribution
 (Theorems \ref{th:normalknown} and \ref{th:normalnew}), uniform distributions in a ball (Theorem \ref{th:ballnew}) or in a unit cube (Corollary \ref{cor:uniform}), or multivariate exponential distribution (Theorem \ref{th:expon}).

More generally, we have established new separation theorems for i.i.d. data from any fixed given distribution $f$, assuming that $f$ is either spherically invariant (Theorem \ref{th:sepfixed}) or a product distribution (Theorem \ref{th:explicit} and Corollary \ref{cor:identcomp}).

In the Theorems listed above, the distribution $f$ is assumed to be known and the bound $M_0$ explicitly depend on $f$. More generally, we may assume that distribution $f$ is unknown but is known to belong to some family ${\cal F}$ of distributions. In this case, the bound $M_0$ should depend on ${\cal F}$ but not on $f$. We have proved such separation theorems for i.i.d. data from (unknown) product distribution (Theorems \ref{th:proddistr} and \ref{th:meancentered}), rotation invariant distribution (Theorems \ref{th:rotsimple} and \ref{th:rotgeneral}), isotropic strongly log-concave distribution (Theorems \ref{th:explstrong} and \ref{th:explstrong2}), and, more generally, any mixture of strongly log-concave distributions (Theorem \ref{th:mixture}). This last theorem is very general, because any distribution with exponentially decaying tails may be approximated by a mixture of log-concave ones.

Finally, we have Theorems with i.i.d. assumption relaxed. In particular, in Theorem~\ref{Th:prototype} the probability of separability of a random point from a finite set was estimated  without any assumption about the randomness and distributions of this finite set.  Theorem \ref{th:indbound} treats the case when the data are independent but not identically distributed, and their distributions are strongly log-concave but not isotropic. Theorem \ref{th:dependent} treats the case when the data may be dependent but the conditional distributions are product distributions.

The results are illustrated on Figures \ref{Fig:rotM} - \ref{Fig:prodprob}.

\begin{figure} 
\centering
\includegraphics[width=0.5\textwidth]{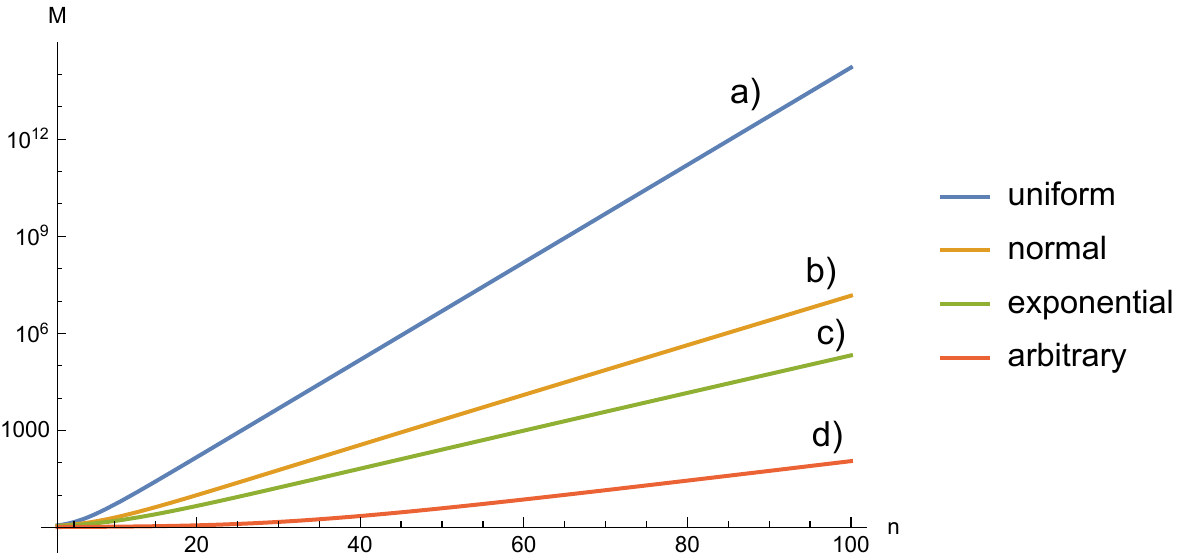}
\caption{The number $M$ of points which are guaranteed to be $1$-Fisher separable with probability $99\%$ as a function of dimension $n$ for (a) the uniform distribution in a ball (Corollary \ref{cor:ball1}), (b) the standard normal distribution (Theorem \ref{th:normalnew}), (c) multivariate exponential distribution (Theorem \ref{th:expon}), and (d) lower bound for $M$ which works for an arbitrary log-concave rotation-invariant distribution (Theorem \ref{th:rotgeneral}).
\label{Fig:rotM}}
\end{figure}

\begin{figure} 
\centering
\includegraphics[width=0.5\textwidth]{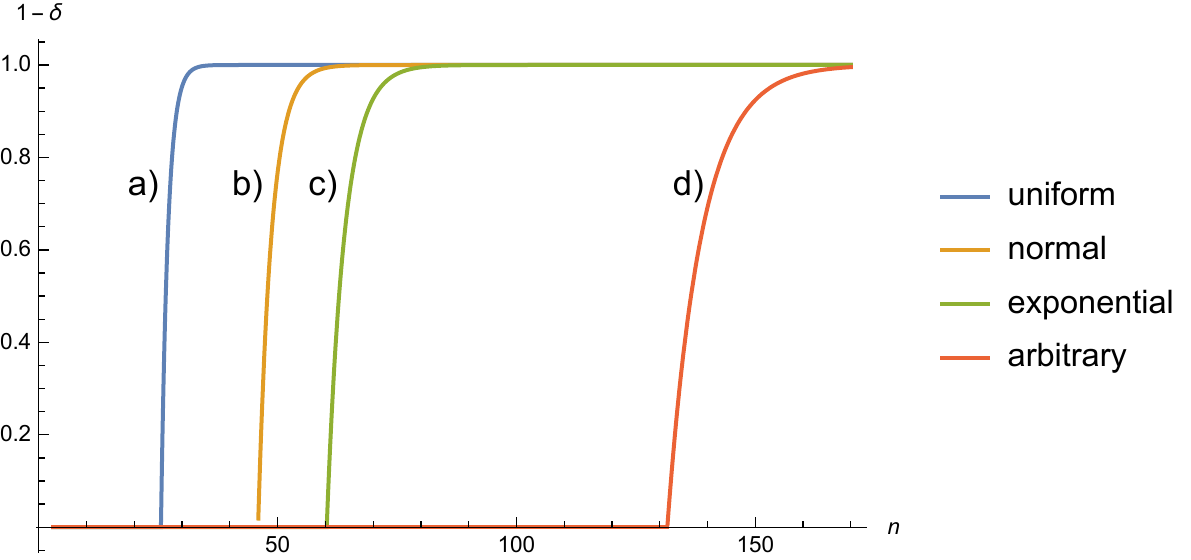}
\caption{The lower bound $1-\delta$ for the probability that the set of $M=10,000$ points is $1$-Fisher separable as a function of dimension $n$ for (a) the uniform distribution in a ball (Corollary \ref{cor:ball1}), (b) the standard normal distribution (Theorem \ref{th:normalnew}), (c) multivariate exponential distribution (Theorem \ref{th:expon}), and (d) an arbitrary log-concave rotation-invariant distribution (Theorem \ref{th:rotgeneral}).
\label{Fig:rotprob}}
\end{figure}

\begin{figure} 
\centering
\includegraphics[width=0.5\textwidth]{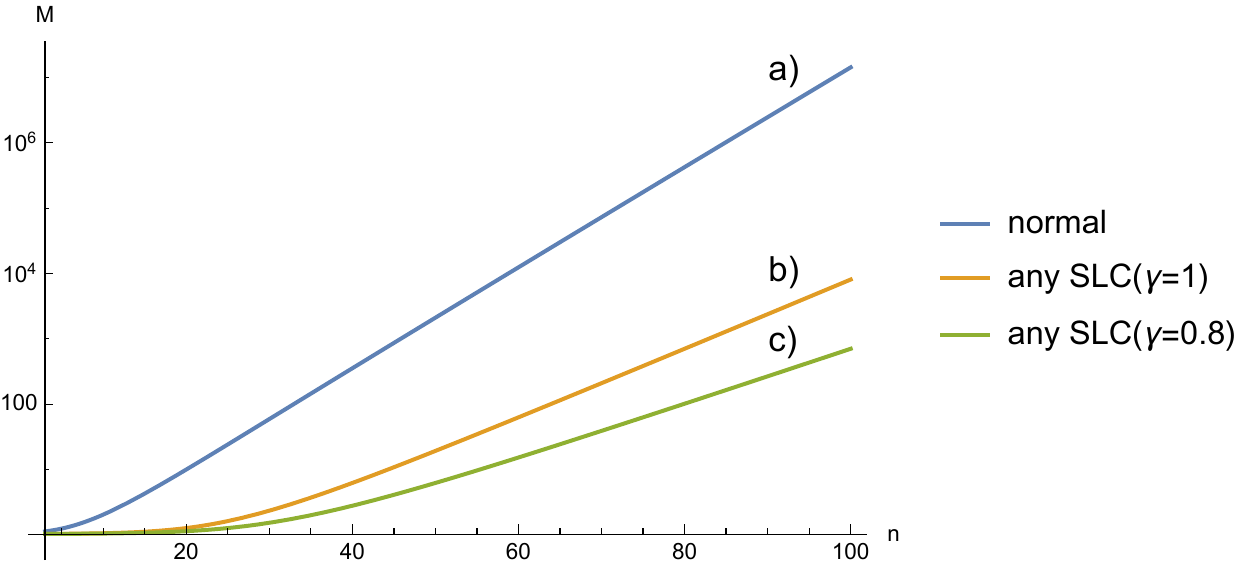}
\caption{The number $M$ of points which are guaranteed to be $1$-Fisher separable with probability $99\%$ as a function of dimension $n$ for (a) the standard normal distribution
(Theorem \ref{th:normalnew}), (b) an arbitrary strictly log-concave distribution with $\gamma=1$ (Theorem \ref{th:explstrong2}), and (c) an arbitrary strictly log-concave distribution with $\gamma=0.8$ (Theorem \ref{th:explstrong2}). Recall that the standard normal distribution is strictly log-concave with $\gamma=1$.
\label{Fig:slcM}}
\end{figure}

\begin{figure} 
\centering
\includegraphics[width=0.5\textwidth]{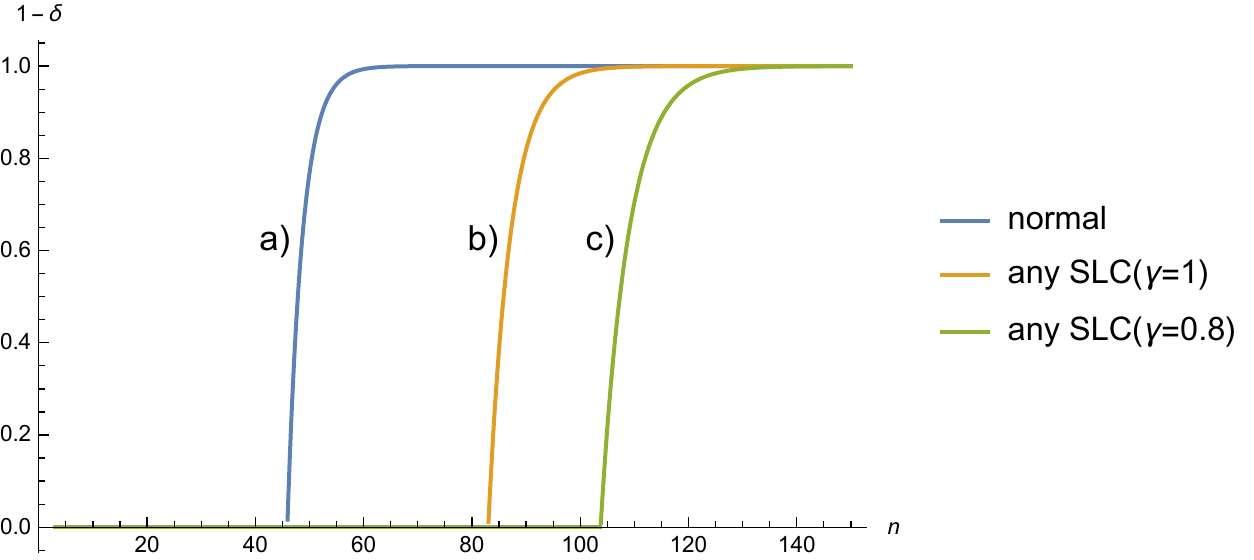}
\caption{The lower bound $1-\delta$ for the probability that the set of $M=10,000$ points is $1$-Fisher separable as a function of dimension $n$ for (a) the standard normal distribution (Theorem \ref{th:explstrong2}), (b) arbitrary strictly log-concave distribution with $\gamma=1$ (Theorem \ref{th:explstrong2}), and (c) arbitrary strictly log-concave distribution with $\gamma=0.8$ (Theorem \ref{th:explstrong2}).
\label{Fig:slcprob}}
\end{figure}

\begin{figure} 
\centering
\includegraphics[width=0.5\textwidth]{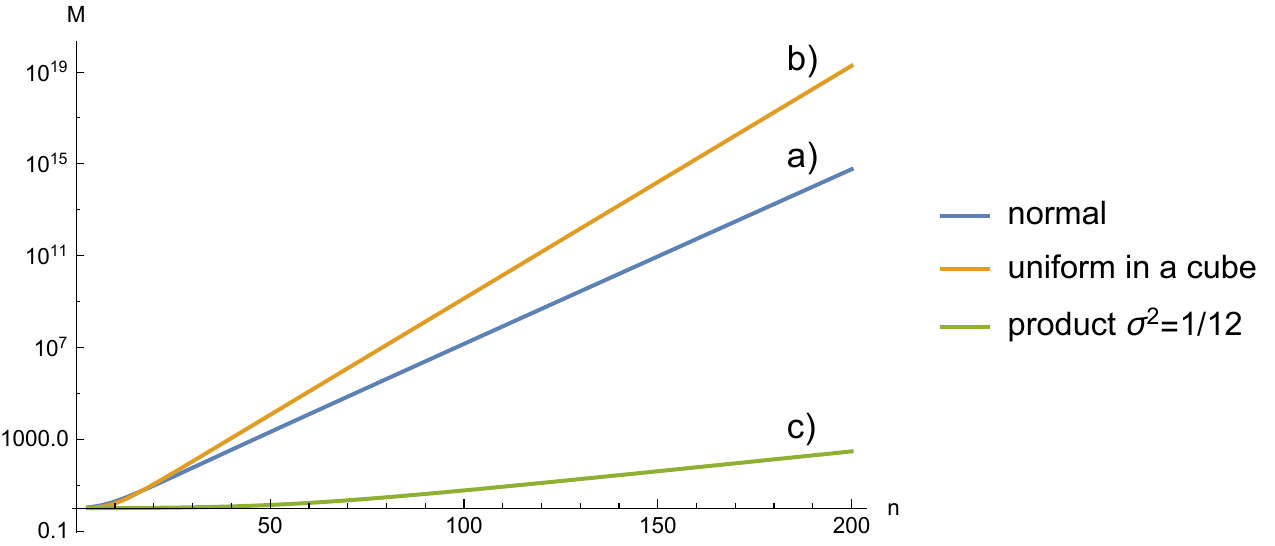}
\caption{The number $M$ of points which are guaranteed to be $1$-Fisher separable with probability $99\%$ as a function of dimension $n$ for (a) the standard normal distribution (Theorem \ref{th:explstrong2}), (b) the uniform distribution in a cube (Corollary \ref{cor:uniform}) and (c) an arbitrary mean-centered product distribution with $\sigma^2=\frac{1}{12}$ (Corollary \ref{cor:smallsigma}). Recall that the uniform distribution in a cube has $\sigma^2=\frac{1}{12}$. 
\label{Fig:prodM}}
\end{figure}

\begin{figure} 
\centering
\includegraphics[width=0.5\textwidth]{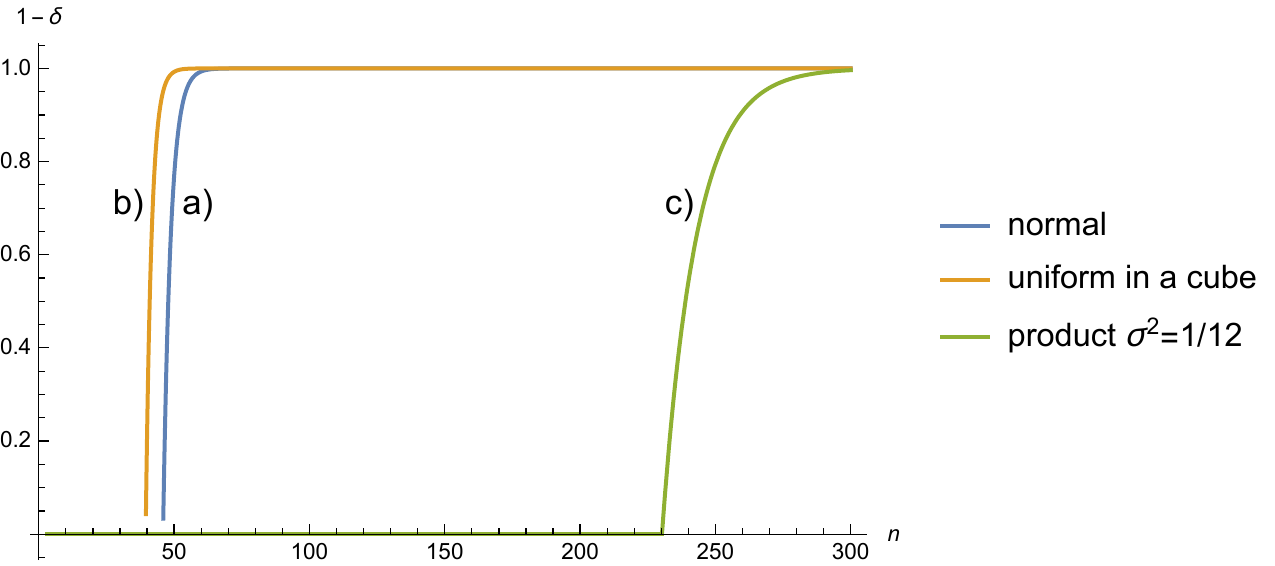}
\caption{The lower bound $1-\delta$ for the probability that the set of  $M=10,000$ points is $1$-Fisher separable as a function of dimension $n$ for (a) the standard normal distribution (Theorem \ref{th:explstrong2}), (b) the uniform distribution in a cube (Corollary \ref{cor:uniform}) and (c) an arbitrary product distribution with $\sigma^2=\frac{1}{12}$ (Corollary \ref{cor:smallsigma}).
\label{Fig:prodprob}}
\end{figure}

\section{Conclusion: what are these estimates for? \label{sec:concl}}

The theorems presented in the paper have, roughly speaking, the following structure: for a given class of distributions, a random set of $M$ vectors in  ${\mathbb R}^n$  is $\alpha$-Fisher separable with probability $\geq p$ if $M \leq M_0$, where $M_0$ depends on $n$, $p$, and $\alpha$ and this dependence is specific for the selected class of probability distributions.
For the distributions without heavy tails and ``clumps'' (sets with relatively low volume but high probability) $M_0$ grows fast with $n$: exponentially for strictly log-concave distributions (tails that decay as $\exp(-a\|\boldsymbol{ x} \|^2)$  or faster) and as exponent of $\sqrt{n}$ (exponential tails that decay as    $\exp(-a\|\boldsymbol{ x} \|$). The main problem solved in the work was to find the best (optimal and explicit) estimates.

Stochastic separation theorems form a relatively new chapter of the measure concentration theory (for the collection of the classical results about concentration of measure we refer to  \citet{GianMilman2000,Ledoux2001,Vershynin2018}). Concentration of random sets in thin shells is well-known: equivalence of microcanonical and canonical ensembles in statistical physics due to concentration near the level sets of energy \citep{Gibbs1902}, concentration of the volume of a ball near its border, the sphere, and concentration of the sphere near its equators \citep{Levy1951,Bal1997} (and general `waist concentration' \citep{Gromov2003}),  etc. Stochastic separation theorems describe the fine structure of this thin layer.

The first theorems of this class were considered as the manifestation of the blessing of dimensionality \citep{GorbanTyuRom2016,GorTyukPhil2018}. Indeed, the fast and non-iterative correction of the AI errors is based on the phenomenon of stochastic separation in high dimensions. The legacy AI systems are supplemented by {\em correctors}. These simple smart devices separate recognized errors and their surroundings from situations with correct functioning and replace the legacy AI solution with the corrected one. One of the possible structures of correcting system is presented in Fig.~\ref{Fig:Correctors}. The correcting system receives a vector of signals that represents the situation in maximal detail. It consist of input vectors of the legacy AI system, vector of internal signals  of that and the output vector (Fig.~\ref{Fig:Correctors}). There are several elementary correctors (Corrector 1, Corrector 2, ... Corrector $n$ in Fig.~\ref{Fig:Correctors}). Each elementary corrector includes a classifier, which separates a cluster of recognized errors from all other situations, and keeps the modified decision rule for this cluster. Dispatcher selects for each situation the closest cluster and sends the vector that represents the situation to the corresponding elementary corrector for further decision. The elementary corrector takes the decision ``an error or not an error'' and acts according to this decision.  Stochastic separation theorems are necessary to evaluate the probability of accurate work of such a system. Of course, its accuracy increases with dimensionality of data. Correctors can be used for solution of the classical problem of sensitivity and specificity improvement (removing  false-positive and false-negative results of classification), for knowledge transfer between artificial intelligence systems \citep{TyukinSofeikov2018}, for training of multiagent systems and other purposes.

\begin{figure} 
\centering
\includegraphics[width=0.4\textwidth]{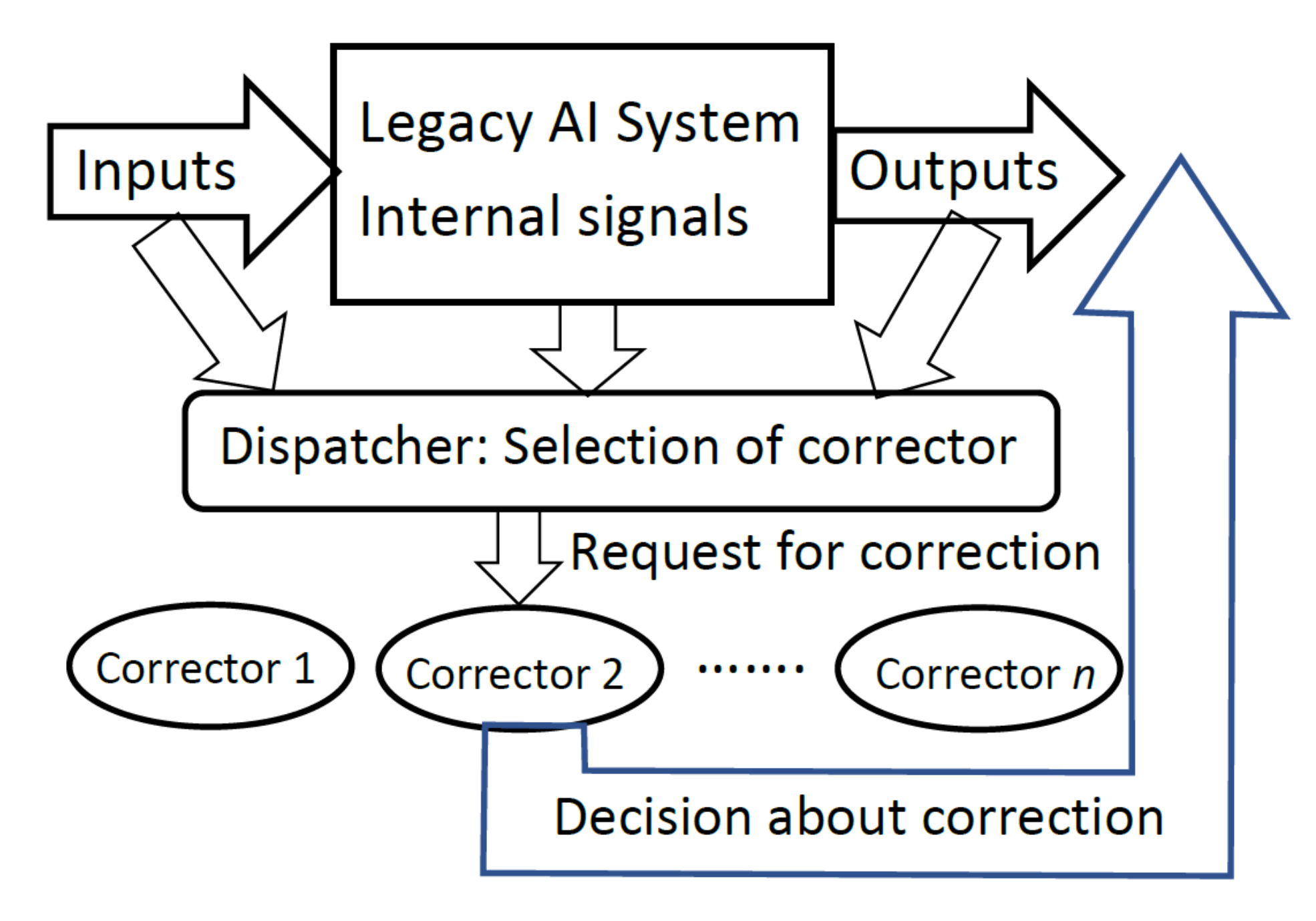}
\caption{Corrector of  a legacy AI system. 
\label{Fig:Correctors}}
\end{figure}

If the AI system works for a long time, then errors and their correctors accumulate.  The `technical debt' increases, and flexibility drops down \citep{Sculley2015}. In this situation, the Interiorization of the accumulated knowledge is necessary. This is incorporation of knowledge into system's inner structure. Interiorization can be organized as  supervised learning that uses the system with correctors as the supervisor. The AI system, equipped with correctors (`teacher'), labels randomly generated examples (proposes the answers or actions) and the AI system without correctors (`student') learns to give the proper answer. At the beginning, the student is the same legacy AI system, as the teacher, but without correctors. During the learning process, the student's skills change. The random generation of examples can be improved by selection of the more realistic examples and by elements of adversarial learning (selection of the examples with higher probability of errors). This play of the system with itself is a realization of the famous selfplay technology of DeepMind (for discussion of the selfplay principle  and DeepMind  Alpha Go Zero technology we refer to \citet{Holcomb2018}).

Stochastic separation theorems have three critical applications. One of them is one-shot correction of errors in intellectual systems. Recently, it was realized that the possibility to correct an AI system opens also the possibility to attack it.  The dimensionality of the AI's decision-making space is a major contributor to the AI's vulnerability \citep{TyukinEtAl2020}. So, the stochastic separation theorems demonstrate also the new version of the curse of dimensionality. As we said, the blessing and curse of dimensionality are two sides of the same coin. Thus, the second application is vulnerability analysis of high-dimensional AI systems in high-dimensional world.
  
 The third application is to explain the ``unreasonable effectiveness'' of small neural ensembles in the multidimensional brain and the emergence of static and associative memories   in the ensembles of single neurons \citep{GorbanMakTyu2019}. A simple enough functional neuronal model is capable of explaining: i) the extreme selectivity of single neurons to the information content of high-dimensional data, ii) simultaneous separation of several uncorrelated informational items from a large set of stimuli, and iii) dynamic learning of new items by associating them with already ``known'' ones \citep{TyukinBMB2019}. These results constitute a basis for organization of complex memories in ensembles of single neurons. The stochastic separation theorems give the theoretical background of existence and efficiency of `concept cells' and sparse coding in a brain \citep{GorbanMakTyu2019, QuianQuiroga(2019),Tapia2020}. (These  `hardware components of thought and memory' are presented in detail by \citet{QuianQuiroga2005, QuianQuiroga2013,Viskontas2009}.)
 
There are also many technical applications of stochastic separation theorems with optimal bounds in various areas of data analysis and machine learning, for example, for estimation of dimensionality of data. The estimated dimension depends linearly on the exponents from these bounds for the methods based on the data separability properties  \citep{BacZinovyev2020, Mirkesetal2020}. Therefore, if we use bound with exponent  twice far from the optimal one, then we  misestimate the data dimension twice.
 
 In recent review by \citet{BacZinovyev2020} the typology of these methods is proposed and a new family of methods based on the data separability properties is presented.
  
Stochastic separation theorems shed light on the fundamental {\em problem of learning from few examples in high dimensions}. This problem is central for understanding when and why modern large-scale systems can learn from post-classic data and generalize so well in practice. Classical generalization bounds stemming from the Vapnik-Chervonenkis theory \cite{vapnik1999overview} alone are too conservative to explain these successes. It has been demonstrated in \cite{zhang2016understanding} that absolutely identical deep neural networks are capable to exhibit both sides of the learning spectrum: to successfully generalize from meaningful training data and, at the same time,  `memorise' random assignments of labels without any generalization. Few-shot learning schemes such as matching \cite{vinyals2016matching} and prototypical networks \cite{snell2017prototypical}, and success of stochastic configuration networks in practice \cite{wang2017stochastic} are another manifestations of the same phenomenon.

These results suggest that neural networks' generalization capabilities are intrinsically linked with internal regularities in the data sets and also with representations of these regularities in the networks' latent spaces. Stochastic separation theorems reveal an important characteristic of this important regularity: if an object has a `compact' representation in the network's latent space  then such object can be learned from just few or even single example. The notion of `compactness' here should be specified. For various classes of problems it can be thought of as covering of data by bounded number of balls with limited radii for some bounds, depending on the dimension and variability of the data, or as a sufficiently fast decay of a  sequence of  dataset diameters.  Absence of such compact representations may require exponentially large training samples to learn from. In this respect, the theorems suggest that a successful learning process in modern networks with large VC dimension must include building an adequate data representation in the network's latent space.

The extreme rarefaction of data in the post-classical multidimensional world leads to many unexpected phenomena: applicability of simple discriminants to apparently complex problem of correcting AI, the possibility of stealth attacks on AI systems and the apparent simplicity of the concept cells and sparse coding in the brain. \citet{Kreinovich2019} characterized this bunch of phenomena as ``unheard-of simplicity'',  following Pasternak's famous verses. Stochastic separation theorems with optimal bounds provide a tool for dealing with these problems..


\begin{thebibliography}{99}

\bibitem[Bac \& Zinovyev(2020)Bac  \&  Zinovyev]{BacZinovyev2020}Bac, J., \&  Zinovyev, A. (2020). Lizard brain: tackling locally low-dimensional yet globally complex organization of multi-dimensional datasets. {\em Frontiers in Neurorobotics}, {\em 13}, 110. \url{https://doi.org/10.3389/fnbot.2019.00110}.

\bibitem[Ball(1997)]{Bal1997}
Ball, K. (1997).
\newblock An Elementary Introduction to Modern Convex Geometry.
  In {\em Flavors of Geometry} (pp. 1--58). Cambridge University Press: Cambridge,  UK.

\bibitem[{B{\'a}r{\'a}ny} \& {F{\"u}redi}(1988)]{convhull}
{B{\'a}r{\'a}ny}, I., \& {F{\"u}redi}, Z. (1988).
\newblock On the shape of the convex hull of random points.
\newblock {\em Probab. Theory Relat. Fields}, {\em 77},~231--240. \url{https://doi.org/10.1007/BF00334039}.

\bibitem[Bobkov(2010)]{Bobkov}Bobkov,  G.G. (2010). Gaussian concentration for a class of spherically invariant measures.  {\em Journal of Mathematical Sciences} {\em 167} (3), 326--339. \url{https://doi.org/10.1007/s10958-010-9922-0}.

\bibitem[Boucheron(2013)]{Boucheron2013}
Boucheron, G., Lugosi, G., \& Massart. P. (2013)
\newblock {\em Concentration inequalities: A nonasymptotic theory of independence.} Oxford university press.

\bibitem[Camastra(2003)]{Camastra2003}Camastra, F. (2003). Data dimensionality estimation methods: a survey. {\em Pattern Recognit.}, {\em 36} (12), 2945--2954. \url{https://doi.org/10.1016/S0031-3203(03)00176-6}.

\bibitem[Donoho(2000)]{Donoho2000}
Donoho, D.L. (2000).
\newblock {High-Dimensional Data Analysis: The Curses and Blessings of
  Dimensionality}.
\newblock  \emph{Invited lecture at Mathematical Challenges of the 21st Century}, AMS National Meeting, Los Angeles, CA, USA, August 6-12, 2000. CiteSeerX \url{http://citeseerx.ist.psu.edu/viewdoc/summary?doi=10.1.1.329.3392}{10.1.1.329.3392}.

\bibitem[Donoho \& Tanner(2009)]{DonohoTanner2009}
Donoho, D., \&  Tanner, J. (2009).
\newblock Observed universality of phase transitions in high-dimensional
  geometry, with implications for modern data analysis and signal processing.
\newblock {\em Phil. Trans. R. Soc. A}, {\em 367},~4273--4293. \url{https://doi.org/10.1098/rsta.2009.0152}.

\bibitem[Giannopoulos\& Milman(2000)]{GianMilman2000}
Giannopoulos, A.A., \&  Milman, V.D. (2000).
\newblock Concentration property on probability spaces.
\newblock {\em Adv. Math.}  {\em 156}, 77--106. \url{https://doi.org/10.1006/aima.2000.1949}.

\bibitem[Gibbs(1960)]{Gibbs1902}
Gibbs, J.W. (1960).
\newblock {\em Elementary Principles in Statistical Mechanics, Developed with
  Especial Reference to the Rational Foundation of Thermodynamics}. Dover
  Publications: New York,  NY, USA.

\bibitem[Gorban   et~al.(2018) Gorban,  Golubkov,   Grechuk,   Mirkes,  \& Tyukin]{Gorbetal2018} Gorban, A.N., Golubkov, A. , Grechuk, B.,  Mirkes, E.M., \& Tyukin I.Y. (2018). Correction of AI systems by linear discriminants: Probabilistic foundations. {\em  Information Sciences}, {\em 466},  303--322. \url{https://doi.org/10.1016/j.ins.2018.07.040}
 
\bibitem[Gorban et~al.(2008)Gorban, K\'egl, Wunsch, \&
  Zinovyev]{GorbanKegl2008}
Gorban, A.N., K\'egl, B., Wunsch, D., Zinovyev, A. (Eds.) (2008).
\newblock {\em Principal Manifolds for Data Visualisation and Dimension
  Reduction}; Springer: Berlin/Heidelberg, {Germany}.
  \url{https://doi.org/10.1007/978-3-540-73750-6}.

\bibitem[Gorban et~al.(2019)Gorban, Makarov, \& Tyukin]{GorbanMakTyu2019}Gorban, A.N., Makarov, V.A., \& Tyukin, I.Y. (2019). The unreasonable effectiveness of small neural ensembles in high-dimensional brain. {\em 	Phys. Life Rev.}, {\em 29}, 55--88. \url{https://doi.org/10.1016/j.plrev.2018.09.005}.

\bibitem[Gorban and Tyukin(2017)Gorban, Tyukin]{GorbTyu2017} Gorban,  A.N., \& Tyukin, I.Y. (2017).  Stochastic separation theorems.  {\em Neural Netw.} {\em 94},  255--259. \url{https://doi.org/10.1016/j.neunet.2017.07.014}.


\bibitem[Gorban and Tyukin(2018)Gorban  \& Tyukin]{GorTyukPhil2018}
Gorban, A.N., \& Tyukin, I.Y.  (2018).
\newblock Blessing of dimensionality: mathematical foundations of the
  statistical physics of data.
\newblock {\em Phil. Trans. R. Soc. A}, {\em 376},~20170237, \url{https://doi.org/10.1098/rsta.2017.0237}.

\bibitem[{Gorban et~al.(2016a)Gorban, Tyukin, Prokhorov, \&
  Sofeikov}]{GorbTyuProSof2016}
\bibinfo{author}{Gorban, A.N.}, \bibinfo{author}{Tyukin, I.},
  \bibinfo{author}{Prokhorov, D.},  \& \bibinfo{author}{Sofeikov, K.} 
  \bibinfo{year}({2016}a).
\newblock \bibinfo{title}{Approximation with random bases: Pro et contra}.
\newblock \bibinfo{journal}{\em Information Sciences} \bibinfo{volume}{\em 364--365},
  \bibinfo{pages}{129--145}. \url{https://doi.org/10.1016/j.ins.2015.09.021}

\bibitem[{Gorban et~al.(2016b)Gorban, Tyukin, \& Romanenko}]{GorbanTyuRom2016}
  \bibinfo{author}{Gorban, A.N.}, \bibinfo{author}{Tyukin, I.Y.}, \& \bibinfo{author}{Romanenko, I.} \bibinfo{year}({2016}b).
  \newblock \bibinfo{title}{The blessing of dimensionality: Separation theorems in the thermodynamic limit.}
  \newblock \bibinfo{journal}{\em IFAC-PapersOnLine}, \bibinfo{volume}{\em 49}, \bibinfo{number}(24), \bibinfo{pages}{64--69}.

\bibitem[Gorban \& Zinovyev(2010)]{GorZin2010}
Gorban, A.N., \& Zinovyev, A. (2010).
\newblock Principal manifolds and graphs in practice: from molecular biology to
  dynamical systems.
\newblock {\em Int. J. Neural Syst.} {\em 20},~219--232, \url{https://doi.org/10.1142/S0129065710002383}.

\bibitem[Grechuk(2019)]{Grechuk}Grechuk  B. (2019).  Practical stochastic separation theorems for product distributions, In {\em Proc. 2019 International Joint Conference on Neural Networks (IJCNN), Budapest, Hungary}, IEEE Press,  pp. 1-8, \url{https://doi.org/10.1109/IJCNN.2019.8851817}.

\bibitem[Gromov(2003)]{Gromov2003}
Gromov, M. (2003).
\newblock Isoperimetry of waists and concentration of maps.
\newblock {\em Geom. Funct. Anal.}, {\em 13},~178--215, \url{https://doi.org/10.1007/s00039-009-0703-1}.

\bibitem[Holcomb  {et~al.}(2018)Holcomb, Porter, Ault, Mao, \& Wang]{Holcomb2018}Holcomb, S.D., Porter,  W.K., Ault, S.V., Mao, G., \& Wang, J. (2018). Overview on DeepMind and its AlphaGo Zero AI. In: Proceedings of ICBDE '18, the 2018 international conference on big data and education. Association for Computing Machinery, New York, pp. 67-71. \url{https://doi.org/10.1145/3206157.3206174}

\bibitem[Hoeffding(1963)]{Hoeffding}
Hoeffding, W. (1963).
\newblock Probability inequalities for sums of bounded random variables.
\newblock {\em 	J. Am. Stat. Assoc.}, {\em 58},~13--30.
\url{https://doi.org/10.1080/01621459.1963.10500830}.

\bibitem[Jolliffe(1993)]{Joliffe2011}
Jolliffe, I. (1993).
\newblock {\em Principal Component Analysis}; Springer: Berlin/Heidelberg, {Germany}.

\bibitem[Kainen(1997)]{Kainen1997}
Kainen, P.C. (1997).
\newblock Utilizing geometric anomalies of high dimension: when complexity
  makes computation easier. In K. Warwick \&  M. K\'{a}rn\'{y},
  M.  (Eds.), {\em Computer-Intensive Methods in Control and
  Signal Processing: The Curse of Dimensionality} (pp. 283--294). Springer: New York, NY, USA. \url{https://doi.org/10.1007/978-1-4612-1996-5_18}.

\bibitem[{Kainen \& K{\r{u}}rkov{\'a}(1993)}]{Kurkova1993}
\bibinfo{author}{Kainen, P.}, \& \bibinfo{author}{K{\r{u}}rkov{\'a}, V.} 
  \bibinfo{year}(1993).
\newblock \bibinfo{title}{Quasiorthogonal dimension of {E}uclidian spaces}.
\newblock \bibinfo{journal}{\em Appl. Math. Lett.}, \bibinfo{volume}{\em 6},
  \bibinfo{pages}{7--10}. \url{https://doi.org/10.1016/0893-9659(93)90023-G}.

\bibitem[{Kainen \& K{\r{u}}rkov{\'a}(2020)}]{KainenKurkova2020}
\bibinfo{author}{Kainen, P.}, \& \bibinfo{author}{K{\r{u}}rkov{\'a}, V.} 
  \bibinfo{year}(2020).
\newblock \bibinfo{title}{Quasiorthogonal dimension}.
\newblock In  Kosheleva, O., Shary, S.P., Xiang, G., Zapatrin, R. (Eds.). {\em  Beyond Traditional Probabilistic Data Processing Techniques: Interval, Fuzzy etc. Methods and Their Applications}. Springer, Cham, 615--629.  \url{https://doi.org/10.1007/978-3-030-31041-7_35}.

\bibitem[Kreinovich(2019)]{Kreinovich2019}Kreinovich, V. (2019). The heresy of unheard-of simplicity: Comment on “The unreasonable effectiveness of small neural ensembles in high-dimensional brain” by AN Gorban, VA Makarov, and IY Tyukin. {\em 	Phys. Life Rev.}, 29, 93-95. \url{https://doi.org/10.1016/j.plrev.2019.04.006}.

\bibitem[K{\r{u}}rkov{\'a}(2019)]{KurkovaComm2019}
K{\r{u}}rkov{\'a}, V.  (2019).
\newblock Some insights from high-dimensional spheres: Comment on ``{T}he
  unreasonable effectiveness of small neural ensembles in high-dimensional
  brain'' by {A}lexander {N}. {G}orban et al.
\newblock {\em Phys. Life Rev.}  {\em 29},~98--100. \url{https://doi.org/10.1016/j.plrev.2019.03.014}.

\bibitem[K{\r{u}}rkov{\'a} \& Sanguineti(2019)]{KurkovaSang2019}
K{\r{u}}rkov{\'a}, \& V.; Sanguineti, M.  (2019).
\newblock Probabilistic Bounds for Binary Classification of Large Data Sets.
\newblock In~ L. Oneto, N. Navarin,
A. Sperduti, \& D. Anguita  (Eds.), {\em Proceedings of the International Neural Networks Society, {Genova, Italy, 16--18 April 2019},  Volume~1}  (pp. 309--319). Springer: {Berlin/Heidelberg, Germany.}   \url{https://doi.org/10.1007/978-3-030-16841-4_32}.

\bibitem[Ledoux(2001)]{Ledoux2001}Ledoux, M.  (2001).  {\em The Concentration of Measure Phenomenon}  (Mathematical Surveys \& Monographs No. 89), AMS.

\bibitem[L\'evy(1951)]{Levy1951}
L\'evy, P. (1951).
\newblock {\em Probl\`emes Concrets D'analyse Fonctionnelle}. Gauthier-Villars:
  Paris,  France.


\bibitem[Li(2011)]{Li} Li, S.  (2011). Concise formulas for the area and volume of a hyperspherical cap. {\em Asian Journal of Mathematics and Statistics} {\em 4} (1), 66--70. \url{https://doi.org/10.3923/ajms.2011.66.70}.



\bibitem[Lopez \& Sesma(1999)]{Lopez}Lopez, J.L., \&  Sesma  J. (1999). Asymptotic expansion of the incomplete beta function for large values of the first parameter. {\em Integral Transforms Spec. Funct.}, {\em 8} (3-4),  233--236. \url{https://doi.org/10.1080/10652469908819230}.

\bibitem[Mirkes et al.(2020)Mirkes,   Allohibi,  \& Gorban]{Mirkesetal2020}Mirkes, E.M., Allohibi, J., Gorban, A.N. (2020) Fractional Norms and Quasinorms Do Not Help to Overcome the Curse of Dimensionality. {\em Entropy},  {\em 22}, 1105.


\bibitem[Moczko \em{et~al.}(2016)Moczko, Mirkes, C\'{a}ceres, Gorban, and
Piletsky]{MoczkoMirkesGorPil2016}
Moczko, E., Mirkes, E.M., C\'{a}ceres, C., Gorban, A.N., \& Piletsky, S. (2016).
\newblock Fluorescence-based assay as a new screening tool for toxic chemicals.
\newblock {\em Sci. Rep.}, {\em 6},~33922. \url{https://doi.org/10.1038/srep33922}.

\bibitem[Petrov(1965)]{Petrov}
Petrov, V. (1965).
\newblock On the probabilities of large deviations for sums of independent random variables.
\newblock {\em Theory Probab. Appl., }, {\em 10} (2), 287--298. \url{https://doi.org/10.1137/1110033}.

\bibitem[Pham(2007)]{Pham}
Pham, H. (2007). Some applications and methods of large deviations in finance and insurance. In Paris-Princeton Lectures on Mathematical Finance 2004 (pp. 191-244). Lecture Notes in Mathematics, vol 1919. Springer, Berlin, Heidelberg. \url{https://doi.org/10.1007/978-3-540-73327-0_5}.

\bibitem[Quian Quiroga(2019)]{QuianQuiroga(2019)}Quian Quiroga, R.   (2019). Akakhievitch revisited Comment on ``The unreasonable effectiveness of small neural ensembles in high-dimensional brain'' by Alexander N. Gorban et al. {\em 	Phys. Life Rev.}, {\em 29}, 111-114.

\bibitem[Quian Quiroga \em{et~al.}(2013)Quian Quiroga, Fried,  \& Koch]{QuianQuiroga2013}Quian Quiroga, R., Fried, I., \& Koch, C. (2013). Brain cells for grandmother. {\em Scientific American}, {\em 308} (2), 30--35. \url{http://www.jstor.org/stable/26017950}.

\bibitem[Quian Quiroga \em{et~al.}(2005)Quian Quiroga, Reddy, Kreiman, Koch \& Fried]{QuianQuiroga2005}Quian Quiroga, R. Q., Reddy, L., Kreiman, G., Koch, C., \& Fried, I. (2005). Invariant visual representation by single neurons in the human brain. {\em Nature}, {\em 435} (7045), 1102--1107. \url{https://doi.org/10.1038/nature03687}.



\bibitem[{Rosenblatt(1962)}]{Rosenblatt1962}
\bibinfo{author}{Rosenblatt, F.}, \bibinfo{year}{(1962)}.
\newblock \bibinfo{title}{Principles of Neurodynamics: Perceptrons and the
  Theory of Brain Mechanisms}.
\newblock \bibinfo{publisher}{Spartan Books}.


\bibitem[Sculley {\em et~al.}(2015)Sculley, Holt, Golovin, Davydov, Phillips, Ebner, Chaudhary, Young, Crespo, \& Dennison]{Sculley2015}Sculley D., Holt G., Golovin D., Davydov E., Phillips  T., Ebner D., Chaudhary V., Young M., Crespo J.-F., \& Dennison D.   (2015).  In C.  Cortes, N.D. Lawrence, D.D. Lee, M. Sugiyama, \& R. Garnett (Eds.), {\em  Advances in Neural Information Processing Systems 28}, Proc. of 28th International Conference on Neural Information Processing Systems (NIPS 2015) (pp. 2503--2511). Curran Associates, Inc., N.Y. \url{http://papers.nips.cc/paper/5656-hidden-technical-debt-in-machine-learning-systems.pdf}.


\bibitem[Sidorov \& Zolotykh(2020)]{Sidorov2020}Sidorov, S., Zolotykh, N. (2020). Linear and {F}isher separability of random points in the d-dimensional spherical layer and inside the d-dimensional cube. {\em Entropy},  {\em 22}(11), 1281. \url{https://doi.org/10.3390/e22111281}

\bibitem[Snell \em{et~al.}(2017)]{snell2017prototypical}Snell, J., Swersky, \& K., Zemel, R. (2017). Prototypical networks for few-shot learning. In  I. Guyon, U.V. Luxburg, S. Bengio, H. Wallach, R. Fergus, S. Vishwanathan, R. Garnett (Eds.), {\em  Advances in Neural Information Processing Systems 30}, Proc. of 30th International Conference on Neural Information Processing Systems (NIPS 2017) (pp. 4077--4087). Curran Associates, Inc., N.Y. \url{https://proceedings.neurips.cc/paper/2017/hash/cb8da6767461f2812ae4290eac7cbc42-Abstract.html}

 
\bibitem[Tapia {\em et~al.}(2020)Tapia, Tyukin, \& Makarov]{Tapia2020}Tapia, C. C., Tyukin, I., \& Makarov, V. A. (2020). Universal principles justify the existence of concept cells. {\em Sci. Rep.}, {\em 10} (1), 1--9. \url{https://doi.org/10.1038/s41598-020-64466-7}.

\bibitem[Tyukin {\em et~al.}(2019)Tyukin, I., Gorban, A. N., Calvo, C., Makarova, J., \& Makarov]{TyukinBMB2019} Tyukin, I., Gorban, A. N., Calvo, C., Makarova, J., \& Makarov, V. A. (2019). High-dimensional brain: A tool for encoding and rapid learning of memories by single neurons. {\em Bull. Math. Biol.}, {\em 81}(11), 4856--4888. \url{https://doi.org/10.1007/s11538-018-0415-5}

\bibitem[Tyukin  {\em et~al.}(2018)Tyukin, Gorban, Sofeikov, and
Romanenko]{TyukinSofeikov2018}
Tyukin, I.Y., Gorban, A.N., Sofeikov, K., \& Romanenko, I.  (2018).
\newblock Knowledge transfer between artificial intelligence systems.
\newblock {\em Front. Neurorobot.}, {\em 12}, 49. \url{https://doi.org/10.3389/fnbot.2018.00049}.

\bibitem[Tyukin {\em et~al.}(2020)Tyukin, Higham, and Gorban]{TyukinEtAl2020} Tyukin, I.Y., Higham, D.J., \&   Gorban, A. N. (2020). On adversarial examples and stealth attacks in artificial intelligence systems. In {\em Proc. 2020 International Joint Conference on Neural Networks (IJCNN), Glasgow, United Kingdom, 2020} (pp. 1--6), IEEE.  \url{https://doi.org/10.1109/IJCNN48605.2020.9207472}.
 
 \bibitem[Vapnik(1999)]{vapnik1999overview}Vapnik, V.N. (1999). An overview of statistical learning theory. {\em IEEE Trans. Neural Netw.}, {\em 10}(5), 988--999. \url{https://doi.org/10.1109/72.788640}.
  
\bibitem[Vinyals \em{et~al.}(2016)]{vinyals2016matching}Vinyals, O.,  Blundell, C., Lillicrap, T., Kavukcuoglu, K., \& Wierstra, D. (2016). Matching networks for one shot learning. In   D.D. Lee, U. von Luxburg, R. Garnett, M. Sugiyama, I. Guyon  (Eds.) Advances in Neural Information Processing Systems 29, Proc. of 30th Annual Conference on Neural Information Processing Systems,  Barcelona, Spain (NIPS 2016), (pp. 3637--3646). Curran Associates, Inc., N.Y.  \url{http://papers.neurips.cc/paper/6385-matching-networks-for-one-shot-learning}.
  
  
\bibitem[Vershynin(2018)]{Vershynin2018}
Vershynin, R. (2018).
\newblock {\em High-Dimensional Probability: An Introduction with Applications
  in Data Science}. Cambridge Series in Statistical and Probabilistic
  Mathematics; Cambridge University Press: Cambridge, UK.
  
\bibitem[Viskontas et al.(2009)Viskontas, Quian Quiroga, \& Fried]{Viskontas2009}Viskontas, I. V., Quian Quiroga, R., \& Fried, I. (2009). Human medial temporal lobe neurons respond preferentially to personally relevant images. 	{\em Proc. Natl. Acad. Sci. U.S.A.}, {\em 106}(50), 21329--21334.  \url{https://doi.org/10.1073/pnas.0902319106}.

\bibitem[Wang \& Li(2017)]{wang2017stochastic}Wang, D. \& Li, M. (2017). Stochastic configuration networks: Fundamentals and algorithms. {\em IEEE Trans. Cybern.}, {\em 47}(10), 3466--3479.\url{https://doi.org/10.1109/TCYB.2017.2734043}.

\bibitem[Wendel(1948)]{Wendel}Wendel, J.G. (1948).  Note on the gamma function. {\em Amer. Math. Monthly},  {\em 55} (9), 563--564. \url{https://www.jstor.org/stable/i314786}.

\bibitem[Wong(2001)]{Wong}Wong,  R. (2001). Asymptotic Approximations of Integrals. Philadelphia, PA: SIAM.

\bibitem[Zhang et al.(2016)]{zhang2016understanding} Zhang, C., Bengio, S., Hardt, M., Recht, B., \& Vinyals, O. (2016). Understanding deep learning requires rethinking generalization. arXiv preprint arXiv:1611.03530. \url{https://arxiv.org/abs/1611.03530}.

\end{thebibliography}
\end{document}